%% file: main.tex
\documentclass{article}

\usepackage{hyperref}
\usepackage{lifetime}
\usepackage[accepted]{icml2021}

\icmltitlerunning{
Leveraging Unlabeled Data to Guarantee Generalization
}

\begin{document}

\twocolumn[
\icmltitle{
% Post-Training Generalization Bound
% Guaranteeing Generalization by Training on Noisily Labeled Data
RATT: Leveraging Unlabeled Data 
to Guarantee Generalization
% Randomly Assign Train and Track
}




\begin{icmlauthorlist}
\icmlauthor{Saurabh Garg}{cmu}
\icmlauthor{Sivaraman Balakrishnan}{cmu}
\icmlauthor{J. Zico Kolter}{cmu}
\icmlauthor{Zachary C. Lipton}{cmu}
\end{icmlauthorlist}

\icmlaffiliation{cmu}{Carnegie Mellon University}
% \icmlaffiliation{csd}{Computer Science Department, Carnegie Mellon University}
% \icmlaffiliation{sds}{Department of Statistics and Data Science, Carnegie Mellon University}
\icmlcorrespondingauthor{Saurabh Garg}{sgarg2@andrew.cmu.edu}
\icmlkeywords{
learning theory, generalization bounds, deep learning, linear models, noisy data}

\vskip 0.3in
]



\printAffiliationsAndNotice{}  % leave blank if no need to mention equal contribution

\begin{abstract}
\input{sections/abstract.tex}
\end{abstract}

% \vspace{-10pt}
\section{Introduction} \label{sec:intro}
\input{sections/01_intro}

\section{Preliminaries} \label{sec:setup}
\input{sections/02_setup.tex}
\section{Generalization Bound for RATT with ERM}
% \section{RATT for ERM---Generalization Bound}
\label{sec:ERM_training}
\input{sections/03_theory_0-1.tex}
\section{Generalization Bound for RATT with Gradient Descent}
\label{sec:linear_models}
\input{sections/04_linear.tex}

\section{Empirical Study and Implications} \label{sec:exp}
\input{sections/05_exp.tex}
\section{Discussion and Connections to Prior Work}\label{sec:discuss}
\input{sections/06_discuss.tex}
\textbf{Other related work.} %\label{sec:prior}
\input{sections/07_prior.tex}

\section{Conclusion and Future work} \label{sec:conc}
\input{sections/08_conclusion.tex}
\section*{Acknowledgements}

SG thanks Divyansh Kaushik for help with NLP code.  
This material is based on research sponsored by Air Force Research Laboratory (AFRL) under agreement number FA8750-19-1-1000. The U.S. Government is authorized to reproduce and distribute reprints for Government purposes notwithstanding any copyright notation therein. 
The views and conclusions contained herein are those of the authors and should not be interpreted as necessarily representing the official policies or endorsements, either expressed or implied, of Air Force Laboratory, DARPA or the U.S. Government. 
SB acknowledges funding from the NSF grants DMS-1713003 and CIF-1763734, as well as Amazon AI and a Google Research Scholar Award. 
ZL acknowledges Amazon AI, Salesforce Research, Facebook, UPMC, Abridge, the PwC Center, the Block Center, the Center for Machine Learning and Health, and the CMU Software Engineering Institute (SEI),
for their generous support of ACMI Lab's research on machine learning under distribution shift.

% \newpage
\bibliography{lifetime}
\bibliographystyle{icml2021}

\newpage
\appendix
\input{sections/appendix}

\end{document}

%% file: sections/abstract.tex
To assess generalization, 
machine learning scientists typically either
(i) bound the generalization gap 
and then (after training) 
plug in the empirical risk 
to obtain a bound on the true risk; 
or (ii) validate empirically on holdout data.
However, (i) typically yields vacuous guarantees 
for overparameterized models;
and (ii) shrinks the training set
and its guarantee erodes
with each re-use of the holdout set.
In this paper, 
we leverage unlabeled data
to produce generalization bounds. 
After augmenting our (labeled) training set
with randomly labeled data,
we train in the standard fashion.
Whenever classifiers achieve 
low error on the clean data
but high error on the random data,
our bound ensures that the true risk is low.
We prove that our bound is valid 
for 0-1 empirical risk minimization 
and with linear classifiers
trained by gradient descent. 
Our approach is especially useful 
in conjunction with deep learning
due to the early learning phenomenon 
whereby networks fit true labels 
before noisy labels
but requires one intuitive assumption.
Empirically, on canonical computer vision and NLP tasks, 
our bound provides non-vacuous generalization guarantees
that track actual performance closely.
This work 
enables practitioners to certify generalization
even when (labeled) holdout data is unavailable 
and provides insights into the relationship 
between random label noise and generalization. 
Code is available at 
\href{https://github.com/acmi-lab/RATT\_generalization\_bound}{https://github.com/acmi-lab/RATT\_generalization\_bound}.

%% file: sections/01_intro.tex
Typically, machine learning scientists
establish generalization in one of two ways.
One approach, favored by learning theorists,
places an \emph{a priori} bound on the gap 
between the empirical and true risks,
usually in terms of the complexity 
of the hypothesis class.
After fitting the model on the available data,
one can plug in the empirical risk 
to obtain a guarantee on the true risk. 
The second approach, favored by practitioners,
involves splitting the available data
into training and holdout partitions,
fitting the models on the former
and estimating the population risk with the latter.

Surely, both approaches are useful,
with the former providing theoretical insights 
and the latter guiding the development 
of a vast array of practical technology.
Nevertheless, both methods have drawbacks.
Most \emph{a priori} generalization bounds 
rely on uniform convergence 
and thus fail to explain the ability 
of overparameterized networks to generalize
\citep{zhang2016understanding,nagarajan2019uniform}. 
On the other hand, 
provisioning a holdout dataset 
restricts the amount of labeled data 
available for training.
Moreover, risk estimates based on holdout sets 
lose their validity 
with successive re-use of the holdout data
due to adaptive overfitting 
\citep{murphy2012machine,dwork2015preserving,blum2015ladder}.
However, recent empirical studies suggest
that on large benchmark datasets,
adaptive overfitting is surprisingly absent
\citep{recht2019imagenet}.

In this paper, we propose
Randomly Assign, Train and Track (RATT), 
a new method that leverages unlabeled data 
to provide a \emph{post-training} 
bound on the true risk 
(i.e., the population error).
Here, we assign random labels 
to a fresh batch of unlabeled data, 
augmenting the clean training dataset
with these randomly labeled points. 
Next, we train on this data, 
following standard risk minimization practices. 
Finally, we track the error
on the randomly labeled 
portion of training data,
estimating the error 
on the mislabeled portion
and using this quantity 
to upper bound the population error.
% by roughly twice the error on the mislabeled portion minus one. 

Counterintuitively, 
we guarantee generalization 
by guaranteeing overfitting.
Specifically, we prove that 
Empirical Risk Minimization (ERM) 
with 0-1 loss leads to lower error 
on the \emph{mislabeled training data}
than on the \emph{mislabeled population}.
% \emph{the population} of mislabeled data. 
Thus, if despite minimizing the loss 
on the combined training data,
we nevertheless have high error 
on the mislabeled portion,
then the (mislabeled) population error will be even higher. 
Then, by complementarity, 
the (clean) population error must be low.
Finally, we show how to obtain this guarantee
using randomly labeled (vs mislabeled data),
thus enabling us to incorporate unlabeled data.

\begin{figure}[t!]
    % {r}{0.5\textwidth}
        \centering 
        % \vspace{-15pt}
        % \includegraphics[width=0.9\linewidth]{example-image-a}
        \includegraphics[width=0.9\linewidth]{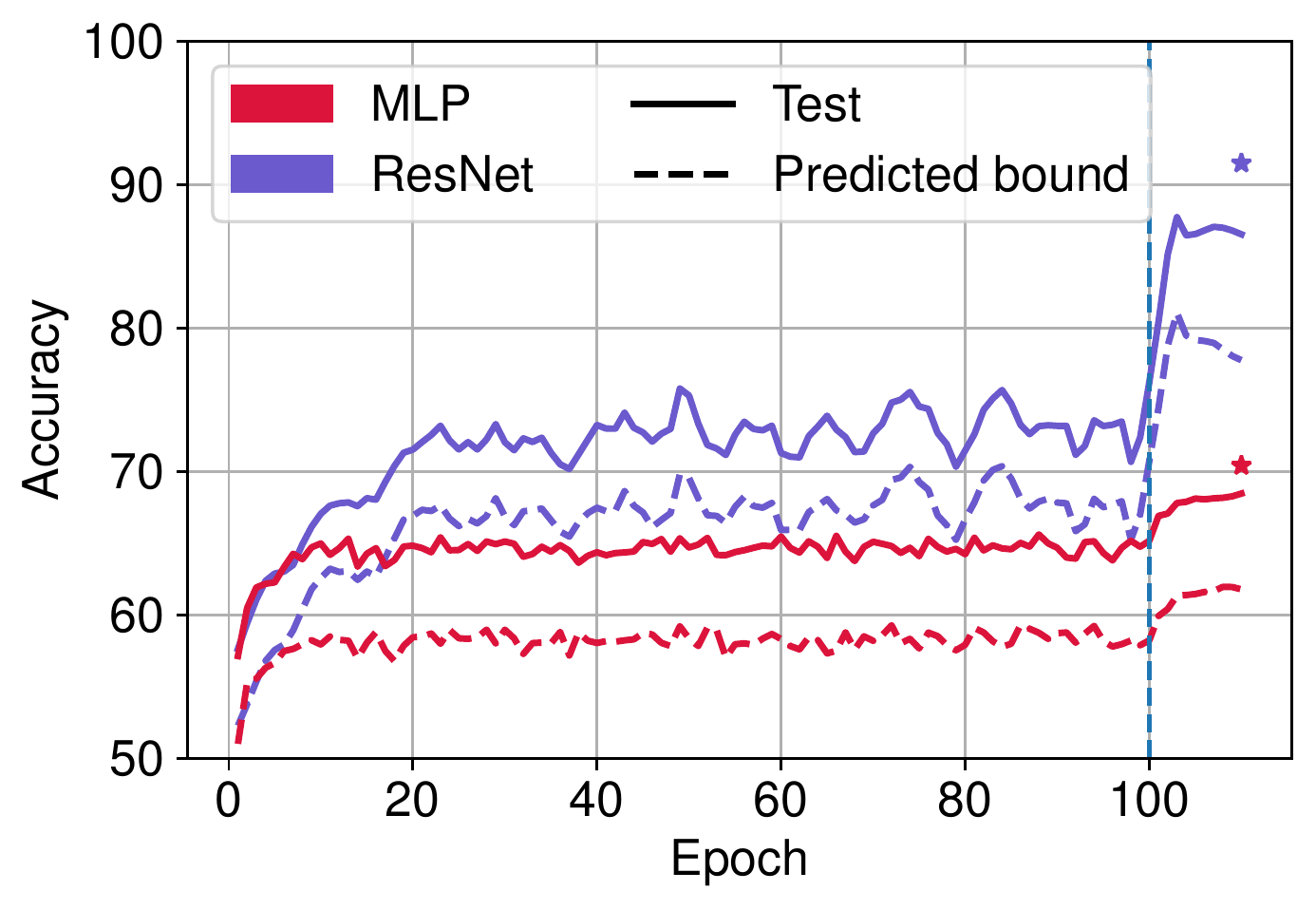}
        \vspace{-8pt}
        \caption{ \textbf{Predicted lower bound on the clean population error} 
        with ResNet and MLP on binary CIFAR. 
        Results aggregated over $5$ seeds. 
        `*' denotes the best test performance 
        % on training with just the clean data 
        achieved when training with only clean data
        and the same hyperparameters 
        (except for the stopping point).
        % except for the stopping point. 
        The bound predicted by RATT (RHS in \eqref{eq:thm1}) closely tracks 
        the population accuracy on clean data. 
        }\label{fig:error_CIFAR10}
        \vspace{-5pt}
\end{figure}

To expand the applicability of our idea
beyond ERM on 0-1 error,
we prove corresponding results 
for a linear classifier 
trained by gradient descent
to minimize squared loss. 
Furthermore, leveraging the connection 
between early stopping 
and $\ell_2$-regularization 
in linear models 
\citep{ali2018continuous,ali2020implicit,suggala2018connecting}, 
our results extend to early-stopped gradient descent.  
Because we make no assumptions
on the data distribution, 
our results on linear models
hold for more complex models 
such as kernel regression 
and neural networks in the Neural Tangent Kernel (NTK) regime
\citep{jacot2018neural, du2018gradient, du2019gradient, allen2019convergence, chizat2019lazy}.

Addressing practical deep learning models,
% while 
our guarantee requires an additional (reasonable) assumption.
% we experimentally verify 
Our experiments show that the bound 
yields non-vacuous guarantees
that track test error
across several major architectures 
on a range of benchmark datasets
for computer vision and Natural Language Processing (NLP).
Because, in practice,
overparameterized deep networks
exhibit an \emph{early learning phenomenon},
fitting clean data 
before mislabeled data
\citep{liu2020early,arora2019fine,li2019gradient}, 
our procedure yields tight bounds 
in the early phases of learning. 
Experimentally, we confirm 
the early learning phenomenon 
in standard Stochastic Gradient Descent (SGD) training 
and illustrate the effectiveness of weight decay 
combined with large initial learning rates 
in avoiding interpolation to mislabeled data 
while maintaining fit on the training data, 
strengthening the guarantee provided by our method.

% Our work derives inspiration 
% from recent observations
% on deep learning with noisy data 
% \citep{zhang2016understanding,rolnick2017deep,arpit2017closer}. 
% In particular, \citet{hu2019simple}
% and \cite{li2019gradient} have shown
% that even when training data is contaminated
% by a small amount of mislabeled data,
% overparameterized networks 
% nevertheless generalize well 
% when trained with early stopping and weight decay.
% Moreover, the tend overparameterized models 
% tend to fit the clean training data first
% before eventually fitting the mislabeled data 
% \citep{liu2020early,arora2019fine,li2019gradient}.
% The ability of our method to give non-vacuous
% bounds relies on these phenomena.
% Due to these phenomena, %give rise
% we can provide non-vacuous bounds. 

%%
%% IS the table enough or we should add something else?
%%
% Finally, we demonstrate the usefulness 
% of our bound in guiding hyperparameter tuning
% (i.e., weight decay and early stopping point) 
% without requiring access to clean validation data. 

To be clear, we do not advocate RATT 
as a blanket replacement 
for the holdout approach.
Our main contribution is to introduce 
a new theoretical perspective on generalization
and to provide a method that may be applicable
even when the holdout approach is unavailable. 
Of interest, unlike generalization bounds
based on uniform-convergence 
that restrict the complexity 
of the hypothesis class
\citep{neyshabur2018role,neyshabur2015norm,neyshabur2017pac,bartlett2017spectrally,nagarajan2019deterministic}, 
our \emph{post hoc} bounds depend 
only on the fit to mislabeled data.
We emphasize that our theory 
does not guarantee \emph{a priori}
% that the early learning phenomenon should take place
that early learning should take place
but only \emph{a posteriori} that when it does,
we can provide non-vacuous bounds 
on the population error.
Conceptually, this finding underscores 
the significance 
of the early learning phenomenon 
in the presence of noisy labels
and motivates further work 
to explain why it occurs.

%% file: sections/02_setup.tex
By $\enorm{\cdot}$, 
and $\inner{\cdot}{\cdot}$
we denote the Euclidean norm
and inner product,
respectively.
For a vector $v\in \Real^d$, 
we use $v_j$ to denote its $j^\text{th}$ entry, and for an event $E$ we let $\indict{E}$ denote the binary indicator of the event.

Suppose we have a multiclass classification problem
with the input domain $\calX \subseteq \Real^d$ 
and label space $\calY = \{1, 2, \ldots, k\}$\footnote{For 
binary classification, 
we use $\calY = \{-1,1\}$.}. 
By $\calD$, we denote the distribution 
over $\calX \times \calY$. 
A dataset $S \defeq \{(x_i, y_i)\}_{i=1}^n \sim \calD^n$
contains $n$ points sampled i.i.d. from $\calD$.
By $\calS$, $\calT$, and $\wt \calS$, 
we denote the (uniform) empirical distribution
over points in datasets $S$, $T$, 
and $\wt S$, respectively. 
Let $\calF$ be a class of hypotheses 
mapping $\calX$ to $\Real^{k}$.  
A \emph{training algorithm} $\calA$:
takes a dataset $S$
and returns a classifier $f(\calA,S) \in \calF$. 
When the context is clear,
we drop the parentheses for convenience.
Given a classifier $f$ and datum $(x, y)$,
we denote the 0-1 error 
(i.e., classification error)
on that point by 
$\error(f(x), y) \defeq \indict{ y\not\in \argmax_{j\in\calY} f_j(x) }$,
We express the \emph{population error} on $\calD$ as
$\error_\calD (f) \defeq \Expt{(x,y) \sim \calD}{\error(f(x),y) }$ 
and the \emph{empirical error} on $S$ as  
$\error_\calS(f) \defeq \Expt{(x,y) \sim \calS}{\error(f(x),y)} = \frac{1}{n} \sum_{i=1}^n {\error(f(x_i),y_i)}$.

Throughout,
% the paper, 
we consider 
a \emph{random label assignment} procedure: 
draw $x\sim \calD_\calX$ 
(the underlying distribution over $\calX$),
and then assign a label sampled
uniformly at random. 
We denote a randomly labeled dataset 
by $\wt S \defeq \{ (x_i, y_i)\}_{i=1}^m \sim \wt \calD^m $, 
where $\wt \calD$
is the distribution of randomly labeled data. 
By $\calDm$, we denote the mislabeled distribution
that corresponds to selecting examples $(x,y)$ 
according to $\calD$ and then re-assigning the label 
by sampling among the incorrect labels $y' \neq y$
(renormalizing the label marginal). 

%% file: sections/03_theory_0-1.tex
We now present our generalization bound 
and proof sketches 
for ERM on the 0-1 loss
(full proofs in \appref{app:proof_erm}).
For any dataset $T$, 
ERM returns the classifier $\widehat f$
that minimizes the empirical error:
% as follows
% that minimizes the error 
\begin{align}
    \widehat f \defeq \argmin_{f\in \calF} \error_\calT(f) \,. \label{eq:erm}
\end{align}
% We can now examine RATT applied
% 
We focus first on
binary classification.  
% with balanced classes.
Assume we have a clean dataset 
$S \sim \calD^n$ of $n$ points
and a randomly labeled dataset 
$\wt S \sim \wt \calD^m$ of $m~(<n)$ points with
% Because the classes are balanced, 
labels in $\wt S$ are assigned uniformly at random.
We show that with 0-1 loss minimization 
on the union of $S$ and $\wt S$,
we obtain a classifier 
whose error on $\calD$
% on data distribution
is upper bounded 
by a function of the empirical errors 
on clean data $\error_\calS$
(lower is better)
and on randomly labeled data 
$\error_{\wt \calS}$ 
(higher is better): 
% by the \emph{fit} on the random data 
% plus the error on the clean training data. 
% Formally,  
\begin{theorem} \label{thm:error_ERM}
    % Let $S \sim \calD^n $ and $\wt S \sim \wt D^m$. 
    For any classifier $\wh f$ obtained by ERM \eqref{eq:erm} on dataset $S \cup \wt S$,~ 
    %Assume we perform ERM as in \eqref{eq:erm} on $S \cup \wt S$ 
    %and obtain a classifier $\wh f$. 
    %Then 
    for any $\delta > 0$, 
    with probability at least $1-\delta$, 
    % over the random draws of datasets $\wt S$ and $S$, 
    % over the samples $S$ and $\wt S$
    we have  
    \begin{align*}
        \error_\calD(\widehat f)  \le \, & \error_\calS(\widehat f) + 1 - 2 \error_{\wt\calS}(\widehat f) \\
        & + \left(\sqrt{2} \error_{\wt S}(\widehat f)   + 2 + \frac{m}{2n}\right) \sqrt{\frac{\log(4/\delta)}{m}} \,. \label{eq:thm1} \numberthis
    \end{align*}
    %for some constant $c \le 3.2$.
\end{theorem}
% 
% 
% RHS is sum of the error on clean portion 
% and an additional term which we refer as \emph{fit} on the random data.
%% TODO SG: Fix the "dominating term" 
% Besides the empirical error on the clean portion, 
% the dominating term 
% Note that the RHS 
% % \eqref{eq:thm1} 
% % is
% depends on
% $(1 - 2 \error_{\wt\calS}(\widehat f))$ or, 
% equivalently, $2(1/2 - \calE_{\wt S} (\widehat f))$.
% which can be re-written as $2(1/2 - \calE_{\wt S} (\widehat f))$.
% ($1/2 - \calE_{\wt S} (\widehat f)$)
% If $\wt S$ were not included in the training, then $\calE_{\wt S}(f)$ would be approximately 1/2 for any $f$.
In short, this theorem tells us that 
if after training on both clean and randomly labeled data,
we achieve low error on the clean data
but high error (close to $1/2$) 
on the randomly labeled data,
then low population error is guaranteed.
Note that because the labels in ${\wt S}$
are assigned randomly,
the error $\error_{\wt \calS}(f)$
for any fixed predictor $f$
(not dependent on ${\wt S}$)
will be approximately 1/2. 
Thus, if ERM produces a classifier 
that has not fit to the randomly labeled data,
then $(1 - 2 \error_{\wt \calS} (\widehat f))$ 
will be approximately $0$,
and our error will be determined
by the fit to clean data.
The final term %in the bound
accounts for finite sample error---notably,
% and notably, it
it (i) does not depend 
on the complexity 
of the hypothesis class;
and (ii) approaches $0$
at a $\calO(1/\sqrt{m})$ rate
(for $m < n$).

% is just the accuracy of classifier $\widehat f$ 
% assessed on the randomly labeled data $\wt S$.  

Our proof strategy 
unfolds in three steps.
First, in \lemref{lem:fit_mislabeled}
we bound $\error_{\calD} (\wh f)$
in terms of the error %$\error_{\wt S_M} (\wh f)$ 
on the mislabeled subset of $\wt S$.
Next, in Lemmas \ref{lem:mislabeled_error} and \ref{lem:clear_error},
we show that the error on the mislabeled subset
can be accurately estimated 
using only clean and randomly labeled data.

To begin, assume that we actually knew 
the original labels for the randomly labeled data.
By $\wt S_C$ and  $\wt S_M$,
we denote the clean and mislabeled portions
of the randomly labeled data, respectively
(with $\wt S = \wt S_M \cup \wt S_C$).
% Recall, we refer to the distribution 
% of mislabeled points as $\calDm$.
Note that for binary classification,
a lower bound on %the error 
mislabeled population error $\error_{\calDm}  (\wh f)$
directly upper bounds 
the error on the original population $\error_{\calD}  (\wh f)$.
Thus we only need to prove that 
the empirical error on the mislabeled portion of our data 
is lower than the error on unseen mislabeled data,
i.e., $\error_{\wt S_M} (\wh f) \leq \error_{\calDm}  (\wh f) = 1 - \error_{\wt \calS_M}(\widehat f)$ (upto $\calO(1/\sqrt{m})$).
\begin{lemma} \label{lem:fit_mislabeled}
    Assume the same setup as in \thmref{thm:error_ERM}. 
    Then for any $\delta >0$, with probability at least  $1-\delta$ 
    over the random draws of mislabeled data $\wt S_M$, we have 
    \begin{align}
        \error_\calD(\widehat f)  \le 1 -\error_{\wt \calS_M}(\widehat f) + \sqrt{\frac{\log(1/\delta)}{m}}\,. \label{eq:lemma1}
    \end{align}   
\end{lemma} 

% We now provide a proof sketch for 
% the lemma. 
% Lemma \ref{lem:fit_mislabeled}.
\begin{hproof} 
    % The main idea is to define a fixed classifier $f^*$ 
    % which is optimal over random draws of mislabeled points. 
    % The classifier $f^*$ allows relating 
    % classification error of $\widehat f$ 
    % on mislabeled training data $\wt S_M$
    % and the corresponding mislabeled 
    % data distribution $\calDm$. 
    % The main idea of the proof is to establish a comparison between the errors in \eqref{eq:lemma1} by defining a classifier $f^*$ that allows us to relate these errors by using ERM optimality condition of $\wh f$.    
    % 
    The main idea of our proof is to regard 
    the clean portion of the data 
    ($S \cup \wt S_C$) as fixed.
    Then, there exists a classifier $f^*$ 
    that is optimal over draws 
    of the mislabeled data $\wt S_M$.
    Formally, 
    \begin{align*}
        f^* \defeq \argmin_{f \in \calF} \error_{\widecheck {\calD}} (f),
    \end{align*}
    where $\widecheck \calD$ is a combination of 
    the \emph{empirical distribution} 
    over correctly labeled data $S \cup \wt S_C$
    % in $S\cup \wt S$ 
    and the (population) distribution 
    over mislabeled data $\calDm$. 
    Recall that $\widehat f \defeq \argmin_{f \in \calF} \error_{\calS \cup \wt \calS } (f)$. 
    Since, $\widehat f$ minimizes 0-1 error 
    on $S \cup \wt S$, 
    we have $\error_{\calS \cup \wt \calS}(\widehat f) \le \error_{
    \calS \cup \wt \calS}(f^*)$. 
    Moreover, since $f^*$ is independent of $\wt S_M$, 
    % \footnote{For a fully rigorous argument,
    % refer to the complete proof in App.~\ref{app:proof_erm}.} 
    we have with probability at least $1-\delta$ that
    \begin{align*}
      \error_{\wt \calS_M}(f^*) \le \error_{ \calDm}(f^*) +  \sqrt{\frac{\log(1/\delta)}{m}} \,. 
    \end{align*}
    %$ 
    %for some constant $c_1\le 1/2$. 
    Finally, since $f^*$ is the optimal classifier on $\widecheck \calD$, 
    we have $\error_{\widecheck \calD}(f^*) \le \error_{\widecheck \calD}(\widehat f)$.     
    Combining the above steps and using the fact 
    that $\error_\calD = 1- \error_{\calDm} $, 
    we obtain the desired result.  
\end{hproof}
% 
% Intuitively, \lemref{lem:fit_mislabeled}, we are saying that ERM with 0-1 loss leads to overfitting on the training data. While LHS in \eqref{eq:lemma1} depends on the unknown portion $\wt S_M$, in practice, we can not identify the mislabeled points from randomly labeled data. However, if we can approximate the performance on the clean portion of randomly labeled data (i.e., $\error_{\wt\calS_C}(\widehat f)$), then we can use it to estimate $\error_{\wt\calS_M}(\widehat f)$. 
% 
% 
% Intuitively, \lemref{lem:fit_mislabeled} states 
% that ERM with 0-1 loss leads 
% to overfitting on the training data. 
% 
While the LHS in \eqref{eq:lemma1} depends 
on the unknown portion $\wt S_M$, 
our goal is to use unlabeled data 
(with randomly assigned labels)
for which the mislabeled portion 
cannot be readily identified. 
% in practice,
% we can not 
% identify the mislabeled points from randomly labeled data.
% However, 
% exploiting the exchangeability of the clean data $S$
% and the clean portion of the random data $S_C$ (\lemref{lem:mislabeled_error})
Fortunately, we do not need to identify
the mislabeled points to estimate
% because we can estimate the label marginal,
% we can estimate 
the error on these points in aggregate
$\error_{\wt\calS_M}(\widehat f)$.
% in terms of the errors on the clean 
% and randomly labeled portions of the data. 
% by estimating of the error on clean and rnd
% by approximating the performance 
% on the clean portion of the randomly labeled data 
% (i.e., $\error_{\wt\calS_C}(\widehat f)$),
% we can estimate. 
% 
% 
% by an error estimate on the clean training data. 
Note that because the label marginal is uniform, 
approximately half of the 
% data points 
data
will be correctly labeled 
and the remaining half will be mislabeled. 
% ~\footnote{In principle, we account for errors with finite sampling, but ease of exposition we make this assumption. In \appref{app:proof_erm}, we take this error into account.}. 
Consequently, we can utilize the value 
of $\error_{\wt\calS}(\widehat f)$ 
and an estimate of $\error_{\wt\calS_C}(\widehat f)$ 
to lower bound $\error_{\wt\calS_M}(\widehat f)$. 
We formalize this as follows: 

\begin{lemma} \label{lem:mislabeled_error}
    Assume the same setup as \thmref{thm:error_ERM}. Then for any $\delta >0$, with probability at least $1-\delta$ over the random draws of $\wt S$, we have  
    % \begin{align}
        $\abs{2\error_{\wt \calS}(\widehat f) - \error_{\wt \calS_C}(\widehat f) -  \error_{\wt \calS_M}(\widehat f) } \le  2\error_{\wt \calS}(\widehat f)  \sqrt{\frac{\log(4/\delta)}{2m}}\,. $ % \label{eq:lemma2}
    % \end{align}   
    %  for some constant $c_3 \le 1.0\,$.
\end{lemma}

% Can we use Ergo here or is it latin?
% $2\error_{\wt\calS}(\widehat f) \approx \error_{\wt\calS_M}(\widehat f) + \error_{\wt\calS_C}(\widehat f)$. Hence,
%  knowing $\error_{\wt\calS}(\widehat f)$ , if we otherwise  

To complete the argument, 
% we now present a lemma that 
% bounds the gap 
% exploits the exchangeability of the clean data $S$
% and the clean portion of the random data $S_C$
we show that due to the exchangeability of the clean data $S$
and the clean portion of the randomly labeled data $S_C$,
we can estimate the error on the latter $\error_{\wt \calS_C} (\wh f)$
by the error on the former $\error_{\calS} (\wh f)$.
% We formalize this idea in \lemref{lem:mislabeled_error}.
% between the error on clean training data $\error_{S_C}$ 
% and the error on the clean portion of mislabeled data. 
% 
% Even though the classifier is dataset dependent,
% as the dataset is sampled identically, 
% we can use the exchangeability property for this purpose. 
%  to obtain an estimate of $\error_{\wt \calS_M}(\widehat f)$ in terms of $\error_{\calS}(\widehat f)$.
% We concretize this intution in
% \lemref{lem:mislabeled_error} below:
% which establishes a tight bound on the difference of the error of classifier $\widehat f$ on $\wt S_M$ and on $S$:

\begin{lemma} \label{lem:clear_error}
    Assume the same setup as \thmref{thm:error_ERM}. 
    Then for any $\delta >0$, with probability at least $1-\delta$ 
    over the random draws of $\wt S_C$ and $S$, we have 
    % \begin{align}
        $\abs{\error_{\wt \calS_C}(\widehat f) - \error_{\calS}(\widehat f) } \le \left(1 + \frac{m}{2n}\right) \sqrt{\frac{\log(2/\delta)}{m}}\,.$ %\label{eq:lemma3}
    % \end{align}   
    % for some constant $c_2 \le 1.2\,$.
\end{lemma} 

\lemref{lem:clear_error} establishes a tight bound 
on the difference of the error 
of classifier $\widehat f$
on $\wt S_C$ and on $S$. 
The proof uses Hoeffding's inequality for randomly sampled points from a fixed population~\citep{hoeffding1994probability,bardenet2015concentration}.  

% Now, we are ready to discuss the proof strategy for \thmref{thm:error_ERM}. 
Having established these core components,
we can now summarize
the proof strategy for \thmref{thm:error_ERM}. 
We bound the population error on clean data  
(the term on the LHS of \eqref{eq:thm1}) in three steps: 
(i) use \lemref{lem:fit_mislabeled} to upper bound 
the error on clean distribution $\error_\calD(\widehat f)$, 
by the error on mislabeled training data $\error_{\wt \calS_M}(\widehat f)$; 
(ii) approximate $\error_{\wt \calS_M}(\widehat f)$ 
by $\error_{\wt \calS_C}(\widehat f)$ 
and the error on randomly labeled training data 
(i.e.,  $\error_{\wt \calS}(\widehat f)$) using \lemref{lem:mislabeled_error}; 
and (iii) use \lemref{lem:clear_error} 
to estimate $\error_{\wt \calS_C}(\widehat f)$ 
using the error on clean training data 
($\error_{\calS}(\widehat f)$).
% \todos{Not sure if we should use the mathematical notation for ease?}
% 
% The generalization gap of the ERM classifier $\widehat f$ on clean data, i.e. $\error_\calD(\widehat f) - \error_\calS(\widehat f)$, can be 
%   
% Discuss sources of error from approximating ...
% 

% with uniform label marginal. 

\paragraph{Comparison with Rademacher bound}
Our bound in \thmref{thm:error_ERM}  
% illustrates that we can obtain an upper bound
shows that we can upper bound
% on
the clean population error of a classifier 
by estimating 
its accuracy on the clean and 
randomly labeled portions of the training data. 
% of randomly portion added during training. 
Next, we show that our bound's
% the 
dominating term 
% in our bound 
is upper bounded by the \emph{Rademacher complexity} \citep{shalev2014understanding},
a standard distribution-dependent complexity measure.

\begin{prop} \label{prop:rademacher}
    Fix a randomly labeled dataset $\wt S \sim \wt \calD^m$. 
    Then for any classifier $f \in \calF$ 
    (possibly dependent on $\wt S$)\footnote{We 
    restrict $\calF$ to functions which output
    a label in $\calY = \{-1,1\}$.} 
    and for any $\delta >0$, 
    with probability at least $1-\delta$ 
    over random draws of $\wt S$,
    we have 
    \begin{align*}
        1 - 2\error_{\wt S}(f) \le \Expt{\epsilon,x}{\sup_{f\in \calF} \left( \frac{\sum_i \epsilon_i f(x_i)}{m} \right) } + \sqrt{\frac{2\log(\frac{2}{\delta})}{m}},
    \end{align*} 
    % Proof fo
    where $\epsilon$ is drawn 
    from a uniform distribution over $\{-1, 1\}^m$ 
    and $x$ is drawn from $\calD_\calX^m$.      
\end{prop}

In other words, the proposition above highlights that the accuracy 
on the randomly labeled data 
% added during training 
is never larger than the Rademacher complexity of $\calF$ 
w.r.t. the underlying distribution over $\calX$, 
implying that our bound is never looser 
than a bound based on Rademacher complexity.
The proof follows 
% simply 
by application 
of the bounded difference condition 
and McDiarmid’s inequality~\citep{mcdiarmid1989}. 
We now discuss extensions 
of \thmref{thm:error_ERM} 
to regularized ERM
% binary classification with non-uniform label marginal 
and multiclass classification.

\paragraph{Extension to regularized ERM} 
Consider any function $R: \calF \to \Real$,
e.g., a regularizer 
that penalizes some measure of complexity 
for functions in class $\calF$.
Consider the following regularized ERM:
\begin{align}
    \widehat f \defeq \argmin_{f\in \calF} \error_\calS(f) + \lambda R(f) \,, \label{eq:regularized_erm}
\end{align}
where $\lambda$ is a regularization constant. 
If the regularization coefficient is independent 
of the training data $S \cup \wt S$, 
then our guarantee (\thmref{thm:error_ERM}) holds.
% doesn't change. 
Formally,
%% Can remomve this if need be
\begin{theorem} \label{thm:error_regularized_ERM}
    For any regularization function $R$, 
    assume we perform regularized ERM 
    as in \eqref{eq:regularized_erm} on $S \cup \wt S$ 
    and obtain a classifier $\widehat f$.
    Then, for any $\delta > 0$, 
    with probability at least $1-\delta$, we have
    % over the random draws of datasets $\wt S$ and $S$, we have  
    % \begin{align}
        $\error_\calD(\widehat f)  \le \error_\calS(\widehat f) + 1 - 2 \error_{\wt\calS}(\widehat f) + \left(\sqrt{2} \error_{\wt S}(\widehat f)  + 2 + \frac{m}{2n}\right)  \sqrt{\frac{\log(1/\delta)}{m}} \,.$ % \label{eq:regularized_ERM}
    % \end{align}
    % for some constant $c \le 3.2 $.
\end{theorem}

A key insight here is 
that the proof of \thmref{thm:error_ERM} 
treats the clean data $S$ as fixed
and considers random draws of the mislabeled portion.
Thus a data-independent 
regularization function does not
%% SG: I removed this 
%%%%%
% Performance on the (fixed) $S$
% is just such a function $\calF \to \Real$.
% and the proof goes through 
% for the very reason 
% that our inequality holds
% for any such function.
%%%%%
%
% and thus applies
% point wise for any fixed sample of the clean data.
% (fixed) sample of the clean data.
% Thus, adding an additional $R$ to the loss 
alter our chain of arguments and hence,
% that does not depend on our sample
has no impact on the resulting inequality.
We prove this result formally in \appref{app:proof_erm}.

% % We prove the result in \appref{app:proof_erm}. 
% The key insight is that the result 
% in \lemref{lem:fit_mislabeled} holds %the same 
% %  establishes the main inequality underlying proof of \thmref{thm:error_ERM}. 
% when the regularization coefficient
% is independent of $S \cup \wt S$. 

%% Maybe we can drop this next corollary? 

We note one immediate corollary 
from \thmref{thm:error_regularized_ERM}:
when learning any function $f$ parameterized by $w$ 
with $L_2$-norm penalty on the parameters $w$, 
the population error with $\widehat f$ is determined by the error on the clean training data as long as the error on randomly labeled data is high (close to $1/2$).

\paragraph{Extension to multiclass classification} %\label{subsec:multiclass}
Thus far, we have addressed binary classification.
% we assumed a binary classification problem. 
% with uniform label marginal. 
% Now, we will relax this assumption.
We now extend these results to the multiclass setting.
% with balanced classes. 
% 
% Assume $p_i$ be the probabiltiy of positive class. 
% Similar to 
As before, we obtain datasets 
$S$ and $\wt S$. 
% Recall the random label assignment procedure. 
% Recall that the 
Here, random labels are assigned
uniformly among all classes.
% procedure assigns any label $y$ 
% to a randomly selected $p_\calD(y)$ 
% fraction of the unlabeled data points. 
% Before presenting the result, we introduce some more notation. For any dataset $S$, let $S^{(y)}$ denote the portion of dataset labeled $y$. 
\begin{theorem} \label{thm:multiclass_ERM}
    For any regularization function $R$, 
    assume we perform regularized ERM 
    as in \eqref{eq:regularized_erm} on $S \cup \wt S$ 
    and obtain a classifier $\widehat f$. 
    % Then for a binary classification problem any $\delta > 0$, with probability $1-\delta$ over the random draws of datasets $\wt S$ and $S$, we have  
    % \begin{align}
        % $\error_\calD(\widehat f) - \error_\calS(\widehat f) \le 1 - (\error_{\wt\calS ^{(1)}}(\widehat f) + \error_{\wt\calS ^{(-1)}}(\widehat f)) + c\sqrt{\frac{\log(1/\delta)}{m}} \,$, %\label{eq:unbalanced_ERM}
    % \end{align}
    % for some constant $c = \cdot$. More generally, 
    For a multiclass classification problem 
    with $k$ classes,
    for any $\delta >0$, with probability at least $1-\delta$,
    % over the random draws of datasets $\wt S$ and $S$, 
    we have
    \vspace{-10pt}
    \begin{align*}
        \error_\calD(\widehat f)  \le \error_\calS(\widehat f)  + (k-1) &\left(1 - \tfrac{k}{k-1} \error_{\wt\calS}(\widehat f)\right) \\
        &+ c\sqrt{\frac{\log(\frac{4}{\delta})}{2m}} \,,\numberthis \label{eq:multiclass_ERM}
    % \vspace{-20pt}
    \end{align*}
    for some constant $c \le (2k+\sqrt{k} + \frac{m}{n\sqrt{k}})$.
\end{theorem}

We first discuss the implications of \thmref{thm:multiclass_ERM}.
Besides empirical error on clean data, the dominating term in
% \eqref{eq:unbalanced_ERM} 
the above expression
is given by $(k-1)\left(1 - \frac{k}{k-1} \error_{\wt\calS}(\widehat f)\right)$.  
% If $\wt S$ were not included in the training set then the error $\error_{\wt\calS ^{(1)}}(f) + \error_{\wt\calS ^{(-1)}}(f)$ would be approximately $1$ for any $f$. 
For any predictor $f$ 
(not dependent on $\wt S$), 
the term $\error_{\wt\calS}(\widehat f)$ 
would be approximately $(k-1)/{k}$ 
and for $\widehat f$, 
the difference now evaluates to the accuracy 
of the randomly labeled data.  
% We have an extra 
% More generally for any multiclass classification problem, the dominating term in the inequality of \thmref{thm:multiclass_ERM} defines the fit. 
Note that for binary classification, 
% we obtain 
% RHS of
\eqref{eq:multiclass_ERM} 
simplifies to 
% the same bound as in
\thmref{thm:error_ERM}. 
% \todos{Shall we discuss this k-1 in the product makes the bound loose and why we expect a tighter bound in general?} 

% Again, the crux of the proof lies 
The core of our proof involves
% in 
obtaining an inequality 
similar to \eqref{eq:lemma1}. 
While for binary classification, 
we could upper bound $\error_{\wt \calS_M}$ 
with $1-\error_\calD$ 
(in the proof of \lemref{lem:fit_mislabeled}), 
for multiclass classification, 
error on the mislabeled data 
and accuracy on the clean data 
in the population 
are not so directly related.  
% related which allowed us replacing
% 
% Hence 
To establish an inequality 
analogous to \eqref{eq:lemma1},
%an inequality 
% similar to \eqref{eq:lemma1},
we break the error on the 
(unknown) mislabeled data 
into two parts: one term corresponds 
to predicting the true label on mislabeled data, 
and the other corresponds to predicting 
neither the true label 
nor the assigned (mis-)label.  
Finally, we relate these errors to their
population counterparts to establish 
an inequality similar to \eqref{eq:lemma1}. 

%% file: sections/04_linear.tex
% \subsection{Bounds for linear models}
In the previous section, 
we presented results with ERM on 0-1 loss. 
While minimizing the 0-1 loss is hard in general,
%due to the underlying non-convexity, 
these results provide important theoretical insights. 
In this section, %motivated by these findings, 
% extending those results,
we show parallel results for linear models 
trained with Gradient Descent (GD).

% squared error over a given training set $S$.
% % , i.e. 
% % % \begin{align}
% % $\wh w  \defeq \argmin_{w\in R^d} \, \Expt{(x,y) \sim \calS}{L(x,y;w)} \,$,  
% % % \end{align}
% % where $L(x,y;w) = (w^Tx - y)^2$. 
% More generally,

% But first,
To begin, we introduce the setup and some additional notation.
% the setup and notation we will be using throughout this section.
For simplicity, we begin discussion with binary classification with $\calX = \Real^d$. 
% Results for multiclass classification deferred to \appref{app:multiclass_linear}. 
Define a linear function $f(x; w) \defeq w^T{x}$ 
for some $w \in \Real^d$ and $x\in \calX$. %\todos{Change all the $R^d$ to $\Real^d$.}
% Given training set $S \defeq \{ (x_i, y_i)\}_{i=1}^n$,
Given training set $S$,
we suppose that the parameters of the linear function 
are obtained via gradient descent on 
the following $L_2$ regularized problem: 
% \begin{align}
%     \calL_S(w) \defeq \Expt{(x,y) \sim \calS}{(w^Tx - y)^2} + \lambda \norm{w}{2} \,, \label{eq:l2_MSE}   
% \end{align}
% \begin{align}
%     \calL_S(w) \defeq \sum_{(x_i,y_i) \sim \calS}{(w^Tx_i - y_i)^2} + \lambda \norm{w}{2} \,, \label{eq:l2_MSE}   
% \end{align}
\begin{align}
    % n in denominator is avoided deliberately
    \calL_S(w; \lambda) \defeq \sum_{i=1}^n{(w^Tx_i - y_i)^2} + \lambda \norm{w}{2}^2 \,, \label{eq:l2_MSE}   
\end{align}
where $\lambda\ge0$ is a regularization parameter. 
%  i.e., $\widehat w \defeq \argmin_{w\in R^d} \, \calL_S(w)$.  
% In this section, we will only focus on binary classification problem. 
% We restrict 
% attention to binary classification, noting that
% results on multiclass classification can be obtained 
% as in \secref{subsec:multiclass}.  
% 
% \wh w   \argmin_{w\in R^d}
% 
% 
% justfiy the use of squared loss by citing belkin's recent papers nad NTK work.
% 
Our choice to analyze squared loss minimization for linear networks 
is motivated in part by its analytical convenience, and follows
recent theoretical work which analyze 
neural networks trained via squared loss minimization
in the Neural Tangent Kernel (NTK) regime when 
they are well approximated by linear networks~\citep{jacot2018neural,arora2019fine,du2019gradient,hu2019simple}. 
% Moreover, in practice, it has been shown 
Moreover, recent research suggests that for classification tasks,
squared loss minimization performs comparably to cross-entropy loss 
minimization
\citep{muthukumar2020classification,hui2020evaluation}. 
% 
% 
% \todos{Not sure if people put this, check recent papers}
% 
% Now we introduce the main result of this section. 
%

For a given training set $S$, 
% we define a perturbed training set $S^\prime_{(i)}$ 
% with the $i^\text{th}$ point replaced 
% by a new sample $(x,y) \sim \calD$. 
we use $S_{(i)}$ to denote the training set $S$ 
with the $i^{\text{th}}$ point removed.
We now introduce one stability condition:

\begin{condition}[Hypothesis Stability] 
    \label{cond:hypothesis_stability}
    We have $\beta$ hypothesis stability 
    if our training algorithm $\calA$ satisfies the following for all $i \in \{1,2,\ldots, n\}$: 
    \begin{align*}
    % ${\sum_{i=1}^n \frac{\error_{\calD}( f(\calA, S_{(i)}))}{n} - \error_\calD(f(\calA, S))} \le \beta\,$.
    \Expt{\calS, (x,y) \in \calD}{ \abs{\error\left( f(x) ,y  \right) - \error\left( f_{(i)}(x), y \right) }} \le \frac{\beta}{n} \,,
    \end{align*}
    where $f_{(i)} \defeq f(\calA, S_{(i)})$ and $ f \defeq f(\calA, S)$.
\end{condition}

% Define Leave-One-Out (LOO) error as $$\error_{\LOO} = \sum_{i=1}^n \error_{\calD}(\widehat f_{(i)})/n$$
This condition is similar to a notion of stability 
called \emph{hypothesis stability} 
\citep{bousquet2002stability,kearns1999algorithmic,elisseeff2003leave}. 
Intuitively, \codref{cond:hypothesis_stability} states 
that 
%  show the concentration 
% of the 
empirical leave-one-out error 
and average population error of leave-one-out classifiers
are close.
% 
% the average population error of classifiers, 
% each obtained by leaving one point out in the training set, 
% cannot be much larger 
% than the population error 
% of the original classifier.
% \todos{Fix this?}
This condition is mild 
and does not guarantee generalization. 
We discuss the implications in more detail in \appref{app:discuss_cond1}.

Now we present the main result of this section. 
As before, we assume access to a clean dataset 
$S = \{(x_i, y_i)\}_{i=1}^n \sim \calD^n$ 
and randomly labeled dataset 
$\wt S = \{(x_i, y_i)\}_{i=n+1}^{n+m} \sim \wt \calD^m$. 
Let $\bX = [x_1, x_2, \cdots, x_{m+n}]$ 
and $\by = [y_1, y_2, \cdots, y_{m+n}]$. 
Fix a positive learning rate $\eta$ such that 
$\eta \le 1/\left(\norm{\bX^T\bX}{\text{op}} + \lambda^2\right)$ 
and an initialization $w_0 = 0$. 
% \todos{Assumption made for simplicty}. 
Consider the following gradient descent iterates 
to minimize objective \eqref{eq:l2_MSE} on $S \cup \wt S$:
\begin{align}
w_t = w_{t-1} - \eta \grad_w \calL_{S \cup \wt S} (w_{t-1}; \lambda) \quad \forall t=1,2,\ldots \,. \label{eq:GD_iterates}
\end{align} 
Then we have $\{ w_t\}$ converge to the limiting solution 
$\wh w = \left( \bX^T\bX+\lambda \boldsymbol{I}\right)^{-1}\bX^T\by$. Define $\widehat f (x) \defeq f(x ; \wh w) $.  

%%% Change the theorem statement to only have the statement.
% \begin{theorem} \label{thm:linear}
%     % Fix a positive learning rate 
%     % $\eta \le \frac{1}{\norm{\bX^T\bX}{\text{op}} + \lambda^2}$ 
%     % and an initialization $w_0 = 0$. 
%     % % \todos{Assumption made for simplicty}. 
%     % Consider the following gradient descent iterates 
%     % to minimize objective \eqref{eq:l2_MSE} on $S \cup \wt S$:
%     % % \begin{align}
%     % $w_t = w_{t-1} - \eta \grad_w \calL_{S \cup \wt S} (w_{t-1}; \lambda) \,,$ for all $t=1,2,\ldots\,.$ %\quad \forall t=1,2,\ldots
%     % % \end{align} 
%     % Then $\{ w_t\}$ converge to a limiting solution 
%     % $\wh w = \left( \bX^T\bX+\lambda \boldsymbol{I}\right)^{-1}\bX^T\by$. 
%     % Define $\widehat f (x) \defeq f(x ; \wh w) $.  
%     % 
%     Assume that this gradient descent algorithm satisfies \codref{cond:error_stability}
%     with $\beta=\calO(\frac{1}{n+m})$.  
%     Then for any $\delta >0$, with probability at least $1-\delta$ 
%     over the random draws of datasets $\wt S$ and $S$, we have:
%     % \begin{align}
%         $\error_\calD(\widehat f) - \error_\calS(\widehat f) \le 1 - 2 \error_{\wt\calS}(\widehat f) + 3.5 \sqrt{\frac{\log(1/\delta)}{m}} + \calO\left(1/(n+m)\right) \,.$ %\label{eq:gd_error}
%     % \end{align} 
%     % for some constant $c\le 3.2$.
% \end{theorem}
% 
\begin{theorem}\label{thm:linear}
    Assume that this gradient descent algorithm satisfies \codref{cond:hypothesis_stability}
    with $\beta=\calO(1)$.  
    Then for any $\delta >0$, with probability at least $1-\delta$ 
    over the random draws of datasets $\wt S$ and $S$, we have:
    \begin{align*}
        \error_\calD(\widehat f) \le \error_\calS(\widehat f) &+ 1 - 2 \error_{\wt\calS}(\widehat f) + \sqrt{\frac{4}{\delta}\left(\frac{1}{m} +\frac{3\beta}{m+n} \right)} \\
        & + \left(\sqrt{2}\error_{\wt\calS}(\widehat f) + 1 + \frac{m}{2n} \right) \sqrt{\frac{\log(4/\delta)}{m}} \,. \label{eq:gd_error} \numberthis
    \end{align*} 
    % for some constant $c\le 3.2$.
\end{theorem}
% 
% \todos{Make theorem as direct as possible. Don't define things in the theorem.}
% Recall, . 
With a mild regularity condition, 
we establish the same bound 
on GD training with squared loss, 
notably the same dominating term on the population error, 
as in \thmref{thm:error_ERM}. In \appref{app:multiclass_linear}, we present the extension to multiclass classification, where we again obtain a result parallel to \thmref{thm:multiclass_ERM}. 

% The proof of this lemma mirrors the proof of lemma 1. 
% this is the technical difficulty different from above. 
\begin{hproof}
    Because squared loss minimization  
    does not imply 0-1 error minimization, 
    we cannot use arguments 
    from \lemref{lem:fit_mislabeled}. 
    This is the main technical difficulty.  
    To compare the 0-1 error at a train point with an unseen point, 
    we use the  closed-form expression for $\widehat{w}$.
    We show that the train error on mislabeled points  
    is less than the population error on the distribution of mislabeled data
    (parallel to \lemref{lem:fit_mislabeled}). 
    % Rest of the proof follows similar to \thmref{thm:error_ERM}. 
    
    % We can use arguments from the previous section to obtain parallel inequalities on squared error. But 
    % Because 
    % We can't simply use our 0-1 loss style arguments. Rather, 
    For a mislabeled 
    % unseen
    % (and an unknown) 
    training point $(x_i, y_i)$ in $\wt S$, we show that
    \begin{align}
        % \sum_{i=n+1}^m P_{x_i \sim \calD_\calX}(y_i x_i^T \widehat w \le 0 ) \le \sum_{i=n+1}^m P_{x_i \sim \calD_\calX}(y_i x_i^T \widehat w_{(i)} \le 0 )
        \indict{y_i x_i^T \wh w \le 0 } \le \indict{y_i x_i^T \wh w_{(i)} \le 0 } \,, \label{eq:loo_error}
    \end{align}
    where $\wh w_{(i)}$ is the classifier obtained
    by leaving out the $i^\text{th}$ point from the training set. 
    Intuitively, this condition
    states
    that the train error at a training point 
    is less than the leave-one-out error at that point, i.e. the error obtained 
    by removing that point and re-training.
    % the classifier. 
    % We 
    % then use \codref{cond:train_stability} 
    % to show an exponential convergence 
    % to expected train error 
    % (average of LHS in \eqref{eq:loo_error} for all $i$) 
    % and 
    % then
    % use
    Using \codref{cond:hypothesis_stability}, we then relate the average leave-one-out error 
    (over the index $i$ of the RHS in \eqref{eq:loo_error}) 
    to the population error on the mislabeled distribution 
    %with the original classifier $\wh w$ 
    to obtain an inequality similar to~\eqref{eq:lemma1}.
    % We complete the proof with showing concentration of empr
\end{hproof}

\paragraph{Extensions to kernel regression}
% Extention to kernel regression
% Extension to NTK 
% Holds for deep models when they are in NTK regime. 
% 
Since the result in \thmref{thm:linear} 
% agnostic 
does not impose any regularity conditions 
on the underlying distribution 
over $\calX\times \calY$, 
our guarantees %simply 
extend straightforwardly to kernel regression 
by using the transformation $x\to \phi(x)$ 
for some feature transform function $\phi$. 
Furthermore, recent literature 
has pointed out a concrete connection 
between neural networks 
and kernel regression 
with the so-called 
\emph{Neural Tangent Kernel} (NTK) 
which holds in a certain regime
where weights do not change much 
during training~\citep{jacot2018neural,du2019gradient,du2018gradient,chizat2019lazy}. 
% Unlike linear models, 
% % note that in the NTK regime 
% in the NTK regime, 
% even though the model is linear in parameters, 
% it is still non-linear in the inputs.  
% Our work contributes to this line of work, in particular
% 
Using this concrete correspondence, 
our bounds on the clean population error 
(\thmref{thm:linear}) 
extend to wide neural networks 
operating in the NTK regime.  %% MOve this before end 

% In \secref{sec:exp}, we experimentally demonstrate our 

\paragraph{Extensions to early stopped GD} 
% \todos{Can move this whole to Appendix?}
Often in practice, gradient descent is stopped early. 
% instead of running uptill convergence.
% 
We now provide theoretical evidence 
that our guarantees 
may continue to hold 
for an early stopped GD iterate. Concretely, we show that in expectation,
the outputs of the GD iterates are close 
to that of a problem with data-independent regularization (as considered in~Theorem~\ref{thm:error_regularized_ERM}).
% 
% Before presenting the result
First,
we introduce some notation.  
By $\calL_{S}(w)$, we denote 
the objective in \eqref{eq:l2_MSE} with $\lambda=0$. 
Consider the GD iterates defined in \eqref{eq:GD_iterates}. 
Let $\wt w_{\lambda} = \argmin_{w} \calL_\calS (w;\lambda)$. 
% ) \,,$ for all $t=1,2,\ldots\,.$ 
% 
Define $f_t(x) \defeq f(x;w_t)$ as the solution at the $t^{\text{th}}$ iterate
and $\wt f_\lambda(x) \defeq f(x;\wt w_\lambda)$ as the regularized solution. 
Let $\kappa$ be the condition number of the population covariance matrix 
and let $s_\text{min}$ be the minimum positive singular value 
of the empirical covariance matrix. 
% 
% Next, 

\begin{prop}[informal] \label{prop:early_stop}
    % Assume the underlying distribution over $\calX$ as $N(0,\sigma^2 I)$. 
% Let $S$ be a training data sampled iid from $\calD$. 
% 
For $\lambda = \frac{1}{t\eta}$, we have 
\begin{align*}
    \Expt{x \sim \calD_\calX}{(f_t(x) - \wt f_\lambda(x))^2} &\le c(t,\eta) \cdot \Expt{x \sim \calD_\calX}{f_t(x)^2} \,, %\label{eq:early_stop}
\end{align*}
where $c(t, \eta) \approx \kappa \cdot \min( 0.25, \frac{1}{s_\text{min}^2 t^2 \eta^2})$.
% $s_\text{min}$ denotes the minimum positive singular value 
% of the empirical covariance respectively,
% and $\kappa$ denotes the condition number
% of the true covariance. 
An equivalent guarantee holds for a point $x$ 
sampled from the training data. 
\end{prop} 
% \todos{Siva, I will revisit this nicely after reading Tsybakov’s condition. 
% Do you think we should define Tsybakov’s condition in the main paper.}
% with the regularization constant dependent on . 

% We discuss this in detail in \appref{app:proof_gd}. 

% \textbf{Remark \quad} 

The proposition above states 
that for large enough $t$, 
GD iterates stay close
to a regularized solution 
with data-independent 
regularization constant. 
Together with our guarantees in \thmref{thm:linear} 
for regularization solution with $\lambda = \frac{1}{t\eta}$, 
\propref{prop:early_stop} shows that our guarantees with RATT 
may hold on early stopped GD. 
% We include 
See the formal result 
% and a detailed discussion 
in \appref{app:formal_early_stop}.

\textbf{Remark {} {}} 
\propref{prop:early_stop} only bounds the expected squared difference between the $t^{\text{th}}$ gradient descent iterate and a corresponding regularized solution. 
% to obtain the bound on the classification error and this can be done using 
The expected squared difference
and the expected difference
of classification errors
(what we wish to bound)
are not related, in general.
However, they can be related 
under standard low-noise 
(margin) assumptions.
For instance, under the Tsybakov noise condition~\citep{tsybakov1997nonparametric, yao2007early}, we can lower-bound 
the expression on the LHS 
of \propref{prop:early_stop} 
with the difference 
of expected classification error. 

\paragraph{Extensions to deep learning} 
Note that the main lemma underlying
our bound on (clean) population error  
states that when training on a
mixture of clean and randomly labeled data, 
we obtain a classifier whose empirical error 
on the mislabeled training data
is lower than its population error 
on the distribution of mislabeled data.  
% While we prove this in \lemref{lem:fit_mislabeled}
% for ERM on 0-1 loss, for linear models,
% and for models in the NTK regime,
% we must \emph{assume it} 
% to extend our bound to deep models. 
We prove this for ERM on 0-1 loss (\lemref{lem:fit_mislabeled}). 
For linear models 
(and networks in NTK regime), 
we obtained this result
by assuming hypothesis stability 
and relating training error at a datum
with the leave-one-out error (\thmref{thm:linear}). 
However, to extend our bound to deep models 
we must assume that training 
on the mixture of random and clean data 
leads to overfitting on the random mixture. Formally:  

\begin{assumption} \label{asmp:deep_models}
Let $\wh f$ be a model obtained 
by training with an algorithm $\calA$ 
on a mixture of clean data $S$ 
and randomly labeled data $\wt S$. 
Then with probability $1-\delta$ 
over the random draws 
of mislabeled data $\wt S_M$, 
we assume that the following condition holds:
\begin{align*}
    \error_{\wt \calS_M} (\wh f) \le \error_{\calDm} (\wh f) + c\sqrt{\frac{\log(1/\delta)}{2m}}  \,,
\end{align*}
for a fixed constant $c > 0$.
\end{assumption}

Under \asmpref{asmp:deep_models}, 
our results in \thmref{thm:error_ERM},~\ref{thm:error_regularized_ERM} and~\ref{thm:multiclass_ERM} 
extend beyond ERM with the 0-1 loss 
to general learning algorithms. 
We include the formal result in \appref{appsubsec:ext_DL}. 
% This assumption is directly assuming hypothesis  
% 
% This assumption 
% in \thmref{thm:error_ERM} 
% By way of justification,
% We now justify as
Note that given the ability of neural networks 
to interpolate the data, 
this assumption seems uncontroversial 
in the later stages of training. 
Moreover, concerning the early phases of training,
recent research has shown that 
learning dynamics for complex deep networks 
resemble those for linear models 
\citep{nakkiran2019sgd, hu2020surprising},
much like the wide neural networks 
that we do analyze.
Together, these arguments help 
to justify \asmpref{asmp:deep_models} and hence, the applicability
of our bound in deep learning.
Motivated by our analysis on linear models trained with gradient descent, 
we discuss conditions in \appref{appsubsec:justifying_assumption1} which imply \asmpref{asmp:deep_models} for constant values $\delta > 0$.
In the next section, we empirically demonstrate applicability of our bounds for deep models.

%% file: sections/05_exp.tex
\begin{figure*}[t!]
    \centering 
    % \vspace{-15pt}
    % \includegraphics[width=0.9\linewidth]{example-image-a}
    \subfigure[]{\includegraphics[width=0.3\linewidth]{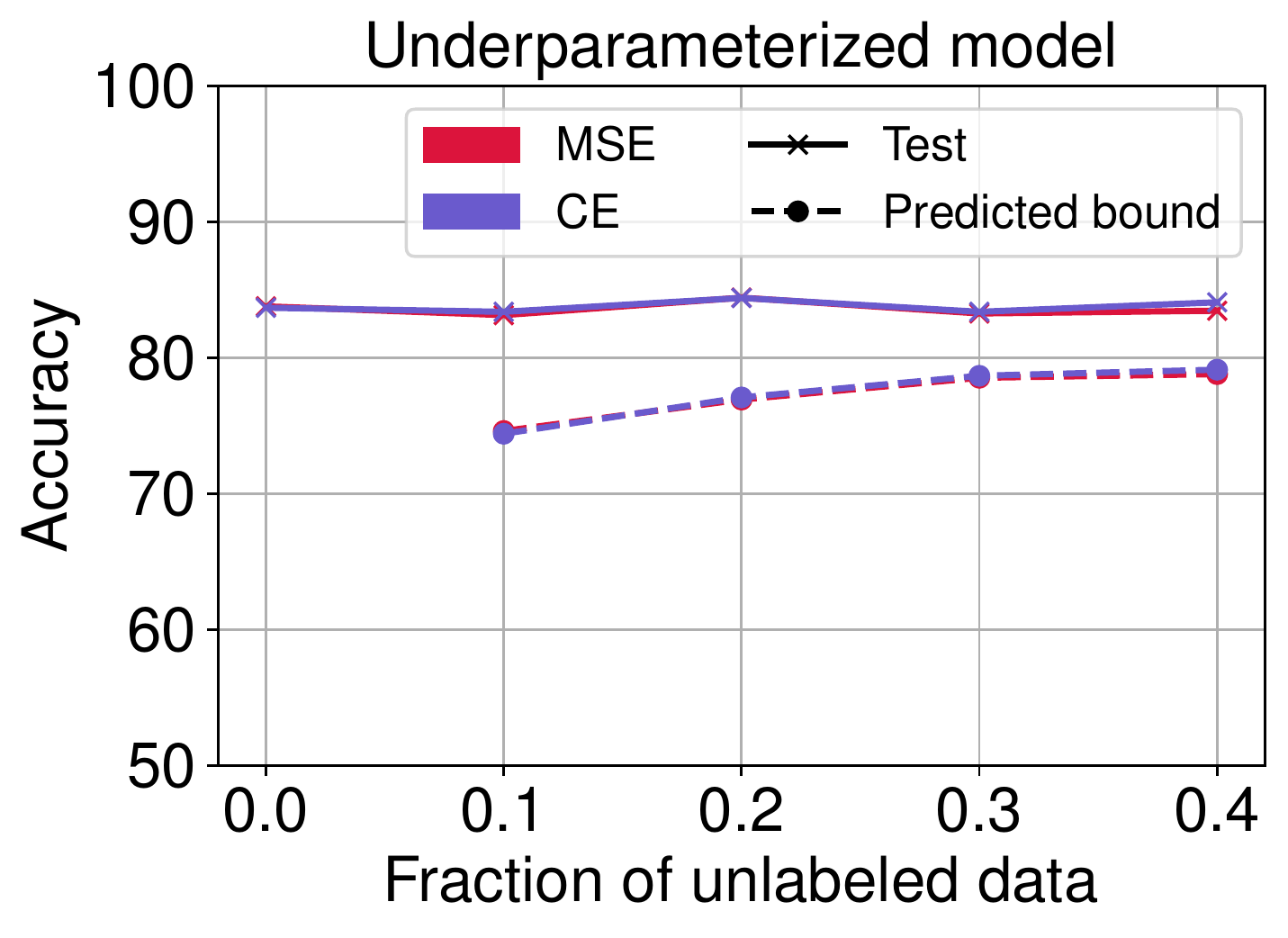}}\hfil
    \subfigure[]{\includegraphics[width=0.3\linewidth]{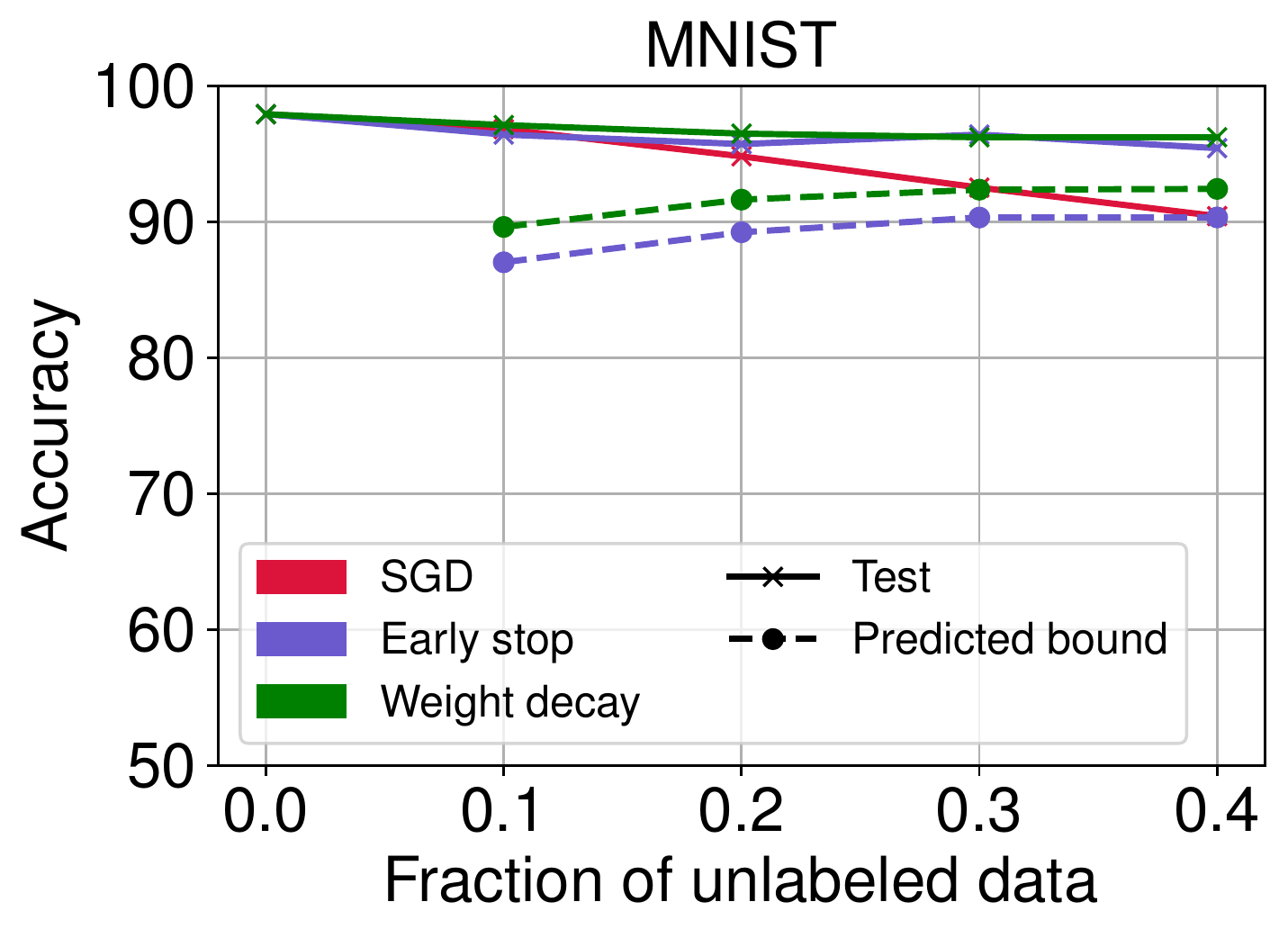}}\hfil
    \subfigure[]{\includegraphics[width=0.3\linewidth]{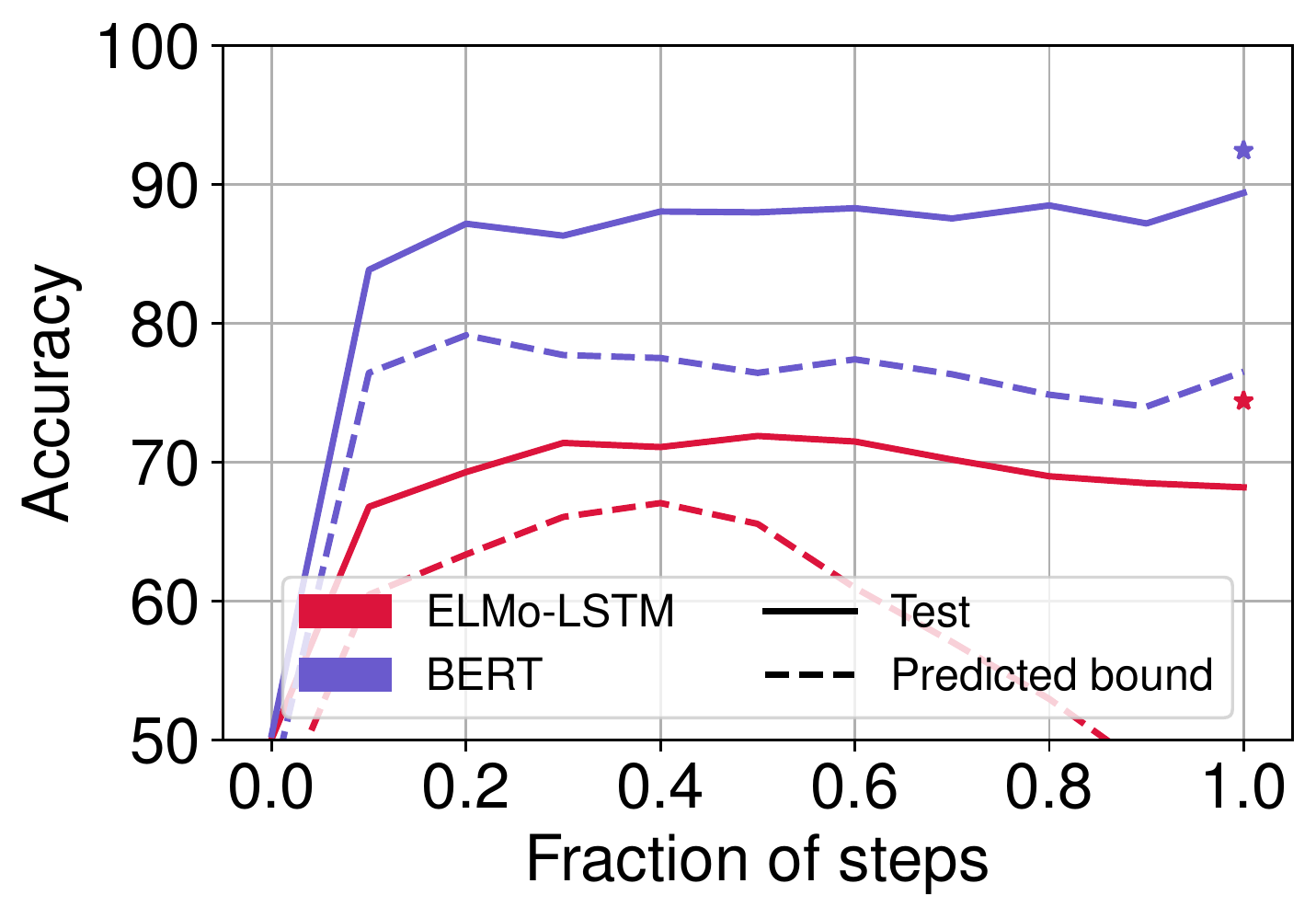}}
    \vspace{-5pt}
    \caption{ 
    % Predicted lower bound 
    % on different
    We plot the accuracy and corresponding bound 
    (RHS in \eqref{eq:erm}) at $\delta = 0.1$. 
    for binary classification tasks. 
    Results aggregated over $3  $ seeds. 
    % i.e., $1-\error$ where $\error$ is the term in the RHS of \eqref{eq:erm}
    (a) Accuracy vs fraction of unlabeled data (w.r.t clean data) 
    in the toy setup with a linear model trained with GD. 
    (b) Accuracy vs fraction of unlabeled data 
    for a 2-layer wide network trained with SGD on binary MNIST. 
    With SGD and no regularization (red curve in (b)),
    we interpolate the training data 
    and hence the predicted lower bound is $0$. 
    However, with early stopping (or weight decay) 
    we obtain tight guarantees. 
    (c) Accuracy vs gradient iteration on IMDb dataset 
    with unlabeled fraction fixed at 0.2. 
    In plot (c), `*' denotes the best test accuracy 
    with the same hyperparameters 
    and training only on clean data. 
    % Exact hyperparameters are included in
    See \appref{app:exp} for exact hyperparameter values. }
    \label{fig:error_binary}
    \vspace{-5pt}
\end{figure*}

Having established our framework theoretically, 
we now demonstrate its utility experimentally.
First, for linear models and wide networks 
in the NTK regime where our guarantee holds,
we confirm that our bound is not only valid, 
but closely tracks the generalization error.
Next, we show that in practical deep learning settings,
optimizing cross-entropy loss by SGD,
the expression for our (0-1) ERM bound 
nevertheless tracks test performance closely
and in numerous experiments
on diverse models and datasets
is never violated empirically. 
\vspace{5pt}

\noindent \textbf{Datasets {} {}} 
To verify our results on linear models,
we consider a toy dataset,
where the class conditional distribution $p(x|y)$
for each label is Gaussian.
For binary tasks, we use binarized CIFAR-10 (first 5 classes vs rest)~\citep{krizhevsky2009learning}, 
binary MNIST (0-4 vs 5-9)~\citep{lecun1998mnist}
and IMDb sentiment analysis dataset~\citep{maas2011learning}. 
For multiclass setup, we use MNIST and CIFAR-10.   
% We also consider a toy dataset 
% with Gaussian features 
% To verify our results on linear models,
% we consider a toy dataset,
% where the class conditional distribution $p(x|y)$
% for each label is Gaussian.
\vspace{5pt}

\noindent \textbf{Architectures {} {}} 
To simulate the NTK regime, 
we experiment with 2-layered wide networks 
both (i) with the second layer fixed 
at random initialization;
(ii) and updating both layers' weights.
% In overparameterized linear nets,
For vision datasets (e.g., MNIST and CIFAR10),
we consider (fully connected) 
multilayer perceptrons (MLPs) 
with ReLU activations
and ResNet18~\citep{he2016deep}. 
For the IMDb dataset, we train Long Short-Term Memory Networks 
(LSTMs;~\citet{hochreiter1997long}) with ELMo embeddings~\citep{Peters:2018} 
and
% and fine-tune an off-the-shelf uncased 
fine-tune an off-the-shelf uncased BERT model
\citep{devlin2018bert, wolf-etal-2020-transformers}. 
\vspace{5pt}

\noindent \textbf{Methodology {} {}} 
% To obtain our bound on the 
To bound the population error,
we require access to both 
clean
% training data
and unlabeled data.
For toy datasets, we obtain unlabeled data 
by sampling from the underlying distribution over $\calX$. 
For 
% real-world 
image and text datasets, 
we hold out a small fraction 
of the clean training data 
and discard their 
labels
to simulate unlabeled data.
% underlying true labels 
% to obtain an unlabeled dataset. 
We use the random labeling procedure 
% as explained 
described
in \secref{sec:setup}. 
After augmenting clean training data 
with randomly labeled data, 
we train in the standard fashion. 
% use the standard training procedure.  
% For each experiment,
% we report architecture, hyperparameters,
% and training details in
% We report  in
% we report the exact architecture 
% with hyperparameters and training procedure 
% in
% \tabref{table:experiments} 
See \appref{app:exp} for experimental details. 
%%% Maybe add 

\paragraph{Underparameterized linear models}
On toy Gaussian data, 
we train linear models 
with GD to minimize cross-entropy loss 
and mean squared error.
% First, we vary
Varying the fraction of randomly labeled data 
% added 
% and
we observe that the accuracy 
on clean unseen data 
is barely impacted
% remains almost the same 
(\figref{fig:error_binary}(a)). 
This highlights that in low dimensional models 
adding randomly labeled data 
with the clean dataset (in toy setup) 
has minimal effect on the performance on unseen clean data. 
Moreover, we find that RATT 
offers a tight lower bound 
on the unseen clean data accuracy. 
We observe the same behavior with 
Stochastic Gradient Descent (SGD) training 
(ref. \appref{app:exp}).   
Observe that the predicted bound goes up as 
the fraction of unlabeled data increases. 
While the accuracy as dictated by 
the dominating term in the RHS of \eqref{eq:thm1} 
decreases with an increase in the fraction 
of unlabeled data, 
we observe a relatively sharper 
decrease in $\mathcal{O}_p\left({1/\sqrt{m}}\right)$ 
term of the bound, leading to an overall increase 
in the predicted accuracy bound. 
In this toy setup, we also evaluated a kernel 
regression bound from \citet{bartlett2002rademacher} (Theorem 21), 
however, the predicted kernel regression bound remains vacuous.

% \textbf{Overparameterized models {} {}}
\paragraph{Wide Nets} 
Next, we consider MNIST binary classification 
with a wide 2-layer fully-connected network.  
In experiments with SGD training on MSE loss
without early stopping 
or weight decay regularization, 
we find that adding extra randomly label data 
hurts the unseen clean performance 
(\figref{fig:error_binary}(b)). 
Additionally, due to the perfect fit 
on the training data, 
our bound is rendered vacuous. 
However, with early stopping (or weight decay), 
we observe close to zero performance difference 
with additional randomly labeled data.
% , i.e.,
% Instead,
% , the degradation depends 
% on the fraction of randomly labeled data,
% and at small fractions,
% We observe . 
Alongside, we obtain tight bounds 
on the accuracy on unseen clean data 
paying only a small price to negligible
for incorporating randomly labeled data. 
% And since with increase 
% in fraction of unlabeled data 
Similar results hold for SGD and GD 
and when cross-entropy loss 
is substituted for MSE (ref. \appref{app:exp}).  

\paragraph{Deep Nets} 
We verify our findings on 
(i) ResNet-18 and 5-layer MLPs
trained with binary CIFAR  
(\figref{fig:error_CIFAR10}); 
and (ii) ELMo-LSTM and BERT-Base models
fine-tuned on the IMDb dataset 
(\figref{fig:error_binary}(c)). 
% Parallel results with deep models 
% on binary MNIST are in
See \appref{app:exp} for additional results 
with deep models on binary MNIST. 
We fix the amount of unlabeled data 
at $20\%$ of the clean dataset size 
and train all models 
% the standard fashion
% in a conventional manner 
with standard hyperparameters. 
Consistently, we find 
that our predicted bounds
are never violated in practice.
% Consistently, we observe that the our bound 
% lower bounds the accuracy on the unseen clean data. 
And as training proceeds, 
the fit on the mislabeled data increases 
with perfect overfitting 
in the interpolation regime 
rendering our bounds vacuous. 
However, with early stopping, 
our bound predicts test performance closely. 
For example, on IMDb dataset with BERT fine-tuning 
we predict $79.8$ as the accuracy of the classifier, 
when the true performance is $88.04$ 
(and the best achievable performance on unseen data is $92.45$). 
Additionally, we observe 
that our method tracks the performance
from the beginning of the training 
and not just towards the end.
%% SG: Add numbers and more takeaways

Finally, we verify our multiclass bound
on MNIST and CIFAR10
with deep MLPs and ResNets
% although our bound (\thmref{thm:multiclass_ERM})
% requires an assumption,
% it is never violated empirically
(see results in \tabref{table:multiclass}
and per-epoch curves in \appref{app:exp}). 
As before, 
we fix the amount of unlabeled data 
at $20\%$ of the clean dataset 
to minimize cross-entropy loss via SGD.
In all four settings, our bound predicts 
non-vacuous performance on unseen data. 
In \appref{app:exp}, we investigate 
our approach on CIFAR100 
showing that even though 
our bound grows pessimistic 
with greater numbers of classes, 
the error on the mislabeled data nevertheless tracks population accuracy. 
% we evaluate our bound on CIFAR100. 
% Although, RATT renders vacuous bounds on CIFAR100, 
% we show that the error on mislabeled portion 
% (not identifiable with just the randomly labeled data)
% can nevertheless tightly track population accuracy on clean data (see Appendix).
% 
% We discuss this in detail in appendix. 
% This shows that while 
% In addition, we   
% on the combined data. 
% Our guarantees on MNIST and CIFAR10 are obtained 
% by trading-off for a small performance gap 
% due to early stopping and additional unlabeled data 
% (the difference between Test Acc. 
% and Best Acc. in \tabref{table:multiclass}).  

% With substantial weight decay, we can avoid any fit to the random data; highlight the usefulness of learning with high-learning rate and cite

\begin{table}[t]
    \centering
    \small
    \tabcolsep=0.12cm
    \begin{tabular}{@{}*{5}{c}@{}}
    \toprule
    Dataset & Model  & \thead{Pred. Acc}  & \thead{Test Acc.} & \thead{Best Acc.}\\
    % Dataset & Model & \thead{Test \\accuracy} & \thead{Predicted  \\bound} & \thead{Test accuracy \\(clean training)}\\
    \midrule
    % \multirow{2}{*}{MNIST}  & MLP & $95.37 - 10.3$ &   & \\
    \multirow{2}{*}{MNIST}  & MLP & $93.1$ & $97.4$  & $97.9$\\
    & ResNet & $96.8$ & $98.8$  & $98.9$ \\
    \midrule
    \multirow{2}{*}{CIFAR10}  & MLP & $48.4$  & $54.2$  & $60.0$ \\
    & ResNet & $76.4$  &  $88.9$  & $92.3$ \\
    \bottomrule 
    \end{tabular}  
    \caption{Results on multiclass classification tasks. With pred. acc. we refer to the dominating term in RHS of \eqref{eq:multiclass_ERM}. At the given sample size and $\delta=0.1$, the remaining term evaluates to $30.7$, decreasing our predicted accuracy by the same. We note that
    test acc. denotes the corresponding accuracy on unseen clean data. Best acc. is the best achievable accuracy with just training on just the clean data (and same hyperparamters except the stopping point). Note that across all tasks our predicted bound is tight and the gap between the best accuracy and test accuracy is small. Exact hyperparameters are included in \appref{app:exp}.  }\label{table:multiclass}
    \vspace{-10pt}
\end{table}

%% file: sections/06_discuss.tex
% Having presented our theoretical and experimental results, 
% We now situate our contribution 
% within the existing literature on deep learning. 

\paragraph{Implicit bias in deep learning} 
Several recent lines of research 
attempt to explain the generalization of neural networks 
despite massive overparameterization
via the \emph{implicit bias} of gradient descent
\citep{soudry2018implicit,gunasekar2018implicitb,gunasekar2018implicita,ji2019implicit,chizat2020implicit}. 
Noting that even for overparameterized linear models,
there exist multiple parameters 
capable of overfitting the training data
(with arbitrarily low loss), 
of which some generalize well
and others do not,
they seek to characterize the favored solution.
Notably, \citet{soudry2018implicit}
find that for linear networks,
gradient descent converges (slowly)
to the max margin solution.
A complementary line of work
focuses on the early phases of training,
finding both empirically~\citep{rolnick2017deep,arpit2017closer} 
and theoretically~\citep{arora2019fine,li2020gradient,liu2020early}
that even in the presence 
of a small amount of mislabeled data, 
gradient descent is biased 
to fit the clean data first
during initial phases of training. 
However, to best our knowledge,
no prior work leverages this phenomenon 
to obtain generalization guarantees on the clean data,
which is the primary focus of our work. 
Our method exploits this phenomenon 
to produce non-vacuous generalization bounds.
Even when we cannot prove \emph{a priori}
that models will fit the clean data well
while performing badly on the mislabeled data,
we can observe that it indeed happens (often in practice),
and thus, \emph{a posteriori}, 
provide tight bounds on the population error.
Moreover, by using regularizers
like early stopping or weight decay, 
we can accentuate this phenomenon,
enabling our framework to
provide even tighter guarantees.

\paragraph{Generalization bounds}
Conventionally, generalization in machine learning 
has been studied through the lens of uniform 
convergence bounds~\citep{blumer1989learnability, vapnik1999overview}.  
% Understanding generalization 
% in overparameterized networks  
% has attracted significant interests. 
Representative works on understanding generalization 
in overparameterized networks
within this framework
% of research 
include 
\citet{neyshabur2015norm,neyshabur2017pac,neyshabur2017exploring,neyshabur2018role,dziugaite2017computing, bartlett2017spectrally,arora2018stronger,li2018learning,zhou2018non, allen2019learning,nagarajan2019deterministic}. 
However, uniform convergence based 
bounds typically remain numerically 
loose relative to the true generalization error. 
Several works have also questioned the ability of 
uniform convergence based approaches to 
explain generalization in overparameterized models~\citep{zhang2016understanding,nagarajan2019uniform}. 
% 
% In light of the inapplicability of 
% traditional complexity-based bounds
% to deep neural networks 
% \citep{zhang2016understanding,nagarajan2019uniform},
% researchers have investigated alternative strategies
% to provide non-vacuous generalization bounds for deep nets
% \citep{neyshabur2015norm,neyshabur2017pac,neyshabur2017exploring,neyshabur2018role,dziugaite2017computing, bartlett2017spectrally,xu2017information, arora2018stronger,li2018learning, allen2019learning,pensia2018generalization, zhou2018non, nagarajan2019deterministic,
% nakkiran2020deep}.
% However, these bounds typically 
% remain numerically loose 
% relative to the true generalization error. 
% \
% \citet{dziugaite2017computing, zhou2018non} 
% provide non-vacuous generalization guarantees. 
% Specifically, they transform a base network
% into consequent networks 
% that do not interpolate the training data 
% either by adding stochasticity 
% to the network weights~\citep{dziugaite2017computing}
% or by compressing the original neural network~\citep{zhou2018non}.
% 
Subsequently, recent works have proposed 
unconventional ways to derive generalization bounds~\citep{negrea2020defense,zhou2020uniform}. 
In a similar spirit, we take departure from 
complexity-based approaches to generalization bounds  
in our work. 
In particular, we leverage unlabeled data 
to derive a post-hoc generalization bound.  
Our work provides 
guarantees on overparameterized networks 
by using early stopping or weight decay regularization, 
preventing a perfect fit on the training data.  
Notably, in our framework, the model
can perfectly fit the clean portion of the data,
so long as they nevertheless fit the mislabeled data poorly.

\paragraph{Leveraging noisy data to provide generalization guarantees}
In parallel work, \citet{bansal2020self} 
presented an upper bound 
on the generalization gap 
of linear classifiers 
trained on representations 
learned via self-supervision. 
Under certain noise-robustness 
and rationality assumptions 
on the training procedure, 
the authors obtained bounds 
dependent on the complexity 
of the linear classifier 
and independent of the complexity of representations. 
By contrast, 
we present generalization bounds 
for 
% the 
% usual end-to-end 
supervised learning 
% the standard supervised learning setting
that are non-vacuous by virtue 
of the early learning phenomenon. 
% While similar in spirit,
% Still, 
While both frameworks highlight
how robustness to random label corruptions 
can be leveraged to obtain bounds 
that do not depend directly
on the complexity of the underlying hypothesis class,
our framework, methodology, claims, 
and generalization results
are very different from theirs.

%% file: sections/07_prior.tex
A long line of work relates early stopped GD
to a corresponding regularized solution
\citep{friedman2003gradient, yao2007early, suggala2018connecting, ali2018continuous,neu2018iterate,ali2020implicit}. 
In the most relevant work, 
\citet{ali2018continuous} and \citet{suggala2018connecting} %where for regression tasks authors 
address a regression task, 
theoretically relating
the solutions of early-stopped GD 
and a regularized problem, 
obtained with a data-independent 
regularization coefficient.  
% 
% 
% To study generalization, 
Towards understanding generalization
numerous stability conditions 
have been discussed 
\citep{kearns1999algorithmic, bousquet2002stability, mukherjee2006learning, shalev2010learnability}. 
\citet{hardt2016train} studies the uniform stability property 
to obtain generalization guarantees with early-stopped SGD. 
%% ZL and SB: Shall we mention more comparison to this? I think this is relevant work to us. 
% 
While we assume a benign stability condition
to relate leave-one-out performance with population error, 
we do not rely on any stability condition
that implies generalization.

% since we do not rely on any stability condition that implies generalization,
% our work is different from this line of work.  
%  the 

% 
% OLD VERSION --- CACHED
% 

% % \textbf{Learning with noisy data}
% There is a long line of work relating early stopped GD with a corresponding regularized solution~\citep{friedman2003gradient, yao2007early, suggala2018connecting, ali2018continuous,neu2018iterate,ali2020implicit}. Closest to our work is~\citet{ali2018continuous} and~\citet{suggala2018connecting} where for regression tasks authors theoretically relate properties of an early stopped GD solution with a regularized solution, obtained with a data-independent regularization coefficient.  
% % 
% % 
% To study generalization, numerous stability conditions have been discussed~\citep{kearns1999algorithmic, bousquet2002stability, mukherjee2006learning, shalev2010learnability}. \citet{hardt2016train} studies the uniform stability property to obtain generalization guarantees with early-stopped SGD. %% ZL and SB: Shall we mention more comparison to this? I think this is relevant work to us. 
% While we assume a benign stability condition to relate expected leave-out-out performance with population error, since we do not rely on any stability condition that implies generalization, our work is different from this line of work.  
% %  the 

% generalization literature

% -- our work stepping stone in to understand generalization in overparameterized models (in region outside interpolation)

%% file: sections/08_conclusion.tex
% 
% Clean
% 
Our work introduces a new approach 
for obtaining generalization bounds
that do not directly depend on the 
underlying complexity of the model class. 
For linear models, we provably obtain a bound 
in terms of the fit on randomly labeled data added during training. 
Our findings raise a number of questions to be explored next. 
While our empirical findings and theoretical results with 0-1 loss 
hold absent further assumptions
and shed light on why the bound 
may apply for more general models,
we hope to extend our proof 
that overfitting (in terms classification error)
to the finite sample of mislabeled data
occurs with SGD training on broader classes 
of models and loss functions. 
We hope to build on some early results
\citep{nakkiran2019sgd, hu2020surprising} 
which provide evidence that deep models
behave like linear models 
in the early phases of training.  
We also wish to extend our framework
to the interpolation regime.
Since many important aspects of neural network learning 
take place within early epochs
\citep{achille2017critical,frankle2020early},
including gradient dynamics converging 
to very small subspace~\citep{gur2018gradient},
we might imagine operationalizing our bounds
in the interpolation regime
by discarding the randomly labeled data
after initial stages of training.

%% file: sections/appendix.tex
\onecolumn

% \tableofcontents{}

% \newpage

\section*{Supplementary Material}
\addcontentsline{toc}{section}{Supplementary Material}

Throughout this discussion, 
we will make frequently use 
of the following standard results
concerning the exponential concentration 
of random variables:

\begin{lemma}[Hoeffding's inequality for independent RVs~\citep{hoeffding1994probability}] Let $Z_1, Z_2, \ldots, Z_n$ be independent bounded random variables with $Z_i \in [a,b]$ for all $i$, then 
    \begin{align*}
        \prob\left( \frac{1}{n} \sum_{i=1}^n (Z_i - \Expo{Z_i}) \ge t \right) \le \exp{\left( -\frac{2nt^2}{(b-a)^2} \right) }
    \end{align*} 
    and 
    \begin{align*}
        \prob\left( \frac{1}{n} \sum_{i=1}^n (Z_i - \Expo{Z_i}) \le -t \right) \le \exp{\left( -\frac{2nt^2}{(b-a)^2} \right) }
    \end{align*} 
    for all $t \ge 0$. 
\end{lemma}

\begin{lemma}[Hoeffding's inequality for sampling with replacement~\citep{hoeffding1994probability}] \label{lem:hoeffding_sampling} Let $\calZ = (Z_1, Z_2, \ldots, Z_N)$ be a finite population of $N$ points with $Z_i \in [a.b]$ for all $i$. Let $X_1, X_2, \ldots X_n$ be a random sample drawn without replacement from $\calZ$. Then for all $t \ge 0$, we have 
    \begin{align*}
        \prob\left( \frac{1}{n} \sum_{i=1}^n (X_i - \mu ) \ge t \right) \le \exp{\left( -\frac{2nt^2}{(b-a)^2} \right) }
    \end{align*} 
    and 
    \begin{align*}
        \prob\left( \frac{1}{n} \sum_{i=1}^n (X_i - \mu ) \le -t \right) \le \exp{\left( -\frac{2nt^2}{(b-a)^2} \right) } \,,
    \end{align*} 
    where $\mu = \frac{1}{N} \sum_{i=1}^{N} Z_i$. 
\end{lemma}

We now discuss one condition that generalizes the exponential concentration to dependent random variables.
\begin{condition}[Bounded difference inequality] \label{cond:BDC} Let $\calZ$ be some set and $\phi: \calZ^n \to \Real$. We say that $\phi$ satisfies the bounded difference assumption if 
there exists $c_1, c_2, \ldots c_n \ge 0$ s.t. for all $i$, we have 
\begin{align*}
    \sup_{Z_1,Z_2, \ldots,Z_n, Z_i^\prime \in \calZ^{n+1} } \abs{\phi (Z_1, \ldots, Z_i, \ldots, Z_n ) - \phi (Z_1, \ldots, Z_i^\prime, \ldots, Z_n ) } \le c_i \,.
\end{align*} 
\end{condition}

\begin{lemma}[McDiarmid’s inequality~\citep{mcdiarmid1989}] \label{lem:McDiarmid} Let $Z_1, Z_2, \ldots, Z_n$ be independent random variables on set $\calZ$ and $\phi : \calZ^n \to \Real$ satisfy bounded difference inequality (\codref{cond:BDC}). Then for all $t>0$, we have 
    \begin{align*}
        \prob\left( \phi(Z_1, Z_2, \ldots, Z_n) - \Expo{\phi(Z_1, Z_2, \ldots, Z_n)} \ge t \right) \le \exp{\left( -\frac{2t^2}{\sum_{i=1}^n c_i^2} \right) } 
    \end{align*} 
    and 
    \begin{align*}
        \prob\left( \phi(Z_1, Z_2, \ldots, Z_n) - \Expo{\phi(Z_1, Z_2, \ldots, Z_n)} \le -t \right) \le \exp{\left( -\frac{2t^2}{\sum_{i=1}^n c_i^2} \right) } \,.
    \end{align*} 
\end{lemma}

\section{Proofs from \secref{sec:ERM_training}}\label{app:proof_erm}

\textbf{Additional notation {} {}} Let $m_1$ be the number of mislabeled points ($\wt S_M$) and $m_2$ be the number of correctly labeled points ($\wt S_C$). Note $m_1 + m_2 = m$.

\subsection{Proof of \thmref{thm:error_ERM}}

\begin{proof}[Proof of \lemref{lem:fit_mislabeled}] 
    The main idea of our proof is to regard 
    the clean portion of the data 
    ($S \cup \wt S_C$) as fixed.   
    Then, there exists an (unknown) classifier $f^*$ 
    that minimizes the expected risk
    calculated on the (fixed) clean data
    and (random draws of) the mislabeled data $\wt S_M$. 
    Formally, 
    \begin{align}
    f^* \defeq \argmin_{f \in \calF} \error_{\widecheck {\calD}} (f) \,, \label{eq:modified_ERM}
    \end{align}
    where $$\widecheck \calD = \frac{n}{m+n} \calS + \frac{m_2}{m+n} \wt \calS_C  + \frac{m_1}{m+n}\calDm \,.$$ 
    Note here that $\widecheck \calD$ is a combination 
    of the \emph{empirical distribution} 
    over correctly labeled data $S \cup \wt S_C$
    and the (population) distribution 
    over mislabeled data $\calDm$.
    Recall that 
    \begin{align}
    \wh f \defeq \argmin_{f \in \calF} \error_{\calS \cup \wt S} (f) \,. \label{eq:orig_ERM}
    \end{align}
    Since, $\widehat f$ minimizes 0-1 error 
    on $S \cup \wt S$, using ERM optimality on \eqref{eq:orig_ERM},  
    we have 
    \begin{align}
        \error_{\calS \cup \wt \calS}(\widehat f) \le \error_{
            \calS \cup \wt \calS}(f^*) \,.    \label{eq:step1}
    \end{align}
    Moreover, since $f^*$ is independent of $\wt S_M$, using Hoeffding's bound,
    % \footnote{For a fully rigorous argument,
    % refer to the complete proof in App.~\ref{app:proof_erm}.} 
    we have with probability at least $1-\delta$ that
    \begin{align}
      \error_{\wt \calS_M}(f^*) \le \error_{ \calDm}(f^*) +  \sqrt{\frac{\log(1/\delta)}{2 m_1}} \,. \label{eq:step2} 
    \end{align}
    %$ 
    %for some constant $c_1\le 1/2$. 
    Finally, since $f^*$ is the optimal classifier on $\widecheck \calD$, 
    we have 
    \begin{align}
        \error_{\widecheck \calD}(f^*) \le \error_{\widecheck \calD}(\widehat f) \,. \label{eq:step3}
    \end{align}
    Now to relate \eqref{eq:step1} and \eqref{eq:step3}, we multiply \eqref{eq:step2} by $\frac{m_1}{m+n}$ and add $\frac{n}{m+n} \error_{\calS} (f)  + \frac{m_2}{m+n} \error_{\wt \calS_C} (f)$ both the sides. Hence, 
    we can rewrite \eqref{eq:step2} as follows: 
    \begin{align}
        \error_{\calS \cup \wt\calS}(f^*) \le \error_{ \widecheck \calD}(f^*) +  \frac{m_1}{m+n}\sqrt{\frac{\log(1/\delta)}{2 m_1}} \,. \label{eq:step4} 
    \end{align}
    Now we combine equations \eqref{eq:step1}, \eqref{eq:step4}, and \eqref{eq:step3}, to get 
    \begin{align}
        \error_{\calS \cup \wt \calS}(\wh f) \le \error_{\widecheck \calD}(\wh f) +  \frac{m_1}{m+n}\sqrt{\frac{\log(1/\delta)}{2 m_1}} \,, 
    \end{align}
    which implies 
    \begin{align}
        \error_{ \wt \calS_M}(\wh f) \le \error_{\calDm}(\wh f) + \sqrt{\frac{\log(1/\delta)}{2 m_1}} \,. \label{eq:lemma1_final}
    \end{align}
    Since $\wt S$ is obtained by randomly labeling an unlabeled dataset, we assume $2m_1 \approx m$ \footnote{Formally, with probability at least $1-\delta$, we have  $(m - 2m_1)\le \sqrt{m\log(1/\delta)/2}$.}. Moreover, using $\error_{\calDm} = 1 - \error_{\calD}$ we obtain the desired result.   
    % Combining the above steps and using the fact 
    % that $\error_\calD = 1- \error_{\calDm} $, 
    % we obtain the desired result.
\end{proof}

\begin{proof}[Proof of \lemref{lem:mislabeled_error}]
    Recall $\error_{\wt S} (f) = \frac{m_1}{m} \error_{\wt S_M}(f) + \frac{m_2}{m} \error_{\wt S_C}(f)$. Hence, we have 
    \begin{align}
        2\error_{\wt S}(f) - \error_{\wt S_M}(f) - \error_{\wt S_C}(f) &= \left(\frac{2m_1}{m} \error_{\wt S_M}(f) - \error_{\wt S_M}(f)\right) + \left(\frac{2m_2}{m} \error_{\wt S_C}(f) - \error_{\wt S_C}(f)\right) \\ &= \left(\frac{2m_1}{m} - 1\right) \error_{\wt S_M}(f) + \left(\frac{2m_2}{m} - 1 \right)\error_{\wt S_C} (f) \,.
    \end{align} 
    Since the dataset is labeled uniformly at random, with probability at least $1-\delta$, we have  $\left(\frac{2m_1}{m} - 1\right) \le \sqrt{\frac{\log(1/\delta)}{2m}}$. Similarly, we have with probability at least $1-\delta$, $\left(\frac{2m_2}{m} - 1\right) \le \sqrt{\frac{\log(1/\delta)}{2m}}$. Using union bound, with probability at least $1-\delta$, we have
    % \begin{align}
    %     2\error_{\wt S} - \error_{\wt S_M}(f) - \error_{\wt S_C}(f) \le \sqrt{\frac{\log(2/\delta)}{2m}} \left(\error_{\wt S_M}(f) + \error_{\wt S_C}(f) \right) \le 2\sqrt{\frac{\log(2/\delta)}{2m}} \,. \label{eq:lemma2_final}
    % \end{align}
    \begin{align}
        2\error_{\wt S} - \error_{\wt S_M}(f) - \error_{\wt S_C}(f) \le \sqrt{\frac{\log(2/\delta)}{2m}} \left(\error_{\wt S_M}(f) + \error_{\wt S_C}(f) \right) \,. \label{eq:lemma2_prefinal}
    \end{align}
    With re-arranging $\error_{\wt S_M}(f) + \error_{\wt S_C}(f)$ and using the inequality $ 1- a\le \frac{1}{1+a} $, we have  
    \begin{align}
        2\error_{\wt S} - \error_{\wt S_M}(f) - \error_{\wt S_C}(f) \le 2\error_{\wt \calS} \sqrt{\frac{\log(2/\delta)}{2m}}  \,. \label{eq:lemma2_final}
    \end{align}

    % We obtain the desired result by using 
\end{proof}

\begin{proof}[Proof of \lemref{lem:clear_error}]
% Recall 0-1 error on each point  $(x,y) \in S \cup \wt S$ is given by $\I{ f(x)\ne y}$.
In the set of correctly labeled points $S \cup \wt S_C$, we have $S$ as a random subset of $S \cup \wt S_C$. Hence, using Hoeffding's inequality for sampling without replacement (\lemref{lem:hoeffding_sampling}), we have with probability at least $1-\delta$
\begin{align}
    \error_{\wt \calS_C} (\wh f)- \error_{\calS \cup \wt \calS_C}( \wh f) \le  \sqrt{\frac{\log(1/\delta)}{2m_2}} \,.
\end{align}
Re-writing $\error_{\calS \cup \wt \calS_C}( \wh f)$ as $\frac{m_2}{m_2 + n} \error_{\wt \calS_C }(\wh f) + \frac{n}{m_2 + n} \error_{\calS }(\wh f)$, we have with probability at least $1-\delta$
\begin{align}
   \left(\frac{n}{n+m_2}\right) \left(\error_{\wt \calS_C} (\wh f)- \error_{\calS}( \wh f) \right) \le  \sqrt{\frac{\log(1/\delta)}{2m_2}} \,.
\end{align}
As before, assuming $2m_2 \approx m$, we have with probability at least $1-\delta$ 
\begin{align}
    \error_{\wt \calS_C} (\wh f)- \error_{\calS}( \wh f) \le \left(1+\frac{m_2}{n}\right)  \sqrt{\frac{\log(1/\delta)}{m}} \le \left(1 + \frac{m}{2n}\right) \sqrt{\frac{\log(1/\delta)}{m}} \,. \label{eq:lemma3_final}
\end{align} 
\end{proof}

\begin{proof}[Proof of \thmref{thm:error_ERM}] 
    Having established these core intermediate results, we can now combine above three lemmas to prove the main result. 
    In particular, we bound the population error on clean data ($\error_\calD(\wh f)$) as follows:  
    \begin{enumerate}[(i)]
        \item First, use \eqref{eq:lemma1_final}, to obtain an upper bound on the population error on clean data, i.e., with probability at least $1-\delta/4$, we have
        \begin{align}
            \error_{ \calD} (\wh f) \le 1 - \error_{ \wt \calS_M}(\wh f) + \sqrt{\frac{\log(4/\delta)}{m}} \,. 
        \end{align}
        \item  Second, use \eqref{eq:lemma2_final}, to relate the error on the mislabeled fraction with error on clean portion of randomly labeled data and error on whole randomly labeled dataset, i.e., with probability at least $1-\delta/2$, we have 
        \begin{align}
            - \error_{\wt S_M}(f) \le \error_{\wt S_C}(f) - 2\error_{\wt S}  + 2\error_{\wt S} \sqrt{\frac{\log(4/\delta)}{2m}}  \,. 
        \end{align} 
        \item Finally, use \eqref{eq:lemma3_final} to relate the error on the clean portion of randomly labeled data and error on clean training data, i.e., with probability $1-\delta/4$, we have 
        \begin{align}
            \error_{\wt \calS_C} (\wh f)\le - \error_{\calS}( \wh f) + \left(1 + \frac{m}{2n} \right) \sqrt{\frac{\log(4/\delta)}{m}} \,. 
        \end{align} 
    \end{enumerate}

    Using union bound on the above three steps, we have with probability at least $1-\delta$: 
    \begin{align}
        \error_\calD (\wh f) \le \error_{\calS}(\wh f)   + 1 - 2\error_{\wt \calS}(\wh f)   + \left(\sqrt{2} \error_{\wt S} + 2 + \frac{m}{2n}\right)  \sqrt{\frac{\log(4/\delta)}{m}} \,.
    \end{align}
    % Note that $(1/\sqrt{2} + 2.5)$ is a loose constant. In experiments, we use the ratio $\frac{m}{n}$
    %  the exact error $\error_{\wt \calS}(\wh f)$ 
    % to evaluate R.H.S.    
\end{proof}

\subsection{Proof of \propref{prop:rademacher}}

\begin{proof}[Proof of \propref{prop:rademacher}]
    For a classifier $ f: \calX \to \{-1, 1\}$, we have $1 - 2\,\indict{ f(x) \ne y} = y \cdot f(x)$. Hence, by definition of $\error$, we have 
    \begin{align}
        1 -2\error_{\wt \calS}(f) = \frac{1}{m}\sum_{i=1}^m y_i \cdot f(x_i) \le \sup_{f \in \calF} \, \frac{1}{m} \sum_{i=1}^m y_i \cdot f(x_i)  \,. \label{eq:error_rademacher}
    \end{align}
    Note that for fixed inputs $(x_1, x_2, \ldots, x_m)$ in $\wt S$, $(y_1, y_2, \ldots y_m)$ are random labels. Define $\phi_1 (y_1, y_2, \ldots, y_m) \defeq \sup_{f \in \calF} \, \frac{1}{m} \sum_{i=1}^m y_i \cdot f(x_i)$. We have the following bounded difference condition on $\phi_1$. For all i, 
    \begin{align}
        \sup_{y_1, \ldots y_m, y_i^\prime \in \{-1, 1\}^{m+1} } \abs{ \phi_1 (y_1,\ldots, y_i, \ldots, y_m) - \phi_1 (y_1,\ldots, y_i^\prime, \ldots, y_m)  } \le 1/m \,. \label{cond1_rademacher}
    \end{align} 
    
    Similarly, we define $\phi_2 (x_1, x_2, \ldots, x_m) \defeq \Expt{ y_i \sim_U \{-1, 1\}  }{ \sup_{f \in \calF} \, \frac{1}{m}  \sum_{i=1}^m y_i \cdot f(x_i)}$. We have the following bounded difference condition on $\phi_2$. 
    For all i,
    \begin{align}
        \sup_{x_1, \ldots x_m, x_i^\prime \in \calX^{m+1} } \abs{ \phi_2 (x_1,\ldots, x_i, \ldots, x_m) - \phi_1 (x_1,\ldots, x_i^\prime, \ldots, x_m)  } \le 1/m \,. \label{cond2_rademacher}
    \end{align}
    Using McDiarmid’s inequality (\lemref{lem:McDiarmid}) twice 
    with Condition \eqref{cond1_rademacher} and \eqref{cond2_rademacher}, 
    with probability at least $1-\delta$, we have
    \begin{align}
        \sup_{f \in \calF} \, \frac{1}{m} \sum_{i=1}^m y_i \cdot f(x_i)  - \Expt{x,y}{\sup_{f \in \calF} \, \frac{1}{m} \sum_{i=1}^m y_i \cdot f(x_i) } \le \sqrt{\frac{2\log(2/\delta)}{m}} \,. \label{eq:final_rademacher}
    \end{align} 
    Combining \eqref{eq:error_rademacher} and \eqref{eq:final_rademacher}, we obtain the desired result. 
\end{proof}

\subsection{Proof of \thmref{thm:error_regularized_ERM}}

Proof of \thmref{thm:error_regularized_ERM} follows similar to the proof of \thmref{thm:error_ERM}. Note that the same results in \lemref{lem:fit_mislabeled}, \lemref{lem:mislabeled_error}, and \lemref{lem:clear_error} hold in the regularized ERM case. However, the arguments in the proof of \lemref{lem:fit_mislabeled} change slightly. Hence, we state the lemma for regularized ERM and prove it here for completeness. 

\begin{lemma} \label{lem:lemma1_reg}
    Assume the same setup as \thmref{thm:error_regularized_ERM}. 
    Then for any $\delta >0$, with probability at least  $1-\delta$ 
    over the random draws of mislabeled data $\wt S_M$, we have 
    \begin{align}
        \error_\calD(\widehat f)  \le 1 -\error_{\wt \calS_M}(\widehat f) + \sqrt{\frac{\log(1/\delta)}{m}}\,. 
    \end{align} 
\end{lemma}
\begin{proof}
    The main idea of the proof remains the same, i.e. regard 
    the clean portion of the data 
    ($S \cup \wt S_C$) as fixed.   
    Then, there exists a classifier $f^*$ 
    that is optimal over draws 
    of the mislabeled data $\wt S_M$.

    Formally, 
    \begin{align}
    f^* \defeq \argmin_{f \in \calF} \error_{\widecheck {\calD}} (f)  + \lambda R(f) \,, \label{eq:modified_ERM_reg}
    \end{align}
    where $$\widecheck \calD = \frac{n}{m+n} \calS + \frac{m_1}{m+n} \wt \calS_C  + \frac{m_2}{m+n}\calDm \,.$$ That is, $\widecheck \calD$ a combination of 
    the \emph{empirical distribution} 
    over correctly labeled data $S \cup \wt S_C$
    % in $S\cup \wt S$ 
    and the (population) distribution 
    over mislabeled data $\calDm$.
    Recall that 
    \begin{align}
    \wh f \defeq \argmin_{f \in \calF} \error_{\calS \cup \wt S} (f) + \lambda R(f) \,. \label{eq:orig_ERM_reg}
    \end{align}
    Since, $\widehat f$ minimizes 0-1 error 
    on $S \cup \wt S$, using ERM optimality on \eqref{eq:orig_ERM},  
    we have 
    \begin{align}
        \error_{\calS \cup \wt \calS}(\widehat f) + \lambda R(\wh f) \le \error_{
            \calS \cup \wt \calS}(f^*) + \lambda R(f^*) \,.    \label{eq:step1_reg}
    \end{align}
    Moreover, since $f^*$ is independent of $\wt S_M$, using Hoeffding's bound,
    % \footnote{For a fully rigorous argument,
    % refer to the complete proof in App.~\ref{app:proof_erm}.} 
    we have with probability at least $1-\delta$ that
    \begin{align}
      \error_{\wt \calS_M}(f^*) \le \error_{ \calDm}(f^*) +  \sqrt{\frac{\log(1/\delta)}{2 m_1}} \,. \label{eq:step2_reg} 
    \end{align}
    %$ 
    %for some constant $c_1\le 1/2$. 
    Finally, since $f^*$ is the optimal classifier on $\widecheck \calD$, 
    we have 
    \begin{align}
        \error_{\widecheck \calD}(f^*) + \lambda R(f^*) \le \error_{\widecheck \calD}(\widehat f) + \lambda R(\wh f) \,. \label{eq:step3_reg}
    \end{align}
     Now to relate \eqref{eq:step1_reg} and \eqref{eq:step3_reg}, we can re-write the \eqref{eq:step2_reg} as follows: 
    \begin{align}
        \error_{\calS \cup \wt\calS}(f^*) \le \error_{ \widecheck \calD}(f^*) +  \frac{m_1}{m+n}\sqrt{\frac{\log(1/\delta)}{2 m_1}} \,. \label{eq:step4_reg} 
    \end{align}
    After adding $\lambda R(f^*)$ on both sides in \eqref{eq:step4_reg}, we combine equations \eqref{eq:step1_reg}, \eqref{eq:step4_reg}, and \eqref{eq:step3_reg}, to get 
    \begin{align}
        \error_{\calS \cup \wt \calS}(\wh f) \le \error_{\widecheck \calD}(\wh f) +  \frac{m_1}{m+n}\sqrt{\frac{\log(1/\delta)}{2 m_1}} \,, 
    \end{align}
    which implies 
    \begin{align}
        \error_{ \wt \calS_M}(\wh f) \le \error_{\calDm}(\wh f) + \sqrt{\frac{\log(1/\delta)}{2 m_1}} \,. \label{eq:lemma_reg_final}
    \end{align}
    Similar as before, since $\wt S$ is obtained by randomly labeling an unlabeled dataset, we assume 
    $2m_1 \approx m$. Moreover, using $\error_{\calDm} = 1 - \error_{\calD}$ we obtain the desired result. 
\end{proof}
% \begin{proof}[Proof of ]
    
% \end{proof}

\subsection{Proof of \thmref{thm:multiclass_ERM}}

To prove our results in the multiclass case,
we first state and prove lemmas
parallel to those
% We first state and prove lemmas 
% parallel 
% to the three lemmas 
used in the proof of balanced binary case. 
We then combine these results 
% in the three lemmas 
to obtain the result in \thmref{thm:multiclass_ERM}. 

Before stating the result, 
we define mislabeled distribution $\calDm$ for any $\calD$.
While $\calDm$ and $\calD$ share 
the same marginal distribution over inputs $\calX$,
the conditional distribution over labels $y$ 
given an input $x\sim \calD_\calX$ is changed as follows:
For any $x$, the Probability Mass Function (PMF) over $y$ is defined as:  
$p_{\calDm} (\cdot \vert x) \defeq \frac{1 - p_{\calD}(\cdot \vert x)}{k - 1}$, where $ p_{\calD}(\cdot \vert x)$ is the PMF over $y$ for the distribution $\calD$. 

\begin{lemma} \label{lem:fit_mislabeled_multi}
    Assume the same setup as \thmref{thm:multiclass_ERM}. 
    Then for any $\delta >0$, with probability at least  $1-\delta$ 
    over the random draws of mislabeled data $\wt S_M$, we have 
    \begin{align}
        \error_\calD(\widehat f)  \le (k-1)\left(1 -\error_{\wt \calS_M}(\widehat f)\right) + (k-1)\sqrt{\frac{\log(1/\delta)}{m}}\,. \label{eq:lemma1_multi}
    \end{align}   
\end{lemma} 

\begin{proof}
   
    The main idea of the proof remains the same.
    We begin by regarding the clean portion of the data 
    ($S \cup \wt S_C$) as fixed. 
    Then, there exists a classifier $f^*$ 
    that is optimal over draws 
    of the mislabeled data $\wt S_M$. 
    
    However, in the multiclass case,
    we cannot as easily relate the population error on mislabeled data 
    to the population accuracy on clean data.   
    While for binary classification, 
    % we could upper bound $\error_{\wt \calS_M}$ 
    % with $1-\error_\calD$ 
    we could lower bound the population accuracy $1-\error_\calD$
    with the empirical error on mislabeled data $\error_{\wt \calS_M}$ 
    (in the proof of \lemref{lem:fit_mislabeled}), 
    for multiclass classification, 
    error on the mislabeled data 
    and accuracy on the clean data 
    in the population 
    are not so directly related.  
    To establish \eqref{eq:lemma1_multi},
    we break the error on the 
    (unknown) mislabeled data 
    into two parts: one term corresponds 
    to predicting the true label on mislabeled data, 
    and the other corresponds to predicting 
    neither the true label 
    nor the assigned (mis-)label.  
    Finally, we relate these errors to their
    population counterparts to establish \eqref{eq:lemma1_multi}. 
    
    Formally, 
    \begin{align}
    f^* \defeq \argmin_{f \in \calF} \error_{\widecheck {\calD}} (f)  + \lambda R(f) \,, \label{eq:modified_ERM_reg2}
    \end{align}
    where $$\widecheck \calD = \frac{n}{m+n} \calS + \frac{m_1}{m+n} \wt \calS_C  + \frac{m_2}{m+n}\calDm \,.$$ 
    That is, $\widecheck \calD$ is a combination 
    of the \emph{empirical distribution} 
    over correctly labeled data $S \cup \wt S_C$
    % in $S\cup \wt S$ 
    and the (population) distribution 
    over mislabeled data $\calDm$.
    Recall that 
    \begin{align}
    \wh f \defeq \argmin_{f \in \calF} \error_{\calS \cup \wt S} (f) + \lambda R(f) \,. \label{eq:orig_ERM_reg2}
    \end{align}
    Following the exact steps from the proof of \lemref{lem:lemma1_reg}, 
    with probability at least $1-\delta$, we have  
    \begin{align}
        \error_{ \wt \calS_M}(\wh f) \le \error_{\calDm}(\wh f) + \sqrt{\frac{\log(1/\delta)}{2 m_1}} \,. \label{eq:lemma1_final_multi_prev}
    \end{align}
    Similar to before, since $\wt S$ is obtained 
    by randomly labeling an unlabeled dataset, 
    we assume 
    $\frac{k}{k-1} m_1 \approx m$. 
    
    Now we will relate $\error_{\calDm} (\wh f)$ with $\error_{\calD}(\wh f)$. 
    Let $y^T$ denote the (unknown) true label 
    for a mislabeled point $(x, y)$ 
    (i.e., label before replacing it with a mislabel). 
    \begin{align*}    
         \Expt{(x, y) \in \sim \calDm}{\indict{ \wh f(x) \ne y }}  &= \underbrace{\Expt{(x, y) \in \sim \calDm}{\indict{ \wh f(x) \ne y \land \wh f(x) \ne y^T}}}_{\RN{1}} \\ &\qquad \qquad + \underbrace{\Expt{(x, y) \in \sim \calDm}{\indict{ \wh f(x) \ne y \land \wh f(x) = y^T}}}_{\RN{2}} \,. \numberthis \label{eq:excess_term}
    \end{align*}
    Clearly, term 2 is one minus the accuracy 
    on the clean unseen data, i.e.,
    \begin{align}
        \RN{2} = 1 - \Expt{{x,y} \sim \calD}{ \indict{ \wh f(x) \ne y}} = 1- \error_{\calD}(\wh f) \,. \label{eq:term1}    
    \end{align}
    Next, we relate term 1 with the error on the unseen clean data. 
    We show that term 1 is equal to the error on the unseen clean data 
    scaled by $\frac{k-2}{k-1}$,
    where $k$ is the number of labels.
    Using the definition of mislabeled distribution $\calDm$,  
    we have 
    \begin{align}
        \RN{1} = \frac{1}{k-1} \left( \Expt{(x, y) \in \sim \calD}{ \sum_{i \in \calY \land i\ne y}  \indict{ \wh f(x) \ne i \land \wh f(x) \ne y}} \right) = \frac{k-2}{k-1} \error_{\calD}(\wh f) \,.\label{eq:term2}
    \end{align}    

    Combining the result in \eqref{eq:term1}, \eqref{eq:term2} and \eqref{eq:excess_term}, we have 
    \begin{align}
        \error_{\calDm}(\wh f) = 1- \frac{1}{k-1} \error_{\calD}(\wh f) \,.\label{eq:combine_terms}
    \end{align}
    Finally, combining the result in \eqref{eq:combine_terms} 
    with equation \eqref{eq:lemma1_final_multi_prev}, 
    we have with probability $1-\delta$, 
    \begin{align}
      \error_{\calD}(\wh f) \le  (k-1) \left( 1- \error_{ \wt \calS_M}(\wh f) \right)  + (k-1) \sqrt{\frac{k \log(1/\delta)}{ 2(k-1)m}} \,. \label{eq:lemma1_final_multi}
    \end{align}
\end{proof}

\begin{lemma} \label{lem:mislabeled_error_multi}
    Assume the same setup as \thmref{thm:multiclass_ERM}. 
    Then for any $\delta >0$, 
    with probability at least $1-\delta$ 
    over the random draws of $\wt S$, we have  
    % \begin{align}
        $$\abs{k\error_{\wt \calS}(\widehat f) - \error_{\wt \calS_C}(\widehat f) -  (k-1)\error_{\wt \calS_M}(\widehat f) } \le  2k\sqrt{\frac{\log(4/\delta)}{2m}}\,. $$ % \label{eq:lemma2}
    % \end{align}   
    %  for some constant $c_3 \le 1.0\,$.
\end{lemma}

\begin{proof}
    Recall $\error_{\wt S} (f) = \frac{m_1}{m} \error_{\wt S_M}(f) + \frac{m_2}{m} \error_{\wt S_C}(f)$. Hence, we have 
    \begin{align*}
        k\error_{\wt S}(f) - (k-1)\error_{\wt S_M}(f) - \error_{\wt S_C}(f) &= (k-1)\left(\frac{k m_1}{(k-1) m} \error_{\wt S_M}(f) - \error_{\wt S_M}(f)\right) \\ & \qquad \qquad + \left(\frac{km_2}{m} \error_{\wt S_C}(f) - \error_{\wt S_C}(f)\right) \\ &= k \left[ \left(\frac{m_1}{m} - \frac{k-1}{k}\right) \error_{\wt S_M}(f) + \left(\frac{m_2}{m} - \frac{1}{k} \right) \error_{\wt S_C} (f) \right] \,.
    \end{align*} 
    Since the dataset is randomly labeled, 
    we have with probability at least $1-\delta$, 
    $\left(\frac{m_1}{m} - \frac{k-1}{k}\right) \le \sqrt{\frac{\log(1/\delta)}{2m}}$. 
    Similarly, we have with probability at least $1-\delta$, 
    $\left(\frac{m_2}{m} - \frac{1}{k}\right) \le \sqrt{\frac{\log(1/\delta)}{2m}}$. 
    Using union bound, we have with probability at least $1-\delta$
    % \begin{align}
    %     2\error_{\wt S} - \error_{\wt S_M}(f) - \error_{\wt S_C}(f) \le \sqrt{\frac{\log(2/\delta)}{2m}} \left(\error_{\wt S_M}(f) + \error_{\wt S_C}(f) \right) \le 2\sqrt{\frac{\log(2/\delta)}{2m}} \,. \label{eq:lemma2_final}
    % \end{align}
    \begin{align}
        k\error_{\wt S}(f) - (k-1)\error_{\wt S_M}(f) - \error_{\wt S_C}(f)  \le k \sqrt{\frac{\log(2/\delta)}{2m}} \left(\error_{\wt S_M}(f) + \error_{\wt S_C}(f) \right) \,. \label{eq:lemma2_final_multi}
    \end{align}

    % We obtain the desired result by using 
\end{proof}

\begin{lemma} \label{lem:clear_error_multi}
    Assume the same setup as \thmref{thm:multiclass_ERM}. 
    Then for any $\delta >0$, with probability at least $1-\delta$ 
    over the random draws of $\wt S_C$ and $S$, we have 
    % \begin{align}
        $$\abs{\error_{\wt \calS_C}(\widehat f) - \error_{\calS}(\widehat f) } \le 1.5 \sqrt{\frac{k\log(2/\delta)}{2m}}\,.$$ %\label{eq:lemma3}
    % \end{align}   
    % for some constant $c_2 \le 1.2\,$.
\end{lemma} 
\begin{proof}
    % Recall 0-1 error on each point  $(x,y) \in S \cup \wt S$ is given by $\I{ f(x)\ne y}$.
    In the set of correctly labeled points $S \cup \wt S_C$,
    we have $S$ as a random subset of $S \cup \wt S_C$. 
    Hence, using Hoeffding's inequality 
    for sampling without replacement 
    (\lemref{lem:hoeffding_sampling}), 
    we have with probability at least $1-\delta$
    \begin{align}
        \error_{\wt \calS_c} (\wh f)- \error_{\calS \cup \wt \calS_C}( \wh f) \le  \sqrt{\frac{\log(1/\delta)}{2m_2}} \,.
    \end{align}
    Re-writing $\error_{\calS \cup \wt \calS_C}( \wh f)$ 
    as $\frac{m_2}{m_2 + n} \error_{\wt \calS_C }(\wh f) + \frac{n}{m_2 + n} \error_{\calS }(\wh f)$, 
    we have with probability at least $1-\delta$
    \begin{align}
       \left(\frac{n}{n+m_2}\right) \left(\error_{\wt \calS_c} (\wh f)- \error_{\calS}( \wh f) \right) \le  \sqrt{\frac{\log(1/\delta)}{2m_2}} \,.
    \end{align}
    As before, assuming $km_2 \approx m$, 
    we have with probability at least $1-\delta$ 
    \begin{align}
        \error_{\wt \calS_c} (\wh f)- \error_{\calS}( \wh f) \le \left(1+\frac{m_2}{n}\right)  \sqrt{\frac{k\log(1/\delta)}{2m}} \le \left( 1 + \frac{1}{k}\right) \sqrt{\frac{k\log(1/\delta)}{2m}} \,. \label{eq:lemma3_final_multi}
    \end{align} 
\end{proof}

\begin{proof}[Proof of \thmref{thm:multiclass_ERM}] 
    Having established these core intermediate results, 
    we can now combine above three lemmas. 
    In particular, we bound the population error 
    on clean data ($\error_\calD(\wh f)$) as follows:  
    \begin{enumerate}[(i)]
        \item First, use \eqref{eq:lemma1_final_multi}, 
        to obtain an upper bound on the population error on clean data, 
        i.e., with probability at least $1-\delta/4$, we have
        \begin{align}
            \error_{ \calD} (\wh f) \le (k-1)\left(1 - \error_{ \wt \calS_M}(\wh f) \right) + (k-1) \sqrt{\frac{k\log(4/\delta)}{2(k-1)m}} \,. 
        \end{align}
        \item  Second, use \eqref{eq:lemma2_final_multi}
        to relate the error on the mislabeled fraction 
        with error on clean portion of randomly labeled data 
        and error on whole randomly labeled dataset, 
        i.e., with probability at least $1-\delta/2$, we have 
        \begin{align}
            - (k-1)\error_{\wt S_M}(f) \le \error_{\wt S_C}(f) - k\error_{\wt S}  + k\sqrt{\frac{\log(4/\delta)}{2m}}  \,. 
        \end{align} 
        \item Finally, use \eqref{eq:lemma3_final_multi} 
        to relate the error on the clean portion of randomly labeled data 
        and error on clean training data, 
        i.e., with probability $1-\delta/4$, we have 
        \begin{align}
            \error_{\wt \calS_C} (\wh f)\le - \error_{\calS}( \wh f) + \left(1 + \frac{m}{kn} \right) \sqrt{\frac{k\log(4/\delta)}{2m}} \,. 
        \end{align} 
    \end{enumerate}

    Using union bound on the above three steps, 
    we have with probability at least $1-\delta$: 
    \begin{align}
        \error_\calD (\wh f) \le \error_{\calS}(\wh f) + (k-1) - k\error_{\wt \calS}(\wh f)   + (\sqrt{k(k-1)} + k + \sqrt{k} + \frac{m}{n\sqrt{k}})  \sqrt{\frac{\log(4/\delta)}{2m}} \,.\label{eq:multiclass_ERM_final}
    \end{align}
    Simplifying the term in RHS of \eqref{eq:multiclass_ERM_final}, 
    we get the desired result. 
    % Note that since $\frac{m}{n\sqrt{k}}$ 
    % is much smaller than the sum of the other terms
    % the other terms in summation, 
    % we ignore $\frac{m}{n\sqrt{k}}$  
    % Z: ??? --- great
    % that 
    % them
    in the final bound. 
    % we ignore that in the final bound. 
    % Note that $(1/\sqrt{2} + 2.5)$ is a loose constant. In experiments, we use the ratio $\frac{m}{n}$
    %  the exact error $\error_{\wt \calS}(\wh f)$ 
    % to evaluate R.H.S.    
\end{proof}

\newpage
\section{Proofs from \secref{sec:linear_models}}\label{app:proof_gd}
We suppose that the parameters of the linear function 
are obtained via gradient descent on 
the following $L_2$ regularized problem: 
\begin{align}
    % n in denominator is avoided deliberately
    \calL_S(w; \lambda) \defeq \sum_{i=1}^n{(w^Tx_i - y_i)^2} + \lambda \norm{w}{2}^2 \,, \label{eq:l2_MSE_app}   
\end{align}
where $\lambda\ge0$ is a regularization parameter. 
We assume access to a clean dataset 
$S = \{(x_i, y_i)\}_{i=1}^n \sim \calD^n$ 
and randomly labeled dataset 
$\wt S = \{(x_i, y_i)\}_{i=n+1}^{n+m} \sim \wt \calD^m$. 
Let $\bX = [x_1, x_2, \cdots, x_{m+n}]$ 
and $\by = [y_1, y_2, \cdots, y_{m+n}]$. 
Fix a positive learning rate $\eta$ such that 
$\eta \le 1/\left(\norm{\bX^T\bX}{\text{op}} + \lambda^2\right)$ 
and an initialization $w_0 = 0$. 
% \todos{Assumption made for simplicty}. 
Consider the following gradient descent iterates 
to minimize objective \eqref{eq:l2_MSE_app} on $S \cup \wt S$:
\begin{align}
w_t = w_{t-1} - \eta \grad_w \calL_{S \cup \wt S} (w_{t-1}; \lambda) \quad \forall t=1,2,\ldots \label{eq:GD_iterates_app}
\end{align} 
Then we have $\{ w_t\}$ converge to the limiting solution 
$\wh w = \left( \bX^T\bX+\lambda \boldsymbol{I}\right)^{-1}\bX^T\by$. Define $\widehat f (x) \defeq f(x ; \wh w) $.

\subsection{Proof of \thmref{thm:linear}}
We use a standard result from linear algebra, 
namely the Shermann-Morrison formula 
\citep{sherman1950adjustment} for matrix inversion:  

\begin{lemma}[\citet{sherman1950adjustment}] \label{lem:sherman}
    Suppose $\bA \in \Real^{n \times n}$ 
    is an invertible square matrix 
    and $u,v \in \Real^n$ are column vectors. 
    Then $\bA + uv^T$ is invertible iff $1 + v^T \bA u \ne 0$ 
    and in particular
    \begin{align}
        (\bA + u v^T)^{-1} = \bA^{-1}  - \frac{\bA^{-1} uv^T \bA^{-1} }{ 1 + v^T \bA^{-1} u} \,.
    \end{align}   
\end{lemma}
\newcommand\byy[1]{\by_{\left(#1\right)}}
\newcommand\bXX[1]{\bX_{\left(#1\right)}}
\newcommand\ff[1]{\wh f_{\left(#1\right)}}

For a given training set $S \cup \wt S_C$, 
define leave-one-out error 
on mislabeled points in the training data 
as $$\error_{\text{LOO}(\wt S_M) } = \frac{\sum_{(x_i, y_i) \in \wt S_M} \error( f_{(i)}( x_i), y_i)}{ \abs{\wt S_M }} \,, $$
where $f_{(i)} \defeq f(\calA, (S \cup \wt S)_{(i)})$. 
To relate empirical leave-one-out error and population error 
with hypothesis stability condition, 
we use the following lemma:   

\begin{lemma}[\citet{bousquet2002stability}] \label{lem:stability_error}
    For the leave-one-out error, we have
    \begin{align}
        \Expo{ \left( \error_{\calDm}(\wh f) -\error_{\text{LOO}(\wt S_M) } \right)^2 } \le \frac{1}{2m_1}+  \frac{3\beta}{n + m}\,.
    \end{align}   
    % where $ f \defeq f(\calA, S \cup \wt S) $.
\end{lemma}

Proof of the above lemma is similar 
to the proof of Lemma 9 in \citet{bousquet2002stability} 
and can be found in \appref{app:proof_lem_error}. 
% 
% Before presenting the result, we introduce some notation. 
Before presenting the proof of \thmref{thm:linear}, 
we introduce some more notation. 
Let $\bX_{(i)}$ denote the matrix of covariates 
with the $i^{\text{th}}$ point removed. 
Similarly, let $\by_{(i)}$ be the array of responses 
with the $i^{\text{th}}$ point removed. 
Define the corresponding regularized GD solution 
as $\wh w_{(i)} = \left( \bXX{i}^T\bXX{i}+\lambda \boldsymbol{I}\right)^{-1}\bXX{i}^T\byy{i}$. 
Define $\ff{i}(x) \defeq f(x ; \wh w_{(i)}) $.

\begin{proof}[Proof of \thmref{thm:linear}]
    Because squared loss minimization does not imply 0-1 error minimization, 
    we cannot use arguments from \lemref{lem:fit_mislabeled}. 
    This is the main technical difficulty. 
    To compare the 0-1 error at a train point with an unseen point, 
    we use the closed-form expression for $\widehat{w}$ 
    and Shermann-Morrison formula 
    to upper bound training error 
    with leave-one-out cross validation error. 
    
    The proof is divided into three parts: 
    In part one, we show that 0-1 error 
    on mislabeled points in the training set 
    is lower than the error obtained 
    by leave-one-out error at those points. 
    In part two, we relate this leave-one-out error 
    with the population error on mislabeled distribution
    using \codref{cond:hypothesis_stability}.
    While the empirical leave-one-out error is an unbiased estimator 
    of the average population error of leave-one-out classifiers, 
    we need hypothesis stability 
    to control the variance 
    of empirical leave-one-out error. 
    Finally, in part three, we show 
    that the error on the mislabeled training points 
    can be estimated with just the randomly labeled 
    and clean training data (as in proof of \thmref{thm:error_ERM}).  

    \textbf{Part 1 {} {}} First we relate training error with leave-one-out error.        
    For any training point $(x_i, y_i)$ in $\wt S \cup S$, we have 
    \begin{align}
        \error(\wh f(x_i), y_i ) &= \indict{ y_i \cdot x_i^T \wh w < 0 } = \indict{ y_i \cdot x_i^T \left( \bX^T\bX+\lambda \boldsymbol{I}\right)^{-1}\bX^T\by < 0 } \\
        &= \indict{ y_i \cdot x_i^T \underbrace{\left( \bXX{i}^T\bXX{i} + x_i ^T x_i +\lambda \boldsymbol{I}\right)^{-1}}_{\RN{1}} (\bXX{i}^T\byy{i} + y_i \cdot x_i) < 0 } \,.
    \end{align}
    Letting $\bA = \left(\bXX{i}^T\bXX{i} +\lambda \boldsymbol{I}\right)$ 
    and using \lemref{lem:sherman} on term 1, we have 
    \begin{align}
        \error(\wh f(x_i), y_i ) &= \indict{ y_i \cdot x_i^T \left[\bA^{-1} -  \frac{\bA^{-1} x_i x_i^T \bA^{-1}}{ 1 + x_i ^T \bA^{-1} x_i } \right] (\bXX{i}^T\byy{i} + y_i \cdot x_i) < 0 } \\
        &= \indict{ y_i \cdot\left[ \frac{ x_i^T \bA^{-1} ( 1 + x_i ^T \bA^{-1} x_i ) -  x_i^T \bA^{-1} x_i x_i^T \bA^{-1}}{ 1 + x_i ^T \bA ^{-1}x_i } \right] (\bXX{i}^T\byy{i} + y_i \cdot x_i) < 0 } \\
        &= \indict{ y_i \cdot\left[ \frac{ x_i^T \bA^{-1}}{ 1 + x_i ^T \bA ^{-1}x_i } \right] (\bXX{i}^T\byy{i} + y_i \cdot x_i) < 0 } \,.
    \end{align}

    Since $1 + x_i^T \bA^{-1} x_i > 0$, we have 
    \begin{align}
        \error(\wh f(x_i), y_i ) &= \indict{ y_i \cdot x_i^T \bA^{-1} (\bXX{i}^T\byy{i} + y_i \cdot x_i) < 0 } \\
        &= \indict{ x_i^T \bA^{-1} x_i +  y_i \cdot x_i^T \bA^{-1} (\bXX{i}^T\byy{i}) < 0 } \\
        &\le \indict{ y_i \cdot x_i^T \bA^{-1} (\bXX{i}^T\byy{i}) < 0 } = \error(\ff{i}(x_i), y_i ) \,.\label{eq:LOO_error}
    \end{align}

    Using \eqref{eq:LOO_error}, we have 
    \begin{align}
        \error_{\wt \calS_M } (\wh f) \le \error_{\text{LOO} (\wt S_M)} \defeq \frac{\sum_{(x_i, y_i) \in \wt S_M} \error(\ff{i}(x_i), y_i ) }{\abs{\wt \calS_M}}\label{eq:LOO_error_final} \,.
    \end{align}
    \textbf{Part 2 {}{}} We now relate RHS in \eqref{eq:LOO_error_final} 
    with the population error on mislabeled distribution. 
    To do this, we leverage \codref{cond:hypothesis_stability} 
    and \lemref{lem:stability_error}. 
    In particular, we have 

    \begin{align}
        \Expt{\calS \cup \wt \calS_M }{ \left(\error_{\calDm}(\wh f) - \error_{\text{LOO} (\wt S_M)}\right)^2 } \le \frac{1}{2m_1} + \frac{3\beta}{m+n} \,.
    \end{align}

    Using Chebyshev's inequality, with probability at least $1-\delta$, we have 
    \begin{align}
        \error_{\text{LOO} (\wt S_M)} \le  \error_{\calDm}(\wh f)   + \sqrt{\frac{1}{\delta}\left(\frac{1}{2m_1} +\frac{3\beta}{m+n} \right)} \,. \label{eq:final_mislabeled_linear}
    \end{align}

    \textbf{Part 3 {}{}} Combining \eqref{eq:final_mislabeled_linear} and \eqref{eq:LOO_error_final}, we have 

    \begin{align}
        \error_{\wt \calS_M } (\wh f) \le \error_{\calDm}(\wh f)   + \sqrt{\frac{1}{\delta}\left(\frac{1}{2m_1} +\frac{3\beta}{m+n} \right)} \,. \label{eq:linear_parallel_lem1}
    \end{align}

    Compare \eqref{eq:linear_parallel_lem1} with \eqref{eq:lemma1_final} 
    in the proof of \lemref{lem:fit_mislabeled}. 
    We obtain a similar relationship 
    between $\error_{\wt \calS_M }$ and $\error_{\calDm}$ 
    but with a polynomial concentration 
    instead of exponential concentration. 
    In addition, since we just use concentration arguments 
    to relate mislabeled error to the errors
    on the clean and unlabeled portions 
    of the randomly labeled data, 
    we can directly use the results 
    in \lemref{lem:mislabeled_error} and \lemref{lem:clear_error}. 
    Therefore, combining results in \lemref{lem:mislabeled_error}, \lemref{lem:clear_error}, and \eqref{eq:linear_parallel_lem1} with union bound, 
    we have with probability at least $1-\delta$
    \begin{align}
        \error_\calD(\widehat f) \le \error_\calS(\widehat f) + 1 - 2 \error_{\wt\calS}(\widehat f) + \left(\sqrt{2}\error_{\wt\calS}(\widehat f) + 1 + \frac{m}{2n} \right) \sqrt{\frac{\log(4/\delta)}{m}} + \sqrt{\frac{4}{\delta}\left(\frac{1}{m} +\frac{3\beta}{m+n} \right)}  \,.
    \end{align}

\end{proof}

\subsection{Extension to multiclass classification} \label{app:multiclass_linear}
For multiclass problems with squared loss minimization, as standard practice, we consider one-hot encoding for the underlying label, i.e., a class label $c \in [k]$ is treated as $(0, \cdot, 0,1,0, \cdot, 0) \in \Real^k$ (with $c$-th coordinate being 1).  As before, we suppose that the parameters of the linear function 
are obtained via gradient descent on the following $L_2$ regularized problem: 
\begin{align}
    % n in denominator is avoided deliberately
    \calL_S(w; \lambda) \defeq \sum_{i=1}^n\norm{w^Tx_i - y_i}{2}^2 + \lambda \sum_{j=1}^k \norm{w_j}{2}^2 \,, \label{eq:l2_multiclass_MSE_app}   
\end{align}
where $\lambda\ge0$ is a regularization parameter. 
We assume access to a clean dataset 
$S = \{(x_i, y_i)\}_{i=1}^n \sim \calD^n$ 
and randomly labeled dataset 
$\wt S = \{(x_i, y_i)\}_{i=n+1}^{n+m} \sim \wt \calD^m$. 
Let $\bX = [x_1, x_2, \cdots, x_{m+n}]$ 
and $\by = [e_{y_1}, e_{y_2}, \cdots, e_{y_{m+n}}]$. 
Fix a positive learning rate $\eta$ such that 
$\eta \le 1/\left(\norm{\bX^T\bX}{\text{op}} + \lambda^2\right)$ 
and an initialization $w_0 = 0$. 
% \todos{Assumption made for simplicty}. 
Consider the following gradient descent iterates 
to minimize objective \eqref{eq:l2_MSE_app} on $S \cup \wt S$:
\begin{align}
{w_j}^t = {w_j}^{t-1} - \eta \grad_{w_j} \calL_{S \cup \wt S} (w^{t-1}; \lambda) \quad \forall t=1,2,\ldots \text{ and } j=1,2,\ldots,k  \,. \label{eq:GD_multi_iterates_app}
\end{align} 
Then we have $\{ {w_j}^t\}$ for all $j =1,2,\cdots, k$ converge to the limiting solution 
$\wh w_j = \left( \bX^T\bX+\lambda \boldsymbol{I}\right)^{-1}\bX^T\by_j$. Define $\widehat f (x) \defeq f(x ; \wh w) $.  

\begin{theorem}\label{thm:multi_linear}
    Assume that this gradient descent algorithm satisfies \codref{cond:hypothesis_stability}
    with $\beta=\calO(1)$.  
    Then for a multiclass classification problem wth $k$ classes, for any $\delta >0$, with probability at least $1-\delta$, we have:
    \begin{align*}
        \error_\calD(\widehat f) \le \error_\calS(\widehat f) &+ (k-1)\left(1 - \frac{k}{k-1} \error_{\wt\calS}(\widehat f) \right) \\ &+ \left(k + \sqrt{k} + \frac{m}{n\sqrt{k}} \right) \sqrt{\frac{\log(4/\delta)}{2m}} + \sqrt{k(k-1)} \sqrt{\frac{4}{\delta}\left(\frac{1}{m} +\frac{3\beta}{m+n} \right)}  \,. \numberthis \label{eq:gd_multi_error}
    \end{align*} 
    % for some constant $c\le 3.2$.
\end{theorem}
\begin{proof}
    The proof of this theorem is divided into two parts. In the first part, we relate the error on the mislabeled samples with the population error on the mislabeled data. Similar to the proof of \thmref{thm:linear}, we use Shermann-Morrison formula to upper bound training error with leave-one-out error on each $\wh w^j$. Second part of the proof follows entirely from the proof of \thmref{thm:multiclass_ERM}. In essence, the first part derives an equivalent of \eqref{eq:lemma1_final_multi_prev} for GD training with squared loss and then the second part follows from the proof  of \thmref{thm:multiclass_ERM}. 
    
    \textbf{Part-1:} Consider a training point $(x_i,y_i)$ in $\wt S \cup S $. For simplicity, we use $c_i$ to denote the class of $i$-th point and use $y_i$ as the corresponding one-hot embedding. Recall error in multiclass point is given by $\error(\wh f(x_i), y_i ) = \indict{ c_i \not \in \argmax x_i^T \wh w }$. Thus, there exists a $j \ne c_i \in [k]$, such that we have
     \begin{align}
        \error(\wh f(x_i), y_i ) &= \indict{ c_i \not \in \argmax x_i^T \wh w } = \indict{ x_i^T \wh w_{c_i} < x_i^T \wh w_{j}  } \\ &= \indict{ x_i^T \left( \bX^T\bX+\lambda \boldsymbol{I}\right)^{-1}\bX^T\by_{c_i} < x_i^T \left( \bX^T\bX+\lambda \boldsymbol{I}\right)^{-1}\bX^T\by_{j} } \\
        &= \indict{ x_i^T \underbrace{\left( \bXX{i}^T\bXX{i} + x_i ^T x_i +\lambda \boldsymbol{I}\right)^{-1}}_{\RN{1}} \left(\bXX{i}^T{\by_{c_i}}_{(i)} + x_i - \bXX{i}^T{\by_{j}}_{(i)}\right) < 0 } \,.
    \end{align}
    Letting $\bA = \left(\bXX{i}^T\bXX{i} +\lambda \boldsymbol{I}\right)$ 
    and using \lemref{lem:sherman} on term 1, we have 
    \begin{align}
        \error(\wh f(x_i), y_i ) &= \indict{ x_i^T \left[\bA^{-1} -  \frac{\bA^{-1} x_i x_i^T \bA^{-1}}{ 1 + x_i ^T \bA^{-1} x_i } \right]  \left(\bXX{i}^T{\by_{c_i}}_{(i)} + x_i - \bXX{i}^T{\by_{j}}_{(i)}\right) < 0 } \\
        &= \indict{ \left[ \frac{ x_i^T \bA^{-1} ( 1 + x_i ^T \bA^{-1} x_i ) -  x_i^T \bA^{-1} x_i x_i^T \bA^{-1}}{ 1 + x_i ^T \bA ^{-1}x_i } \right]  \left(\bXX{i}^T{\by_{c_i}}_{(i)} + x_i - \bXX{i}^T{\by_{j}}_{(i)}\right) < 0 } \\
        &= \indict{ \left[ \frac{ x_i^T \bA^{-1}}{ 1 + x_i ^T \bA ^{-1}x_i } \right]  \left(\bXX{i}^T{\by_{c_i}}_{(i)} + x_i - \bXX{i}^T{\by_{j}}_{(i)}\right) < 0} \,.
    \end{align}
    Since $1 + x_i^T \bA^{-1} x_i > 0$, we have 
    \begin{align}
        \error(\wh f(x_i), y_i ) &= \indict{ x_i^T \bA^{-1}  \left(\bXX{i}^T{\by_{c_i}}_{(i)} + x_i - \bXX{i}^T{\by_{j}}_{(i)}\right) < 0 } \\
        &= \indict{ x_i^T \bA^{-1} x_i +  x_i^T \bA^{-1}  \bXX{i}^T{\by_{c_i}}_{(i)}  - x_i^T\bA^{-1}  \bXX{i}^T{\by_{j}}_{(i)} < 0 } \\
        &\le \indict{  x_i^T \bA^{-1}  \bXX{i}^T{\by_{c_i}}_{(i)}  - x_i^T\bA^{-1}  \bXX{i}^T{\by_{j}}_{(i)} < 0  } = \error(\ff{i}(x_i), y_i ) \,.\label{eq:LOO_error_multi}
    \end{align}
    Using \eqref{eq:LOO_error_multi}, we have 
    \begin{align}
        \error_{\wt \calS_M } (\wh f) \le \error_{\text{LOO} (\wt S_M)} \defeq \frac{\sum_{(x_i, y_i) \in \wt S_M} \error(\ff{i}(x_i), y_i ) }{\abs{\wt \calS_M}}\label{eq:LOO_error_multi_final} \,.
    \end{align}
    
    We now relate RHS in \eqref{eq:LOO_error_final} 
    with the population error on mislabeled distribution. 
    Similar as before, to do this, we leverage \codref{cond:hypothesis_stability} 
    and \lemref{lem:stability_error}. Using  \eqref{eq:final_mislabeled_linear} and \eqref{eq:LOO_error_multi_final}, we have 
    \begin{align}
        \error_{\wt \calS_M } (\wh f) \le \error_{\calDm}(\wh f)   + \sqrt{\frac{1}{\delta}\left(\frac{1}{2m_1} +\frac{3\beta}{m+n} \right)} \,. \label{eq:linear_multi_parallel_lem1}
    \end{align}
    
    We have now derived a parallel to \eqref{eq:lemma1_final_multi_prev}. Using the same arguments in the proof of \lemref{lem:fit_mislabeled_multi}, we have 
    \begin{align}
      \error_{\calD}(\wh f) \le  (k-1) \left( 1- \error_{ \wt \calS_M}(\wh f) \right)  + (k-1)\sqrt{\frac{k}{\delta(k-1)}\left(\frac{1}{2m_1} +\frac{3\beta}{m+n} \right)}  \,. \label{eq:lemma1_linear_final_multi}
    \end{align}
    
    \textbf{Part-2:} We now combine the results in \lemref{lem:mislabeled_error_multi} and \lemref{lem:clear_error_multi} to obtain the final inequality in terms of quantities that can be computed from just the randomly labeled and clean data. Similar to the binary case, we obtained a polynomial concentration instead of exponential concentration. Combining \eqref{eq:lemma1_linear_final_multi} with \lemref{lem:mislabeled_error_multi} and \lemref{lem:clear_error_multi}, we have with probability at least $1-\delta$
    \begin{align*}
        \error_\calD(\widehat f) \le \error_\calS(\widehat f) &+ (k-1)\left(1 - \frac{k}{k-1} \error_{\wt\calS}(\widehat f) \right) \\ &+ \left(k + \sqrt{k} + \frac{m}{n\sqrt{k}} \right) \sqrt{\frac{\log(4/\delta)}{2m}} + \sqrt{k(k-1)} \sqrt{\frac{4}{\delta}\left(\frac{1}{m} +\frac{3\beta}{m+n} \right)}  \,. \numberthis \label{eq:gd_multi_error_proof}
    \end{align*} 
\end{proof}

\subsection{Discussion on \codref{cond:hypothesis_stability}} \label{app:discuss_cond1}
The quantity in LHS of \codref{cond:hypothesis_stability} 
measures how much the function learned by the algorithm 
(in terms of error on unseen point) will change 
when one point in the training set is removed. 
% Discussion on exponential concentration and stronger condition. 
% Notice that hypothesis stability implies error stability, i.e., \codref{cond:error_stability} \citep{bousquet2002stability}.  
% In summary, while error stability allowed us 
% to relate the average population error 
% of the leave-one-out classifiers 
% with the population error of the original classifier, 
We need hypothesis stability condition 
to control the variance of the empirical leave-one-out error to show concentration of average leave-one-error with the population error. 

Additionally, we note that while the dominating term in the RHS of \thmref{thm:linear} matches with the dominating term in ERM bound in \thmref{thm:error_ERM}, there is a polynomial concentration term 
(dependence on $1/\delta$ instead of $\log(\sqrt{1/\delta})$) 
in \thmref{thm:linear}. 
Since with hypothesis stability, 
we just bound the variance, 
the polynomial concentration is due 
to the use of Chebyshev's inequality 
instead of an exponential tail inequality
(as in \lemref{lem:fit_mislabeled}).
Recent works have highlighted that 
a slightly stronger condition than hypothesis stability 
can be used to obtain an exponential concentration 
for leave-one-out error \citep{abou2019exponential},
but we leave this for future work for now. 
% We leave 
% However, the constants 

% we also want to highlight  

\subsection{Formal statement and proof of \propref{prop:early_stop}} \label{app:formal_early_stop}

Before formally presenting the result, 
we will introduce some notation.  
By $\calL_{S}(w)$, we denote 
the objective in \eqref{eq:l2_MSE_app} with $\lambda=0$. 
Assume Singular Value Decomposition (SVD) of $\bX$
as $\sqrt{n} \bU \bS^{1/2} \bV^T$. 
Hence $\bX^T \bX = \bV \bS \bV^T$.
Consider the GD iterates defined in \eqref{eq:GD_iterates_app}. 
We now derive closed form expression 
for the $t^\text{th}$ iterate of gradient descent:  
\begin{align}
    w_t = w_{t-1} + \eta \cdot \bX^T (\by - \bX w_{t-1}) = (\bI - \eta \bV \bS \bV^T )w_{k-1} + \eta \bX^T \by \,.
\end{align}
Rotating by $\bV^T$, we get 
\begin{align}
    \wt w_t = (\bI - \eta\bS )\wt w_{k-1} + \eta \wt \by \label{eq:GD_recur},
\end{align}
where $\wt w_t = \bV^T w_t $ and $\wt \by = \bV^T \bX^T \by$. 
Assuming the initial point $w_0 = 0$ 
and applying the recursion in \eqref{eq:GD_recur}, we get
\begin{align}
    \wt w_t = \bS ^{-1} ( \bI - (\bI - \eta \bS)^k ) \wt \by \,, 
\end{align} 
Projecting solution back to the original space, we have 
\begin{align}
     w_t = \bV \bS ^{-1} ( \bI - (\bI - \eta \bS)^k ) \bV^T \bX^T \by \,. 
\end{align} 
% We will work with this GD solution at any iterate $t$ in the next proposition. 
Define $f_t(x) \defeq f(x;w_t)$ 
as the solution at the $t^{\text{th}}$ iterate. 
Let $\wt w_{\lambda} = \argmin_{w} \calL_\calS (w;\lambda) = (\bX^T \bX + \lambda \bI)^{-1} \bX^T \by = \bV (\bS + \lambda \bI )^{-1} \bV^T \bX^T \by $. 
% ) \,,$ for all $t=1,2,\ldots\,.$ 
and define $\wt f_\lambda(x) \defeq f(x;\wt w_\lambda)$ as the regularized solution. 
Assume $\kappa$ be the condition number 
of the population covariance matrix 
and let $s_\text{min}$ be the minimum positive 
singular value of the empirical covariance matrix. 
Our proof idea is inspired from recent work 
on relating gradient flow solution 
and regularized solution 
for regression problems \citep{ali2018continuous}. 
We will use the following lemma in the proof: 
\begin{lemma} \label{lem:ineq_soln}
    For all $x \in [0,1]$ and for all $ k \in \mathbb{N}$, 
    we have (a) $ \frac{kx}{1+kx} \le 1- (1-x)^k$ 
    and (b) $ 1- (1-x)^k \le 2 \cdot \frac{kx}{kx+1} $.
    %  where $g(c)$ is a constant dependent on $c$. For $c = 1$, $g(c) = 2.0$.   
\end{lemma}
\begin{proof}
    % [Proof of \lemref{lem:ineq_soln}]
    % Part (a) is easy. 
    Using $ (1-x)^k \le \frac{1}{1+kx}$, we have part (a). 
    For part (b), we numerically maximize 
    $\frac{ (1+kx ) (1 - (1-x)^k) }{kx}$ 
    for all $k\ge 1$ and for all $x \in [0, 1]$.  
\end{proof}

% 
% Next, 

\begin{prop}[Formal statement of \propref{prop:early_stop}] \label{prop:formal_early_stop}
Let $\lambda = \frac{1}{t\eta}$. 
For a training point $x$, we have 
\begin{align*}
    \Expt{x \sim \calS}{(f_t(x) - \wt f_\lambda(x))^2} &\le c(t,\eta) \cdot \Expt{x \sim \calS}{f_t(x)^2} \,, %\label{eq:early_stop}
\end{align*}
where $c(t, \eta) \defeq \min( 0.25, \frac{1}{s_\text{min}^2 t^2 \eta^2})$. 
Similarly for a test point, we have 
\begin{align*}
    \Expt{x \sim \calD_\calX}{(f_t(x) - \wt f_\lambda(x))^2} &\le \kappa \cdot c(t,\eta) \cdot \Expt{x \sim \calD_\calX}{f_t(x)^2} \,. %\label{eq:early_stop}
\end{align*}
\end{prop} 

\begin{proof}
    %%%%%%%%%%%%% 
    We want to analyze the expected squared difference output 
    of regularized linear regression 
    with regularization constant $\lambda = \frac{1}{\eta t}$ 
    and the gradient descent solution at the $t^\text{th}$ iterate. 
    We separately expand the algebraic expression 
    for squared difference at a training point and a test point. 
    % We start by considering the difference  
    Then the main step is to show that 
    $\left[ \bS ^{-1} ( \bI - (\bI - \eta \bS)^k )  - (\bS + \lambda \bI )^{-1}\right] \preceq c(\eta, t) \cdot \bS ^{-1} ( \bI - (\bI - \eta \bS)^k ) $.

    %%%%%%%%%%%%%
    
   \textbf{Part 1 {} {}} 
    First, we will analyze the squared difference 
    of the output at a training point 
    (for simplicity, we refer to $S \cup \wt S$ as $S$), i.e., 
    \begin{align}
        \Expt{ x \sim \calS }{\left(f_t(x) - \wt f_\lambda (x)\right)^2} &= \norm{\bX w_t - \bX \wt w_\lambda}{2}^2\\ &=   \norm{\bX \bV \bS ^{-1} ( \bI - (\bI - \eta \bS)^t ) \bV^T \bX^T \by - \bX \bV (\bS + \lambda \bI )^{-1} \bV^T \bX^T \by }{2}^2 \\
        &= \norm{\bX \bV \left(\bS ^{-1} ( \bI - (\bI - \eta \bS)^t ) - (\bS + \lambda \bI )^{-1} \right) \bV^T \bX^T \by  }{2} \\
        &=  \by^T \bV \bX \left( \underbrace{\bS ^{-1} ( \bI - (\bI - \eta \bS)^t ) - (\bS + \lambda \bI )^{-1}}_{\RN{1}} \right)^2 \bS \bV^T \bX^T \by \label{eq:train_GD_rel} \,.
        %  (\bX \bV \bS ^{-1} ( \bI - (\bI - \eta \bS)^k ) \bV^T \bX^T \by)^T \bX \bV \bS ^{-1} ( \bI - (\bI - \eta \bS)^k ) \bV^T \bX^T \by
    \end{align}
    We now separately consider term 1. 
    Substituting $\lambda = \frac{1}{t \eta}$, 
    we get
    \begin{align}
        \bS ^{-1} ( \bI - (\bI - \eta \bS)^t ) - (\bS + \lambda \bI )^{-1} &= \bS^{-1} \left( ( \bI - (\bI - \eta \bS)^t ) - (\bI + \bS^{-1} \lambda )^{-1}\right) \\
        &= \underbrace{\bS^{-1} \left( ( \bI - (\bI - \eta \bS)^t ) - (\bI + ( \bS t \eta)^{-1}  )^{-1}\right)}_{\bA} \,.
    \end{align}

    We now separately bound the diagonal entries in matrix $\bA$. 
    With $s_i$, we denote $i^{\text{th}}$ diagonal entry of $\bS$.
    Note that since $ \eta\le 1/\norm{S}{\text{op}}$, 
    for all $i$, $\eta s_i  \le 1$.  
    Consider $i^{\text{th}}$ diagonal term (which is non-zero) 
    of the diagonal matrix $\bA$, we have 
    \begin{align}
        \bA_{ii} = \frac{1}{s_i} \left(  1 - (1 - s_i \eta)^t - \frac{t \eta s_i}{1 + t \eta s_i } \right) &=  \frac{1 - (1 - s_i \eta)^t}{s_i} \left( \underbrace{ 1 - \frac{t \eta s_i}{(1 + t \eta s_i)(1 - (1 - s_i \eta)^t)}}_{\RN{2}} \right) \\ 
         &\le \frac{1}{2}\left[ \frac{1 - (1 - s_i \eta)^t}{ s_i} \right] \tag*{(Using \lemref{lem:ineq_soln} (b))} \,.
    \end{align} 
    Additionally, we can also show the following upper bound on term 2: 
    \begin{align}
         1 - \frac{t \eta s_i}{(1 + t \eta s_i)(1 - (1 - s_i \eta)^t)} &= \frac{(1 + t \eta s_i)(1 - (1 - s_i \eta)^t) - t \eta s_i }{(1 + t \eta s_i)(1 - (1 - s_i \eta)^t)} \\
         & \le  \frac{ 1 -  (1 - s_i \eta)^t - t \eta s_i (1 - s_i \eta)^t}{(1 + t \eta s_i)(1 - (1 - s_i \eta)^t)} \\
         & \le \frac{1}{t\eta s_i} \,. \tag{Using \lemref{lem:ineq_soln} (a)}
        %  &\le \frac{1}{2}\left[ \frac{1 - (1 - s_i \eta)^t}{ s_i} \right] \tag*{(Using \lemref{lem:ineq_soln})} \,.
    \end{align} 

    Combining both the upper bounds 
    on each diagonal entry $\bA_{ii}$, we have 
    \begin{align}
    \bA \preceq c_1(\eta, t) \cdot \bS^{-1} ( \bI - (\bI - \eta \bS)^t ) \,, \label{eq:upperbound_diagonal}
    \end{align}
    where $c_1(\eta, t ) = \min(0.5, \frac{1}{t s_i \eta })$. Plugging this into \eqref{eq:train_GD_rel}, we have 
    \begin{align}
        \Expt{ x \sim \calS }{\left(f_t(x) - \wt f_\lambda (x)\right)^2} &\le c(\eta, t) \cdot \by^T \bV \bX  \left( \bS^{-1} ( \bI - (\bI - \eta \bS)^t ) \right)^2 \bS \bV^T \bX^T \by \\
        &=   c(\eta, t) \cdot \by^T \bV \bX  \left( \bS^{-1} ( \bI - (\bI - \eta \bS)^t ) \right) \bS \left( \bS^{-1} ( \bI - (\bI - \eta \bS)^t ) \right) \bV^T \bX^T \by \\
        & =  c(\eta, t) \cdot \norm{\bX w_t}{2}^2 \\
        &= c(\eta, t) \cdot  \Expt{ x \sim \calS }{\left(f_t(x) \right)^2} \,,
    \end{align}
    where $c(\eta, t ) = \min(0.25, \frac{1}{t^2 s^2_i \eta^2 })$.

    \textbf{Part 2 {} {}} With $\bSigma$, 
    we denote the underlying true covariance matrix. 
    We now consider the squared difference of output at an unseen point: 
    \begin{align}
        \Expt{ x \sim \calD_{\calX} }{\left(f_t(x) - \wt f_\lambda (x)\right)^2} &= \Expt{x \sim \calD_{\calX}}{\norm{x^T w_t - x^T \wt w_\lambda}{2}} \\
        &=   \norm{x^T \bV \bS ^{-1} ( \bI - (\bI - \eta \bS)^t ) \bV^T \bX^T \by - x^T \bV (\bS + \lambda \bI )^{-1} \bV^T \bX^T \by }{2} \\
        &= \norm{x^T \bV \left(\bS ^{-1} ( \bI - (\bI - \eta \bS)^t ) - (\bS + \lambda \bI )^{-1} \right) \bV^T \bX^T \by  }{2} \\
        &= \by^T \bV \bX \left( \bS ^{-1} ( \bI - (\bI - \eta \bS)^t ) - (\bS + \lambda \bI )^{-1} \right) \bV^T \bSigma \bV \\ &\qquad \qquad \qquad \qquad \qquad \left( (\bI - (\bI - \eta \bS)^t ) - (\bS + \lambda \bI )^{-1} \right) \bV^T \bX^T \by \\
        &\le \sigma_{\text{max}} \cdot \by^T \bV \bX \left( \underbrace{\bS ^{-1} ( \bI - (\bI - \eta \bS)^t ) - (\bS + \lambda \bI )^{-1}}_{\RN{1}} \right)^2 \bV^T \bX^T \by \,, \label{eq:test_GD_rel}
        %  (\bX \bV \bS ^{-1} ( \bI - (\bI - \eta \bS)^k ) \bV^T \bX^T \by)^T \bX \bV \bS ^{-1} ( \bI - (\bI - \eta \bS)^k ) \bV^T \bX^T \by
    \end{align}
    where $\sigma_{\text{max}}$ is the maximum eigenvalue 
    of the underlying covariance matrix $\bSigma$. 
    Using the upper bound on term 1 in \eqref{eq:upperbound_diagonal}, 
    we have 
    \begin{align}
        \Expt{ x \sim \calD_{\calX} }{\left(f_t(x) - \wt f_\lambda (x)\right)^2} &\le \sigma_{\text{max}} \cdot c(\eta, t) \cdot \by^T \bV \bX  \left( \bS^{-1} ( \bI - (\bI - \eta \bS)^t ) \right)^2 \bV^T \bX^T \by \\
        &=   \kappa \cdot c(\eta, t) \cdot \sigma_{\text{min}}\cdot \norm{\bV \left( \bS^{-1} ( \bI - (\bI - \eta \bS)^t ) \right) \bV^T \bX^T \by}{2}^2 \\
        &\le \kappa \cdot c(\eta, t) \cdot \left[ \bV \left( \bS^{-1} ( \bI - (\bI - \eta \bS)^t ) \right) \bV^T \bX^T \right]^T \bSigma \\
        &\qquad \qquad \qquad \qquad \qquad \left[ \bV \left( \bS^{-1} ( \bI - (\bI - \eta \bS)^t ) \right) \bV^T \bX^T \right] \by \\
        & = \kappa \cdot c(\eta, t) \cdot \Expt{x \sim \calD_{\calX}}{\norm{x^T w_t}{2}} \,.
    \end{align}
% 
% 
    % Since $ \eta\le 1/\norm{S}{\text{op}}$, invoking \lemref{lem:ineq_soln} to upper bound term 1 with
\end{proof}

\subsection{Extension to deep learning} \label{appsubsec:ext_DL}
Under \asmpref{appsubsec:justifying_assumption1}, we present the formal result parallel to \thmref{thm:multiclass_ERM}. 
\begin{theorem} \label{thm:multiclass_ERM_algoA}
    Consider a multiclass classification problem 
    with $k$ classes. Under \asmpref{asmp:deep_models}, 
    for any $\delta >0$, with probability at least $1-\delta$,
    we have
    \vspace{-10pt}
    \begin{align*}
        \error_\calD(\widehat f)  \le \error_\calS(\widehat f) + (k-1) \left(1 - \tfrac{k}{k-1} \error_{\wt\calS}(\widehat f)\right) + c\sqrt{\frac{\log(\frac{4}{\delta})}{2m}} \,,\numberthis \label{eq:multiclass_ERM_deep}
    % \vspace{-20pt}
    \end{align*}
    for some constant $c \le ((c+1) k+\sqrt{k} + \frac{m}{n\sqrt{k}})$.
\end{theorem}

The proof follows exactly as in step (i) to (iii) in \thmref{thm:multiclass_ERM}.  

\subsection{Justifying~\asmpref{asmp:deep_models}} \label{appsubsec:justifying_assumption1}

Motivated by the analysis on linear models, we now discuss alternate (and weaker) conditions that imply \asmpref{asmp:deep_models}. 
We need hypothesis stability (\codref{cond:hypothesis_stability}) and the following assumption relating training error and leave-one-error: 

\begin{assumption} \label{asmp:loo_error}
Let $\wh f$ be a model obtained by training with algorithm $\calA$ on a mixture of clean $S$ and randomly labeled data $\wt S$. Then we assume we have 
\begin{align*}
    \error_{\wt \calS_M} (\wh f) \le  \error_{\text{LOO} (\wt S_M)} \,, 
\end{align*}
for all $(x_i, y_i) \in  \wt S_M$ where $\wh f_{(i)} \defeq f(\calA, S \cup {{}\wt S_M}_{(i)})$ and  $\error_{\text{LOO} (\wt S_M)} \defeq  \frac{\sum_{(x_i, y_i) \in \wt S_M} \error(\ff{i}(x_i), y_i ) }{\abs{\wt \calS_M}}$.  
\end{assumption}

% we assume this to extend our result (parallel to \thmref{thm:multi_linear}) for deep models. 
Intuitively, this assumption states that the error on a (mislabeled) datum $(x,y)$ included in the training set is less than the error on that datum $(x,y)$ obtained by a model trained on the training set $S - \{(x,y)\}$. We proved this for linear models trained with GD in the proof of \thmref{thm:multi_linear}. 
\codref{cond:hypothesis_stability} with $\beta = \calO(1)$ and \asmpref{asmp:loo_error} together with \lemref{lem:stability_error} implies \asmpref{asmp:deep_models} with a polynomial residual term (instead of logarithmic in $1/\delta$): 
\begin{align}
     \error_{\calS_M} (\wh f) \le  \error_{\calDm}(\wh f)   + \sqrt{\frac{1}{\delta}\left(\frac{1}{m} +\frac{3\beta}{m+n} \right)} \,.
\end{align}
% Note that this  

\newpage 
\section{Additional experiments and details}\label{app:exp}
\newcommand\tab[1][1cm]{\hspace*{#1}}

\subsection{Datasets} \label{sec:app_dataset}

\textbf{Toy Dataset {} {}} Assume fixed constants $\mu$ and $\sigma$. For a given label $y$, we simulate features $x$ in our toy classification setup as follows: 
\begin{align*}
    x \defeq \texttt{concat} \left[ x_1, x_2\right] \quad \text{where} \quad  x_1 \sim  \calN( y \cdot \mu, \sigma^2 I_{d \times d}) \ \  \text{and} \ \  x_1 \sim  \calN( 0, \sigma^2 I_{d \times d}) \,.
\end{align*}  
% where $y$ is the true label and $x$ is the corresponding feature vector. 
In experiements throughout the paper, we fix dimention $d=100$, $\mu = 1.0 $, and $\sigma = \sqrt{d}$. Intuitively, $x_1$ carries the information about the underlying label and $x_2$ is additional noise independent of the underlying label. 

\textbf{CV datasets {} {}} We use MNIST~\citep{lecun1998mnist} and CIFAR10~\cite{krizhevsky2009learning}. 
% For binary tasks, 
We produce a binary variant from the multiclass classification problem by mapping classes $\{0,1,2,3,4\}$ to label $1$ and $\{ 5,6,7,8,9\}$ to label $-1$. For CIFAR dataset, we also use the standard data augementation of random crop and horizontal flip. PyTorch code is as follows: 

\texttt{(transforms.RandomCrop(32, padding=4),\\
\tab transforms.RandomHorizontalFlip())}

\textbf{NLP dataset {} {}} We use IMDb Sentiment analysis~\citep{maas2011learning} corpus.  

\subsection{Architecture Details} 

All experiments were run on NVIDIA GeForce RTX 2080 Ti GPUs. We used PyTorch~\citep{NEURIPS2019a9015} and Keras with Tensorflow~\citep{abadi2016tensorflow} backend for experiments. 
% , ELMo embeddings~\citep{Peters:2018}, and Hugging Face Transformers~\citep{wolf-etal-2020-transformers}. 

\textbf{Linear model {} {}} For the toy dataset, we simulate a linear model with scalar output and the same number of parameters as the number of dimensions.   

\textbf{Wide nets {} {}} To simulate the NTK regime, we experiment with $2-$layered wide nets. The PyTorch code for 2-layer wide MLP is as follows:

\texttt{ nn.Sequential( \\
\tab     nn.Flatten(),\\
\tab    nn.Linear(input\_dims, 200000, bias=True),\\
\tab    nn.ReLU(),\\
\tab    nn.Linear(200000, 1, bias=True)\\
\tab     )}

We experiment both (i) with the second layer fixed at random initialization; (ii)  and updating both layers' weights.     

\textbf{Deep nets for CV tasks {} {}} We consider a 4-layered MLP. The PyTorch code for 4-layer MLP is as follows: 

\texttt{ nn.Sequential(nn.Flatten(), \\
\tab        nn.Linear(input\_dim, 5000, bias=True),\\
\tab        nn.ReLU(),\\
\tab        nn.Linear(5000, 5000, bias=True),\\
\tab        nn.ReLU(),\\
\tab        nn.Linear(5000, 5000, bias=True),\\
\tab        nn.ReLU(),\\
% \tab        nn.Linear(5000, 5000, bias=True),\\
% \tab        nn.ReLU(),\\
\tab        nn.Linear(1024, num\_label, bias=True)\\
\tab        )}

For MNIST, we use $1000$ nodes instead of $5000$ nodes in the hidden layer. 
We also experiment with convolutional nets. In particular, we use ResNet18 \citep{he2016deep}. Implementation adapted from:  \url{https://github.com/kuangliu/pytorch-cifar.git}. 

\textbf{Deep nets for NLP {} {}} We use a simple LSTM model with embeddings intialized with ELMo embeddings~\citep{Peters:2018}. Code adapted from: \url{https://github.com/kamujun/elmo_experiments/blob/master/elmo_experiment/notebooks/elmo_text_classification_on_imdb.ipynb} 

We also evaluate our bounds with a BERT model. In particular, we fine-tune an off-the-shelf uncased BERT model~\citep{devlin2018bert}. Code adapted from Hugging Face Transformers~\citep{wolf-etal-2020-transformers}: \url{https://huggingface.co/transformers/v3.1.0/custom_datasets.html}.

\subsection{Additonal experiments}

\textbf{Results with SGD on underparameterized linear models {} {}} 

\begin{figure*}[h]
    \centering 
    % \vspace{-15pt}
    % \includegraphics[width=0.9\linewidth]{example-image-a}
    \includegraphics[width=0.3\linewidth]{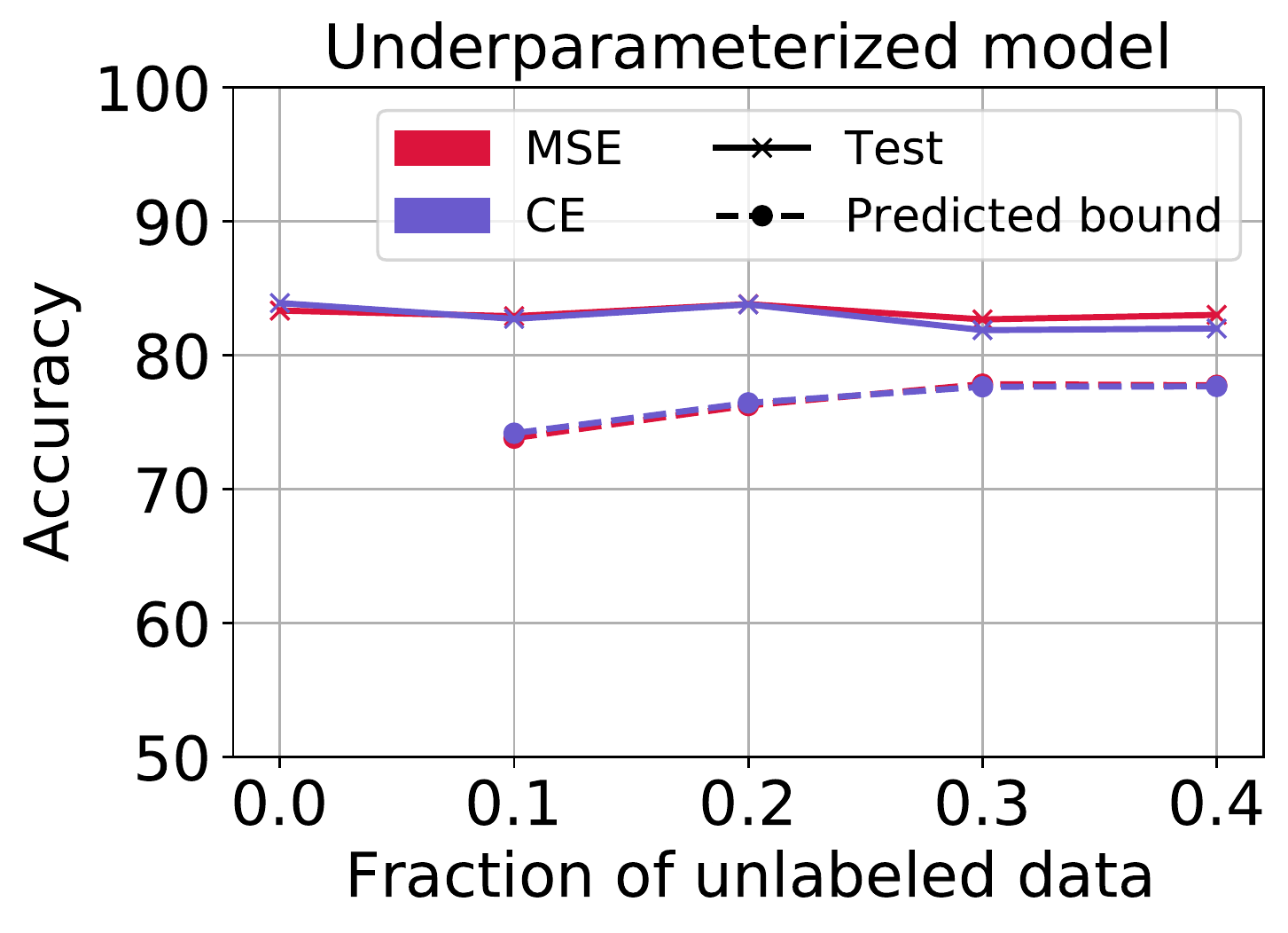}
    \vspace{-5pt}
    \caption{ 
    % Predicted lower bound 
    % on different
    We plot the accuracy and corresponding bound 
    (RHS in \eqref{eq:erm}) at $\delta = 0.1$
    for toy binary classification task. 
    Results aggregated over $3$ seeds. 
    % i.e., $1-\error$ where $\error$ is the term in the RHS of \eqref{eq:erm}
    Accuracy vs fraction of unlabeled data (w.r.t clean data) 
    in the toy setup with a linear model trained with SGD. Results parallel to \figref{fig:error_binary}(a) with SGD.  }
    \label{fig:error_binary_linear}
    \vspace{-5pt}
\end{figure*}

\textbf{Results with wide nets on binary MNIST {} {}}

\begin{figure*}[h]
    \centering 
    % \vspace{-15pt}
    % \includegraphics[width=0.9\linewidth]{example-image-a}
    \subfigure[GD with MSE loss]{\includegraphics[width=0.3\linewidth]{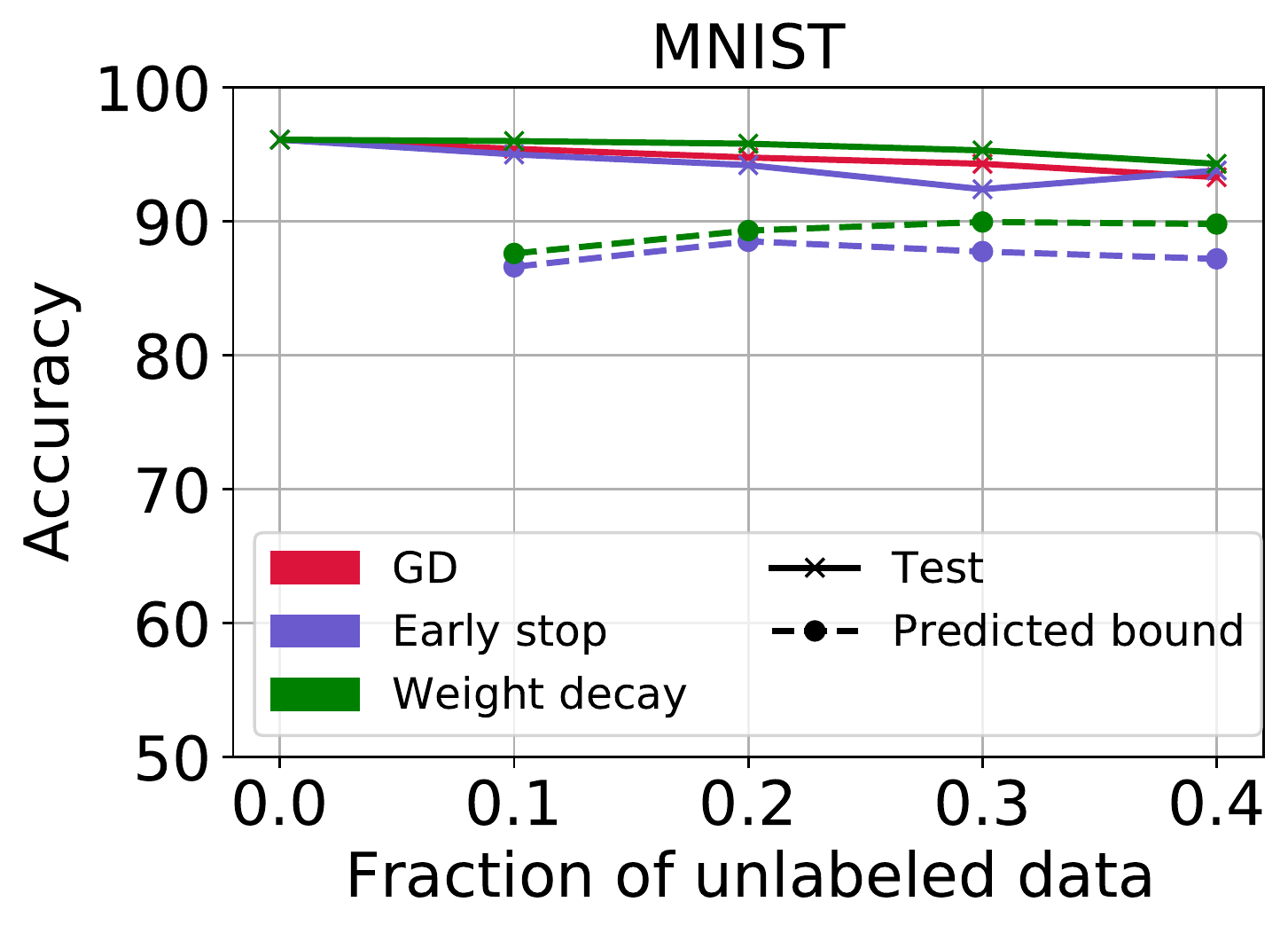}} \hfil
    \subfigure[SGD with CE loss]{\includegraphics[width=0.3\linewidth]{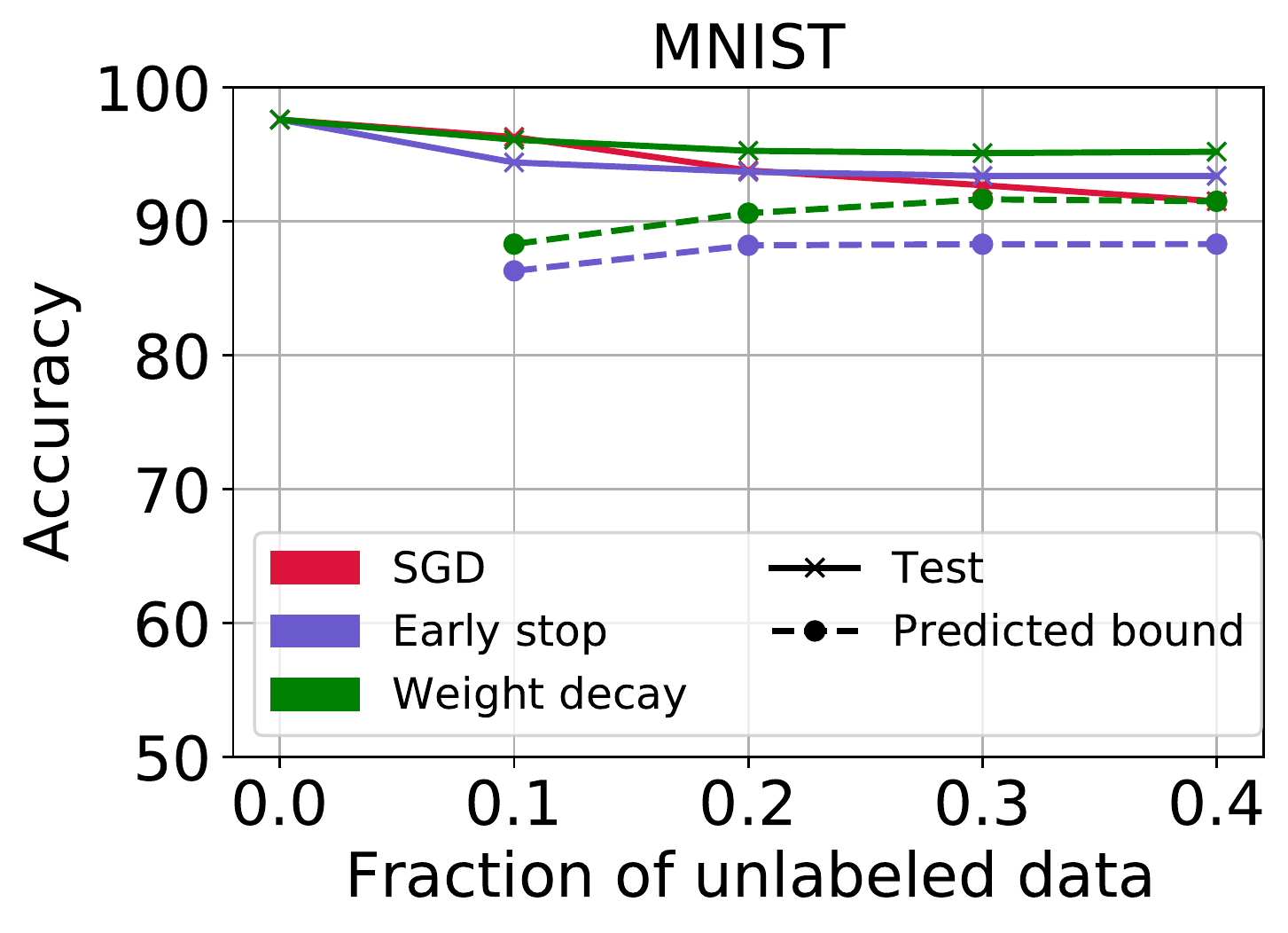}}
    \subfigure[SGD with MSE loss]{\includegraphics[width=0.3\linewidth]{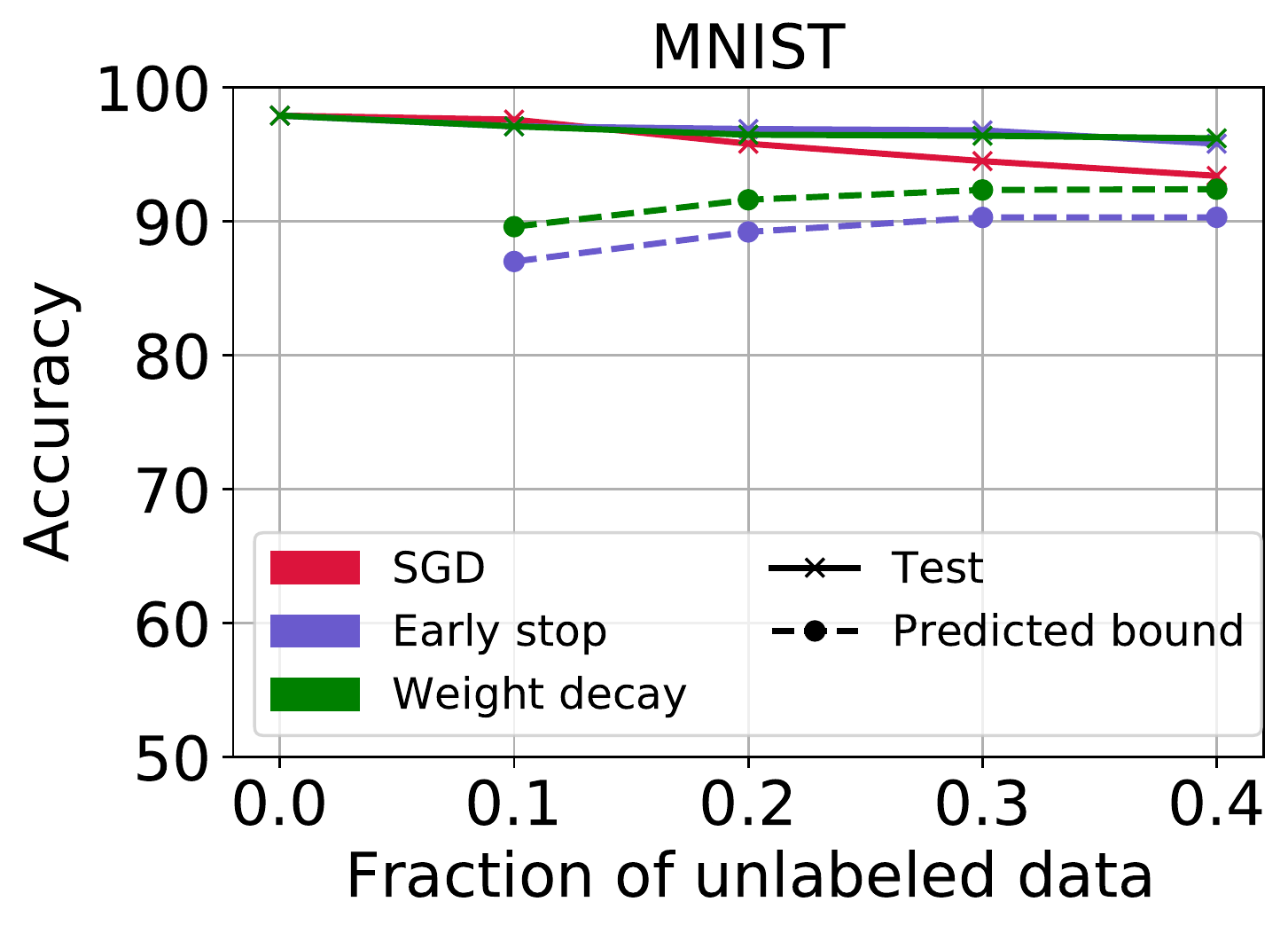}}
    \vspace{-5pt}
    \caption{ 
    % Predicted lower bound 
    % on different
    We plot the accuracy and corresponding bound 
    (RHS in \eqref{eq:erm}) at $\delta = 0.1$ 
    for binary MNIST classification. 
    Results aggregated over $3$ seeds. 
    % i.e., $1-\error$ where $\error$ is the term in the RHS of \eqref{eq:erm}
    Accuracy vs fraction of unlabeled data 
    for a 2-layer wide network on binary MNIST with both the layers training in (a,b) and only first layer training in (c). 
    Results parallel to \figref{fig:error_binary}(b) .  }
    \label{fig:error_binary_MNIST}
    \vspace{-5pt}
\end{figure*}

% \begin{figure*}[h]
%     \centering 
%     % \vspace{-15pt}
%     % \includegraphics[width=0.9\linewidth]{example-image-a}
%     \subfigure[GD with MSE loss]{\includegraphics[width=0.3\linewidth]{figures/MNIST.pdf}} \hfil
    
%     \subfigure[SGD with CE loss]{\includegraphics[width=0.3\linewidth]{figures/MNIST.pdf}}
%     % \includegraphics[width=0.9\linewidth]{figures/{CIFAR10_rn=0.1_lr=0.2_wd=0.005}.png}
%     \vspace{-5pt}
%     \caption{ 
%     % Predicted lower bound 
%     % on different
%     We plot the accuracy and corresponding bound 
%     (RHS in \eqref{eq:erm}) at $\delta = 0.1$
%     for binary MNIST classification. 
%     Results aggregated over $3$ seeds. 
%     % i.e., $1-\error$ where $\error$ is the term in the RHS of \eqref{eq:erm}
%     Accuracy vs fraction of unlabeled data 
%     for a 2-layer wide network on binary MNIST with just the first layer training. 
%     Results parallel to \figref{fig:error_binary}(b) with only the first layer training.  }
%     \label{fig:error_binary_MNIST}
%     \vspace{-5pt}
% \end{figure*}

\textbf{Results on CIFAR 10 and MNIST {} {}} 
We plot epoch wise error curve for results in \tabref{table:multiclass}(\figref{fig:error_epoch_CIFAR10} and \figref{fig:error_epoch_MNIST}). We observe the same trend as in \figref{fig:error_CIFAR10}. Additionally, we plot an \emph{oracle bound} obtained by tracking the error on mislabeled data which nevertheless were predicted as true label. To obtain an exact emprical value of the oracle bound, we need underlying true labels for the randomly labeled data. 
% Note that our bound in \thmref{thm:multiclass_ERM}, lower bounds the accuracy as predicted by the oracle bound. 
While with just access to extra unlabeled data we cannot calculate oracle bound, we note that the oracle bound is very tight and never violated in practice underscoring an importamt aspect of generalization in multiclass problems. This highlight that even a stronger conjecture may hold in multiclass classification, i.e., error on mislabeled data (where nevertheless true label was predicted) lower bounds the population error on the distribution of mislabeled data and hence, the error on (a specific) mislabeled portion predicts the population accuracy on clean data. 
On the other hand, the dominating term of in \thmref{thm:multiclass_ERM} is loose when compared with the oracle bound. The main reason, we believe is the pessimistic upper bound in \eqref{eq:lemma1_final_multi_prev} in the proof of \lemref{lem:fit_mislabeled_multi}. We leave an investigation on this gap for future. 
% of fit 

% However, oracle bound highlights two . One,  

\begin{figure}[h]
    \centering 
    % \vspace{-15pt}
    % \includegraphics[width=0.9\linewidth]{example-image-a}
    \subfigure[MLP]{\includegraphics[width=0.3\linewidth]{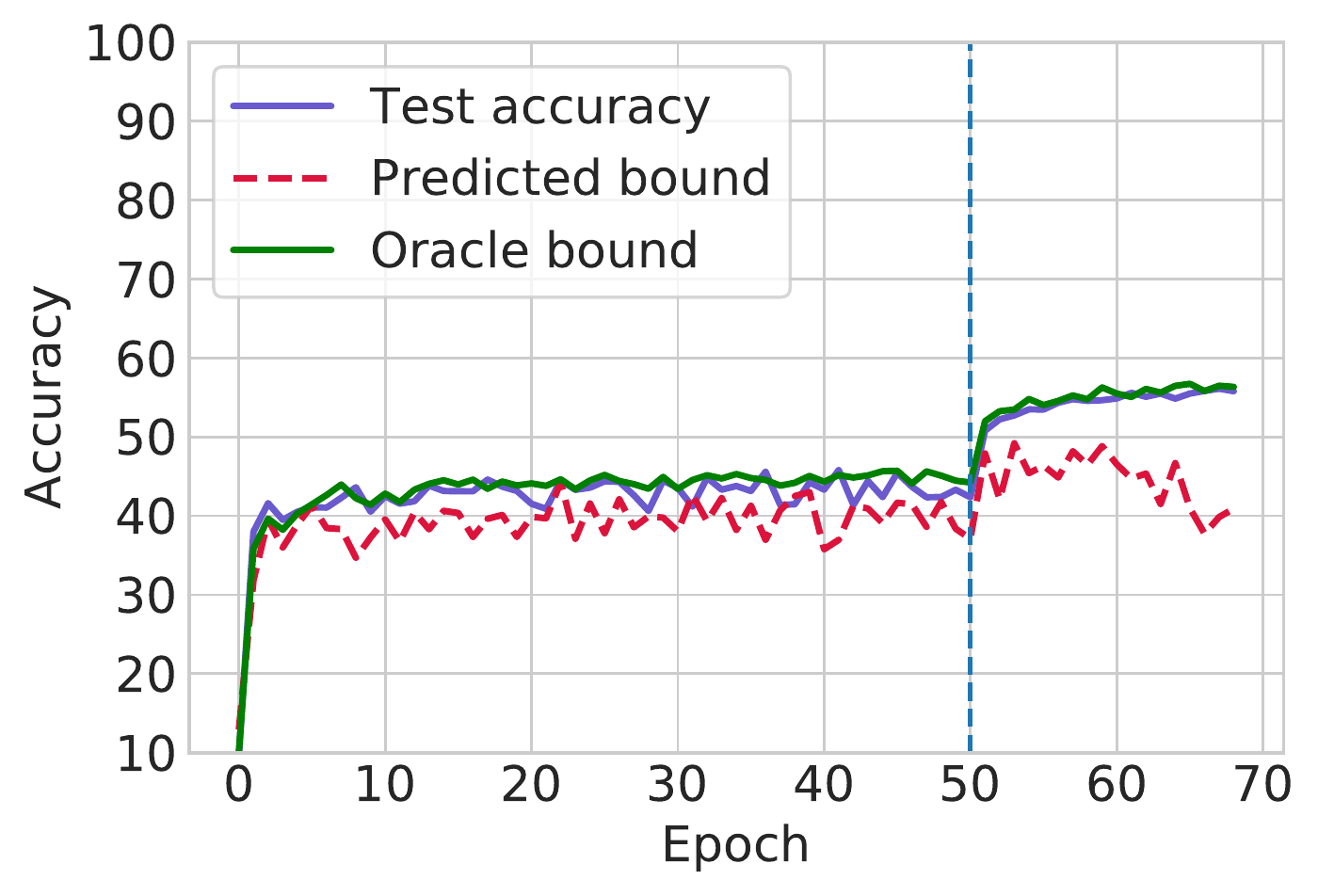}} \hfil
    \subfigure[ResNet]{\includegraphics[width=0.3\linewidth]{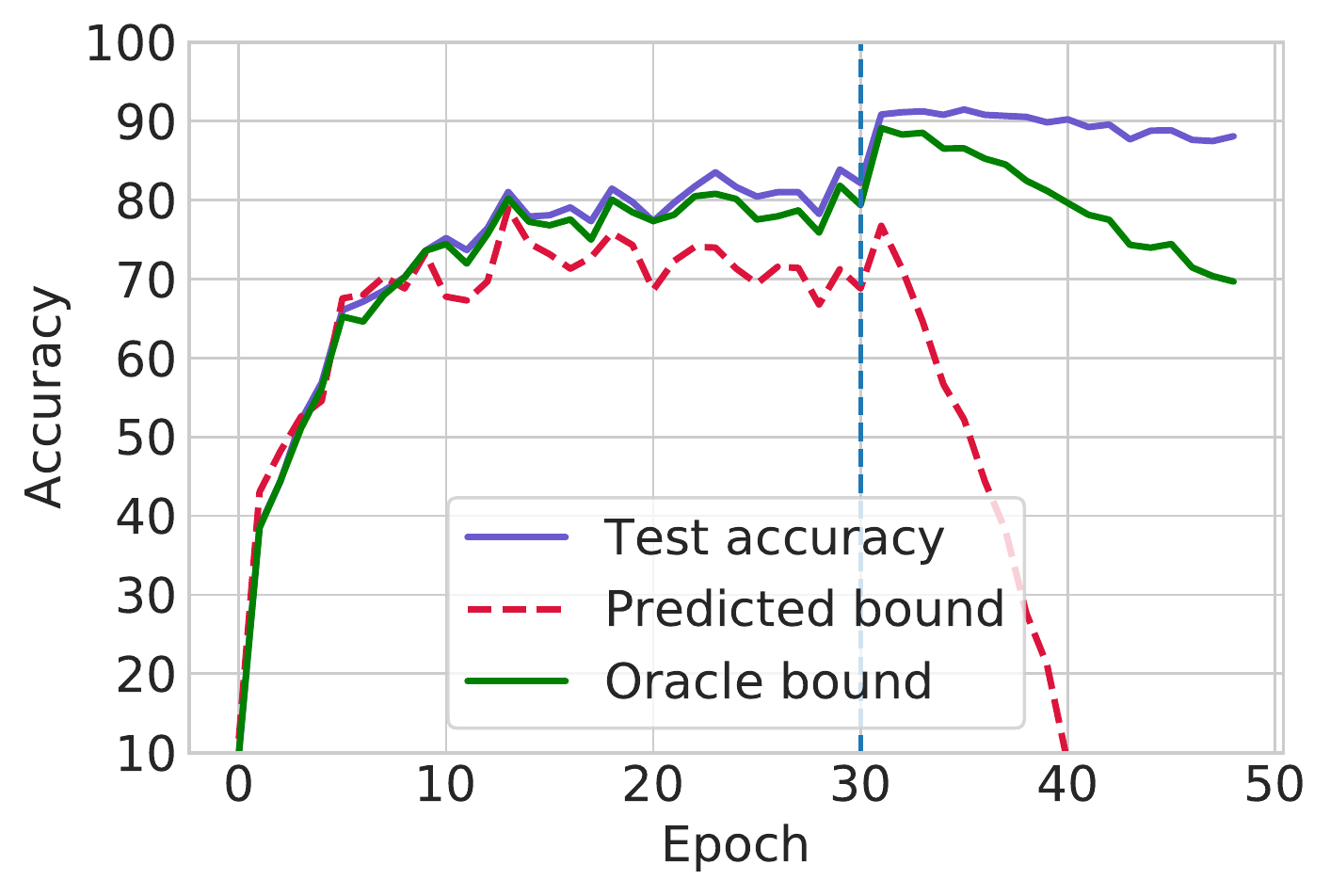}}
    % \includegraphics[width=0.9\linewidth]{figures/{CIFAR10_rn=0.1_lr=0.2_wd=0.005}.png}
    % \vspace{-10pt}
    \caption{ Per epoch curves for CIFAR10 corresponding results in \tabref{table:multiclass}. As before, we just plot the dominating term in the RHS of \eqref{eq:multiclass_ERM} as predicted bound. Additionally, we also plot the predicted lower bound by the error on mislabeled data which nevertheless were predicted as true label. We refer to this as ``Oracle bound''. See text for more details. 
    % 
    % except for the stopping point. 
    % The bound predicted by RATT (RHS in \eqref{eq:multiclass_ERM}) is vacuous. 
    }\label{fig:error_epoch_CIFAR10}
    % \vspace{-15pt}
\end{figure}

\begin{figure}[h]
    \centering 
    % \vspace{-15pt}
    % \includegraphics[width=0.9\linewidth]{example-image-a}
    \subfigure[MLP]{\includegraphics[width=0.3\linewidth]{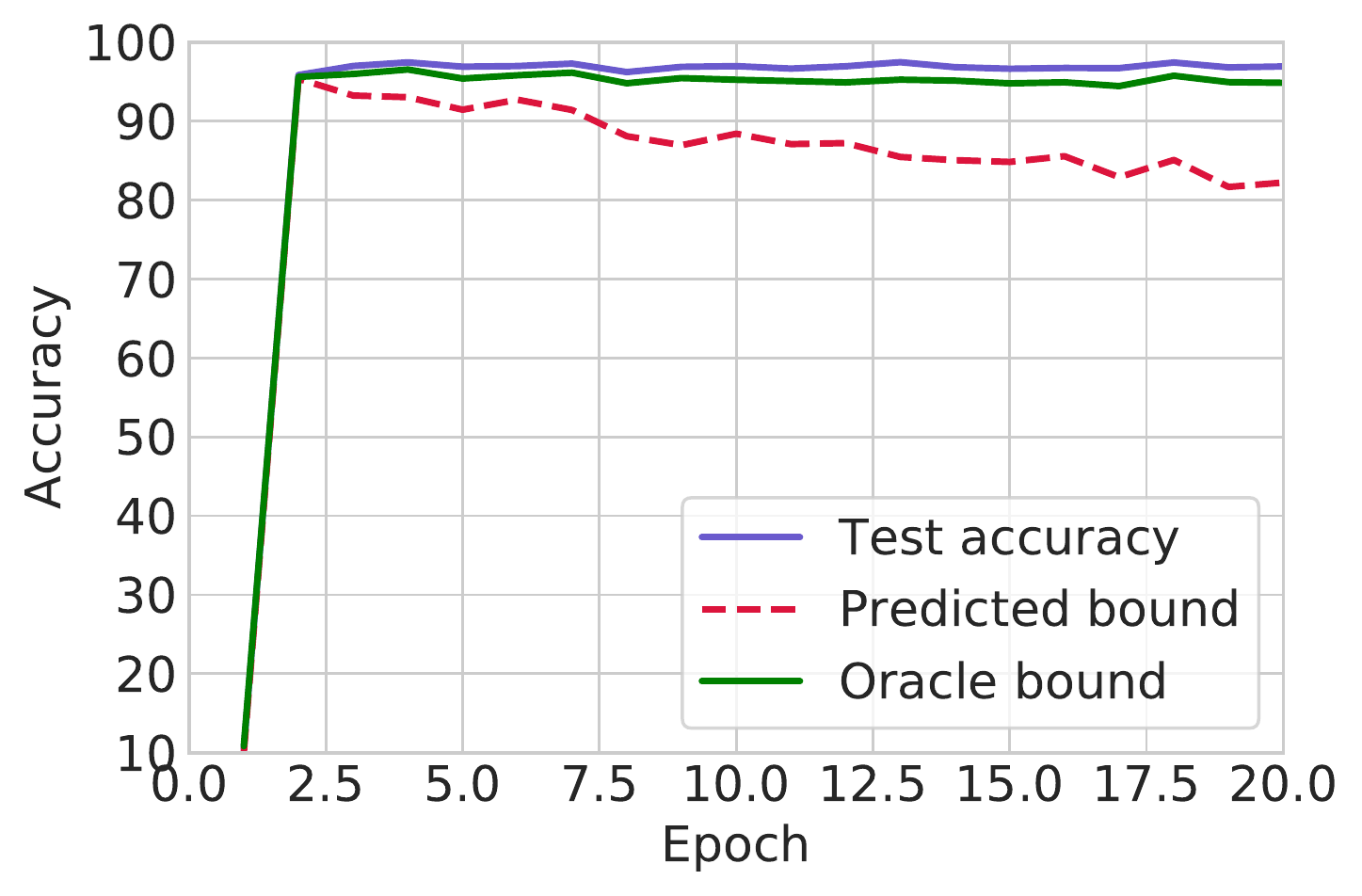}} \hfil
    \subfigure[ResNet]{\includegraphics[width=0.3\linewidth]{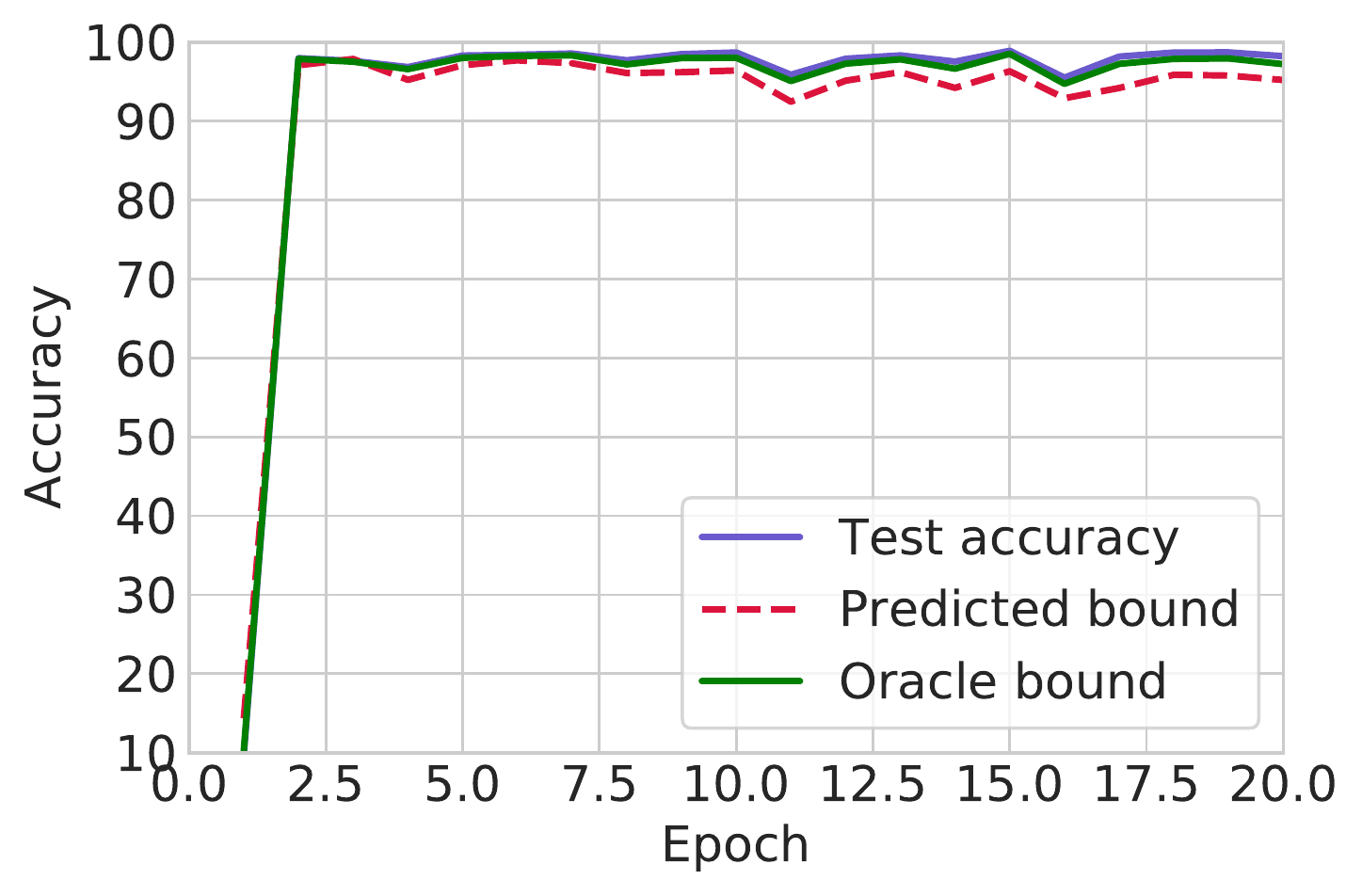}}
    % \includegraphics[width=0.9\linewidth]{figures/{CIFAR10_rn=0.1_lr=0.2_wd=0.005}.png}
    % \vspace{-10pt}
    \caption{ Per epoch curves for MNIST corresponding results in \tabref{table:multiclass}. As before, we just plot the dominating term in the RHS of \eqref{eq:multiclass_ERM} as predicted bound. Additionally, we also plot the predicted lower bound by the error on mislabeled data which nevertheless were predicted as true label. We refer to this as ``Oracle bound''. See text for more details. 
    % 
    % except for the stopping point. 
    % The bound predicted by RATT (RHS in \eqref{eq:multiclass_ERM}) is vacuous. 
    }\label{fig:error_epoch_MNIST}
    % \vspace{-15pt}
\end{figure}

\textbf{Results on CIFAR 100 {} {}} 
On CIFAR100, our bound in \eqref{eq:multiclass_ERM} yields vacous bounds. However, the oracle bound as explained above yields tight guarantees in the initial phase of the learning (i.e., when learning rate is less than $0.1$) (\figref{fig:error_CIFAR100}).  

\begin{figure}[h]
    \centering 
    % \vspace{-15pt}
    % \includegraphics[width=0.9\linewidth]{example-image-a}
    \includegraphics[width=0.3\linewidth]{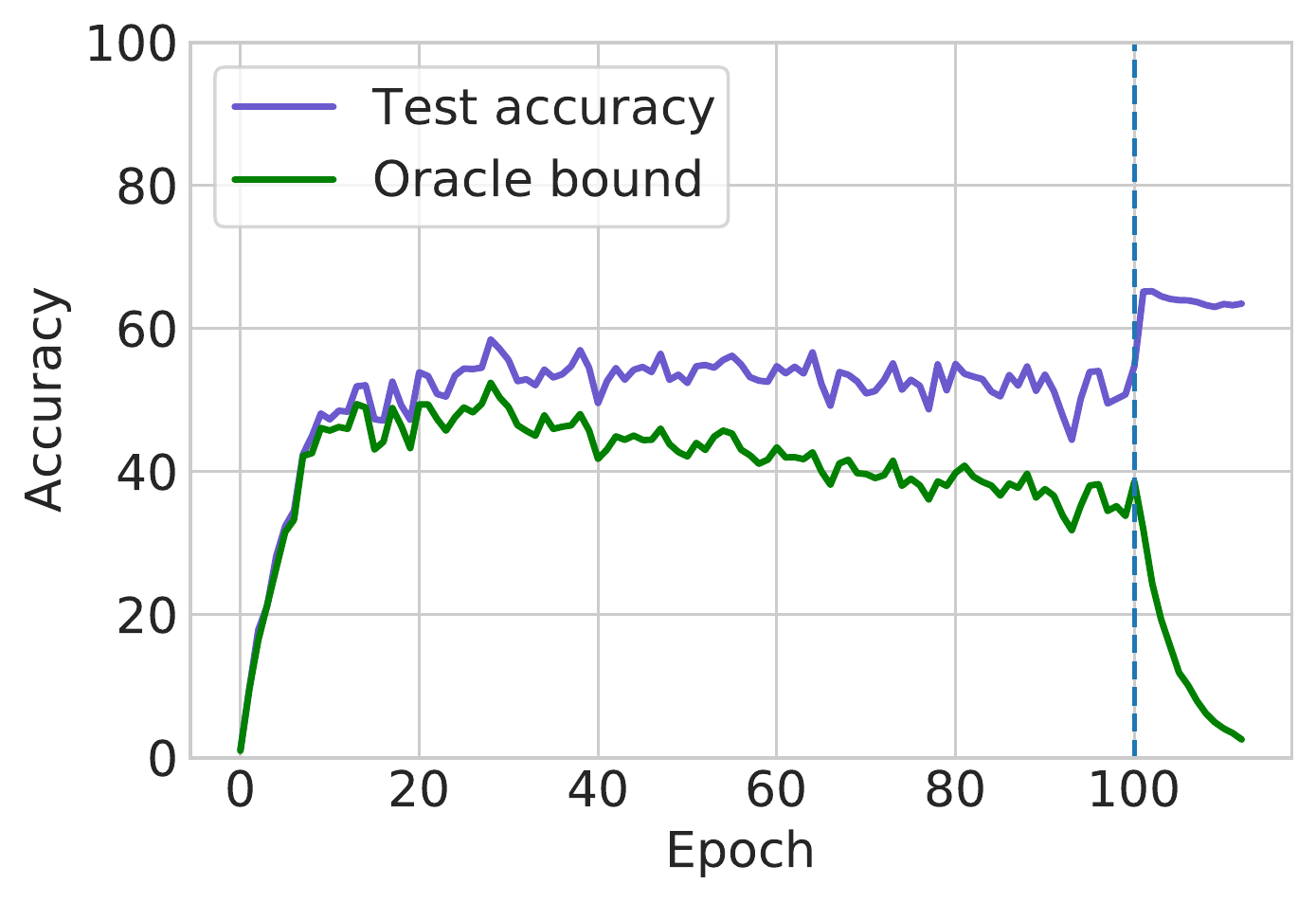}
    % \includegraphics[width=0.9\linewidth]{figures/{CIFAR10_rn=0.1_lr=0.2_wd=0.005}.png}
    % \vspace{-10pt}
    \caption{ Predicted lower bound by the error on mislabeled data which nevertheless were predicted as true label with ResNet18 on CIFAR100. We refer to this as ``Oracle bound''. See text for more details. 
    % 
    % except for the stopping point. 
    The bound predicted by RATT (RHS in \eqref{eq:multiclass_ERM}) is vacuous. 
    }\label{fig:error_CIFAR100}
    % \vspace{-15pt}
\end{figure}

% \paragraph{Experiments on CIFAR100} 

% \subsection{Model Selection using RATT}

\subsection{Hyperparameter Details}

\textbf{\figref{fig:error_CIFAR10} {} {}} We use clean training dataset of size $40,000$. We fix the amount of unlabeled data at $20\%$ of the clean size, i.e. we include additional $8,000$ points with randomly assigned labels. We use test set of $10,000$ points. For both MLP and ResNet, we use SGD with an initial learning rate of $0.1$ and momentum $0.9$. We fix the weight decay parameter at $5\times 10^{-4}$. After $100$ epochs, we decay the learning rate to $0.01$. We use SGD batch size of $100$. 

\textbf{\figref{fig:error_binary} (a) {} {}} We obtain a toy dataset according to the process described in \secref{sec:app_dataset}. We fix $d=100$ and create a dataset of $50,000$ points with balanced classes. Moreover, we sample additional covariates with the same procedure to create randomly labeled dataset. For both SGD and GD training, we use a fixed learning rate $0.1$.    

\textbf{\figref{fig:error_binary} (b) {} {}} Similar to binary CIFAR, we use clean training dataset of size $40,000$ and fix the amount of unlabeled data at $20\%$ of the clean dataset size. To train wide nets, we use a fixed learning of $0.001$ with GD and SGD. We decide the weight decay parameter and the early stopping point that maximizes our generalization bound (i.e. without peeking at unseen data ).  We use SGD batch size of $100$. 

\textbf{\figref{fig:error_binary} (c) {} {}} With IMDb dataset, we use a clean dataset of size $20,000$ and as before, fix the amount of unlabeled data at $20\%$ of the clean data. To train ELMo model, we use Adam optimizer with a fixed learning rate $0.01$ and weight decay $10^{-6}$ to minimize cross entropy loss. We train with batch size $32$ for 3 epochs. To fine-tune BERT model, we use Adam optimizer with learning rate $5\times 10^{-5}$ to minimize cross entropy loss. We train with a batch size of $16$ for 1 epoch.    

\textbf{\tabref{table:multiclass} {} {}} For multiclass datasets, we train both MLP and ResNet with the same hyperparameters as described before. We sample a clean training dataset of size $40,000$ and fix the amount of unlabeled data at $20\%$ of the clean size. We use SGD with an initial learning rate of $0.1$ and momentum $0.9$. We fix the weight decay parameter at $5\times 10^{-4}$. After $30$ epochs for ResNet and after $50$ epochs for MLP, we decay the learning rate to $0.01$.  We use SGD with batch size $100$. 
For \figref{fig:error_CIFAR100}, we use the same hyperparameters as 
CIFAR10 training, except we now decay learning rate after $100$ epochs.

In all experiments, to identify the best possible accuracy on just the clean data, we use the exact same set of hyperparamters except the stopping point. We choose a stopping point that maximizes test performance. 

\subsection{Summary of experiments }

\begin{center}
    \begin{table}[H] 
        \centering
        \begin{tabular}{|c|c|c|c|} 
        \hline
        Classification type & Model category & Model & Dataset  \\ [0.5ex] 
        \hline
        \hline
        \multirow{10}{*}{Binary} & Low dimensional & Linear model & Toy Gaussain dataset  \\
                        \cline{2-4}
                         & Overparameterized 
                        %  & Linear model & Toy Gaussain dataset \\
                        %  \cline{3-4}
                        %  & & 2-layer wide net& Toy Gaussain dataset \\
                        %  \cline{3-4}
                         & \multirow{2}{*}{2-layer wide net} & \multirow{2}{*}{Binary MNIST} \\
                         & linear nets & &  
                         \\
                         \cline{2-4}                 
                         & \multirow{6}{*}{Deep nets} & \multirow{2}{*}{MLP} & Binary MNIST \\
                         \cline{4-4}
                         & &  & Binary CIFAR \\
                         \cline{3-4}
                         &  & \multirow{2}{*}{ResNet} & Binary MNIST \\
                         \cline{4-4}
                         & &  & Binary CIFAR \\
                         \cline{3-4}
                         &  & ELMo-LSTM model & IMDb Sentiment Analysis \\
                         \cline{3-4}
                         & & BERT pre-trained model & IMDb Sentiment Analysis \\
        \hline
        \multirow{5}{*}{Multiclass} & \multirow{5}{*}{Deep nets} & \multirow{2}{*}{MLP} & MNIST \\
                        \cline{4-4} 
                        & & & CIFAR10 \\                   
                        \cline{3-4}
                         &   & \multirow{3}{*}{ResNet} & MNIST \\
                         \cline{4-4}
                         &   & & CIFAR10 \\
                         \cline{4-4}
                         &   & & CIFAR100 \\
        \hline
        \end{tabular}
        % \caption{Summary of experiments performed} \label{table:experiments}
    \end{table}    
    % \footnotetext[6]{We use both MSE loss and cross-entropy loss.}
    % \footnotetext[6]{We try 2 variants: one with a fixed first layer and the other with both layers trainable.}
\end{center}

\newpage
\section{Proof of \lemref{lem:stability_error}} \label{app:proof_lem_error}

\begin{proof}[Proof of \lemref{lem:stability_error}]
    Recall, we have a training set $S \cup \wt S_C$. We defined leave-one-out error on mislabeled points as $$\error_{\text{LOO}(\wt S_M) } = \frac{\sum_{(x_i, y_i) \in \wt S_M} \error( f_{(i)}( x_i), y_i)}{ \abs{\wt S_M }} \,, $$
    where $f_{(i)} \defeq f(\calA, (S \cup \wt S)_{(i)})$. Define $S^\prime \defeq S \cup \wt S$. Assume $(x,y)$ and $(x^\prime,y^\prime)$ as i.i.d. samples from ${\calDm}$. 
    Using Lemma 25 in \citet{bousquet2002stability}, we have
    \begin{align*}
        \Expo{ \left( \error_{\calDm}(\wh f) -\error_{\text{LOO}(\wt S_M) } \right)^2 } \le & \Expt{ S^\prime, (x,y), (x^\prime,y^\prime) }{ \error(\wh f(x), y ) \error(\wh f(x^\prime), y^\prime )} - 2 \Expt{ S^\prime, (x,y) }{ \error(\wh f(x), y ) \error(f_{(i)}(x_i), y_i )} \\
        & + \frac{m_1-1}{m_1}\Expt{ S^\prime }{  \error(f_{(i)}(x_i), y_i )  \error(f_{(j)}(x_j), y_j )} + \frac{1}{m_1} \Expt{ S^\prime }{  \error(f_{(i)}(x_i), y_i ) } \,. \numberthis \label{eq:main_reln}
    \end{align*}
    We can rewrite the equation above as : 
    \begin{align*}
        \Expo{ \left( \error_{\calDm}(\wh f) -\error_{\text{LOO}(\wt S_M) } \right)^2 } \le &  \, \underbrace{\Expt{ S^\prime, (x,y), (x^\prime,y^\prime) }{ \error(\wh f(x), y ) \error(\wh f(x^\prime), y^\prime ) - \error(\wh f(x), y ) \error(f_{(i)}(x_i), y_i )}}_{\RN{1}} \\
        & + \underbrace{\Expt{ S^\prime }{  \error(f_{(i)}(x_i), y_i )  \error(f_{(j)}(x_j), y_j ) -  \error(\wh f(x), y ) \error(f_{(i)}(x_i), y_i )}}_{\RN{2}} \\ &+ \underbrace{\frac{1}{m_1} \Expt{ S^\prime }{  \error(f_{(i)}(x_i), y_i ) - \error(f_{(i)}(x_i), y_i )  \error(f_{(j)}(x_j), y_j ) }}_{\RN{3}} \,. \numberthis \label{eq:main_reln2}
    \end{align*}
    
    We will now bound term $\RN{3}$.  Using Cauchy-Schwarz's inequality, we have
    
    \begin{align}
        \Expt{ S^\prime }{  \error(f_{(i)}(x_i), y_i ) - \error(f_{(i)}(x_i), y_i )  \error(f_{(j)}(x_j), y_j ) }^2 &\le  \Expt{ S^\prime }{  \error(f_{(i)}(x_i), y_i ) }^2 \Expt{S^\prime}{1 -   \error(f_{(j)}(x_j), y_j ) }^2 \\
        &\le \frac{1}{4} \,.\label{eq:term1_lem12}
    \end{align}
    
    Note that since $(x_i,y_i)$, $(x_j ,y_j )$, $(x,y)$, and $(x^\prime, y^\prime)$ are all from same distribution $\calDm$, we directly incorporate the bounds on term $\RN{1}$ and $\RN{2}$ from the proof of Lemma 9 in \citet{bousquet2002stability}. Combining that with \eqref{eq:term1_lem12} and our definition of hypothesis stability in \codref{cond:hypothesis_stability}, we have the required claim.

\end{proof}